\documentclass{article}

\usepackage{microtype}
\usepackage{subfigure}
\usepackage{booktabs} 
\usepackage[stable]{footmisc}
\usepackage{multibib}
\usepackage{xcolor}

\newcommand{\vlat}{\vz}
\newcommand{\lat}{z}
\newcommand{\vLat}{\vZ}
\newcommand{\Lat}{Z}
\newcommand{\varpar}{\lambda}
\newcommand{\vvarpar}{\vlambda}
\newcommand{\vVarpar}{\vLambda}
\newcommand{\vmix}{\vw}

\newcommand{\mix}{w}

\newcommand{\varmean}{m}
\newcommand{\vvarmean}{\vm}
\newcommand{\vVarmean}{\vM}
\newcommand{\Varmean}{M}
\newcommand{\fim}{F}
\newcommand{\vfim}{\vF}

\newcommand\cut[1]{}

\newcommand{\elbofinal}{\mathcal{L}}

\newcommand{\squishlist}{
   \begin{list}{$\bullet$}
    { \setlength{\itemsep}{0pt}      \setlength{\parsep}{3pt}
      \setlength{\topsep}{3pt}       \setlength{\partopsep}{0pt}
      \setlength{\leftmargin}{1.5em} \setlength{\labelwidth}{1em}
      \setlength{\labelsep}{0.5em} } }

\newcommand{\squishlisttwo}{
   \begin{list}{$\bullet$}
    { \setlength{\itemsep}{0pt}    \setlength{\parsep}{0pt}
      \setlength{\topsep}{0pt}     \setlength{\partopsep}{0pt}
      \setlength{\leftmargin}{2em} \setlength{\labelwidth}{1.5em}
      \setlength{\labelsep}{0.5em} } }

\newcommand{\squishend}{
    \end{list}  }

{}
\newtheorem{thm}{Theorem}{}
{}
{}
\newtheorem{defn}{Definition}{}

\newcommand{\half}{\mbox{$\frac{1}{2}$}}

\newcommand{\rnd}[1]{\left(#1\right)}
\newcommand{\sqr}[1]{\left[#1\right]}
\newcommand{\crl}[1]{\left\{#1\right\}}
\newcommand{\myang}[1]{\langle#1\rangle}
\newcommand{\myexpect}{\mathbb{E}}
\newcommand{\Unmyexpect}[1]{\mathbb{E}_{\scaleto{#1\mathstrut}{6pt}}}
\newcommand{\Unmyvar}[1]{\mathbb{V}_{\scaleto{#1\mathstrut}{6pt}}}

\newcommand{\expdist}{\mbox{Exp}}

\newcommand{\gauss}{\mbox{${\cal N}$}}
\newcommand{\chisq}{\mbox{${\chi^2 }$}}
\newcommand{\mgauss}{\mbox{${\cal MN}$}}

\newcommand{\IGauss}{\mbox{InvGauss}}

\newcommand{\myvec}[1]{\mbox{$\mathbf{#1}$}}
\newcommand{\myvecsym}[1]{\mbox{$\boldsymbol{#1}$}}

\newcommand{\valpha}{\mbox{$\myvecsym{\alpha}$}}

\newcommand{\vepsilon}{\mbox{$\myvecsym{\epsilon}$}}
\newcommand{\veta}{\mbox{$\myvecsym{\eta}$}}

\newcommand{\vmu}{\mbox{$\myvecsym{\mu}$}}
\newcommand{\vlambda}{\mbox{$\myvecsym{\lambda}$}}
\newcommand{\vLambda}{\mbox{$\myvecsym{\Lambda}$}}

\newcommand{\vphi}{\mbox{$\myvecsym{\phi}$}}

\newcommand{\vpi}{\mbox{$\myvecsym{\pi}$}}

\newcommand{\vsigma}{\mbox{$\myvecsym{\sigma}$}}
\newcommand{\vSigma}{\mbox{$\myvecsym{\Sigma}$}}

\newcommand{\va}{\mbox{$\myvec{a}$}}
\newcommand{\vb}{\mbox{$\myvec{b}$}}

\newcommand{\vg}{\mbox{$\myvec{g}$}}

\newcommand{\vm}{\mbox{$\myvec{m}$}}

\newcommand{\vs}{\mbox{$\myvec{s}$}}

\newcommand{\vu}{\mbox{$\myvec{u}$}}

\newcommand{\vw}{\mbox{$\myvec{w}$}}

\newcommand{\vx}{\mbox{$\myvec{x}$}}

\newcommand{\vz}{\mbox{$\myvec{z}$}}
\newcommand{\vA}{\mbox{$\myvec{A}$}}
\newcommand{\vB}{\mbox{$\myvec{B}$}}

\newcommand{\vD}{\mbox{$\myvec{D}$}}

\newcommand{\vF}{\mbox{$\myvec{F}$}}
\newcommand{\vG}{\mbox{$\myvec{G}$}}

\newcommand{\vI}{\mbox{$\myvec{I}$}}

\newcommand{\vL}{\mbox{$\myvec{L}$}}
\newcommand{\vM}{\mbox{$\myvec{M}$}}

\newcommand{\vS}{\mbox{$\myvec{S}$}}

\newcommand{\vU}{\mbox{$\myvec{U}$}}
\newcommand{\vV}{\mbox{$\myvec{V}$}}
\newcommand{\vW}{\mbox{$\myvec{W}$}}
\newcommand{\vX}{\mbox{$\myvec{X}$}}

\newcommand{\vZ}{\mbox{$\myvec{Z}$}}

\newcommand{\diag}{\mbox{$\mbox{diag}$}}

\newcommand{\calD}{\mbox{${\cal D}$}}

\newcommand{\data}{\calD}

\newcommand{\be}{\begin{equation}}
\newcommand{\ee}{\end{equation}}
\newcommand{\bea}{\begin{eqnarray}}
\newcommand{\eea}{\end{eqnarray}}
\newcommand{\beaa}{\begin{eqnarray*}}
\newcommand{\eeaa}{\end{eqnarray*}}

\usepackage[accepted]{icml2019}
\usepackage{scalerel}
\usepackage{graphicx,color}
\usepackage{amsmath}
\usepackage{amsfonts}
\usepackage{verbatim}
\usepackage{url}

\newtheorem{lemma}{Lemma}
\newenvironment{proof}{\paragraph{Proof:}}{\hfill$\square$}

\icmltitlerunning{Fast and Simple Natural-Gradient Variational Inference with Mixture of Exponential-family Approximations}

\begin{document}

\twocolumn[
\icmltitle{Fast and Simple Natural-Gradient Variational Inference \\
with Mixture of Exponential-family Approximations}




\begin{icmlauthorlist}
\icmlauthor{Wu Lin}{ubc}
\icmlauthor{Mohammad Emtiyaz Khan}{riken}
\icmlauthor{Mark Schmidt}{ubc}
\end{icmlauthorlist}

\icmlaffiliation{ubc}{University of British Columbia, Vancouver, Canada.}
\icmlaffiliation{riken}{RIKEN Center for Advanced Intelligence Project, Tokyo, Japan}

\icmlcorrespondingauthor{Wu Lin}{wlin2018@cs.ubc.ca}

\icmlkeywords{Variational Inference, Approximate Bayesian Inference, Natural Gradients, Mixture of Exponential Family, Machine Learning, ICML}

\vskip 0.3in
]



\printAffiliationsAndNotice{}  

\begin{abstract}
Natural-gradient methods enable fast and simple algorithms for variational inference, but due to computational difficulties, their use is mostly limited to \emph{minimal} exponential-family (EF) approximations.
In this paper, we extend their application to estimate \emph{structured} approximations such as mixtures of EF distributions.
Such approximations can fit complex, multimodal posterior distributions and are generally more accurate than unimodal EF approximations.
By using a \emph{minimal conditional-EF} representation of such approximations, we derive simple natural-gradient updates. 
Our empirical results demonstrate a faster convergence of our natural-gradient method compared to black-box gradient-based methods with reparameterization gradients. Our work expands the scope of natural gradients for Bayesian inference and makes them more widely applicable than before.

\end{abstract}

\section{Introduction}
Variational Inference (VI) provides a cheap and quick approximation to the posterior distribution, and is now widely used in many areas of machine learning \citep{kingma2013auto, furmston2010variational, WainwrightJordan08, hensman2013gaussian, nguyen2017variational}.
%
In recent years, many natural-gradient methods have been proposed for VI \citep{sato2001online, honkela2007natural, Honkela:11, hensman2012fast,hoffman2013stochastic, khan2017conjugate}.
These works have shown that, for specific types of models and approximations, natural-gradient methods can result in simple updates which converge faster than gradient-based methods.
For example, stochastic variational inference (SVI) \citep{hoffman2013stochastic} is a popular natural-gradient method for conjugate exponential-family models.
\begin{figure}
   \center
   \includegraphics[width=3.0in]{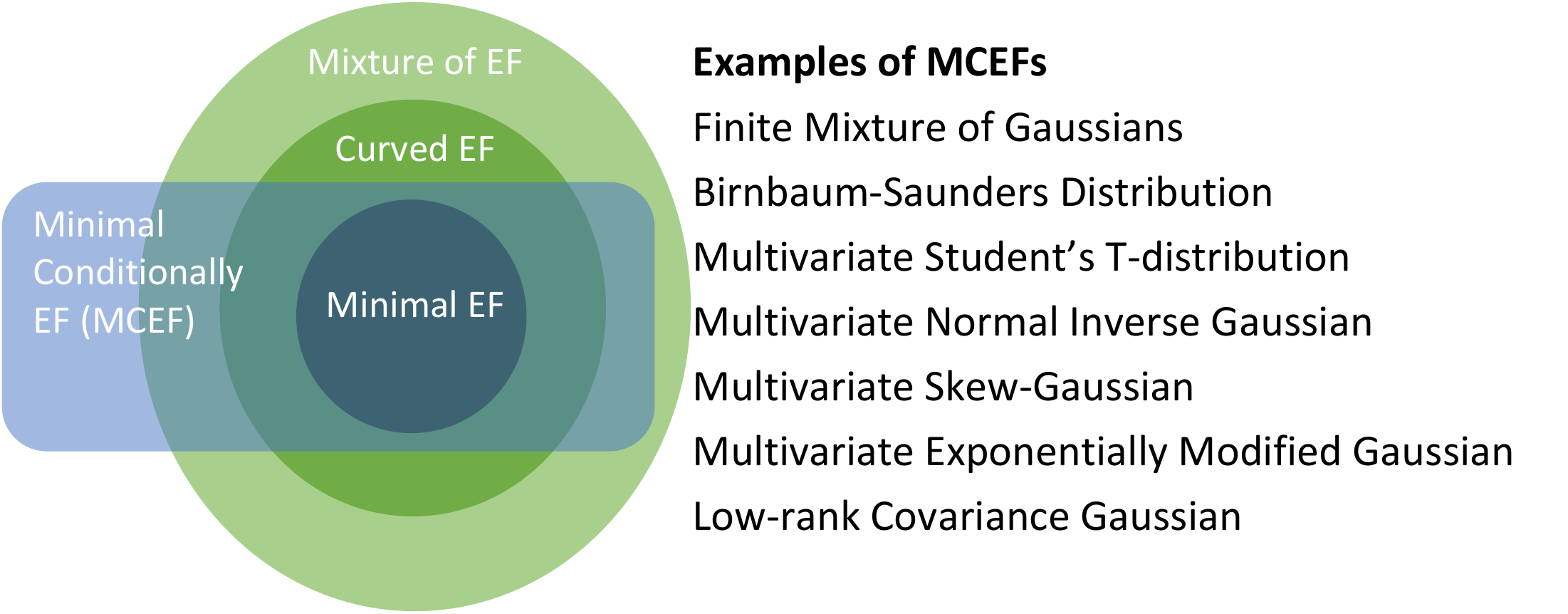}
   \vspace{-0.4cm}
   \caption{We derive simple natural-gradient updates for approximations with a \emph{minimal-conditional EF} (MCEF) representation. Such approximations include all minimal (normalized) EF distributions, and some curved and mixture of EF distributions.
   }
   \label{fig:fig1}
   \vspace{-0.3cm}
\end{figure}
Unfortunately, the simplicity of natural-gradient updates is currently limited to VI with \emph{minimal} exponential-family (EF) approximations.
For such approximations, we can efficiently compute natural-gradients in the natural-parameter space without explicitly computing the Fisher information matrix (FIM) \citep{khan2018fast}.
Unfortunately, this property does not extend to many other approximations such as mixtures of EF distributions.
Such \emph{structured} approximations are more appropriate for complex and multi-modal posterior distributions, giving a more accurate fit than minimal EF distributions. However, computation of natural-gradients is challenging for them.

In this paper, we propose a simple new natural-gradient method for VI with structured approximations.
We define a class of distributions which take a \emph{minimal conditional-EF} (MCEF) form. This includes many members of mixture and curved EF distributions (see Fig.~\ref{fig:fig1}).
Using the MCEF representation and an expectation parameterization associated with it, we derive simple natural-gradient updates.
We show examples on a variety of models where simple natural-gradient updates can be used to estimate flexible and accurate structured approximations.
Our empirical results show faster convergence of our method compared to gradient descent for VI.
Our work extends the simplicity of natural-gradient methods making them more widely applicable than before, while maintaining their fast convergence.

\subsection{Related Works}
Existing work on natural-gradient VI to obtain EF approximations all assume a minimal representation to ensure invertiblility of the FIM. Such methods show fast convergence, result in simple updates, and lead to a straightforward implementation for many types of models. They have been used for conditionally-conjugate EF models \citep{sato2001online, hoffman2013stochastic} and non-conjugate models \citep{khan2017conjugate}, including deep neural networks \citep{khan18a, zhang2017noisy, mishkin2018slang} and
Gaussian processes \citep{khan2016faster, salimbeni2018natural}. Our work presents fast and simple updates for approximations that are beyond the reach of these works.

%


For structured approximations, existing works employ stochastic-gradient methods \cite{ salimans2013fixed, hoffman2015structured, ranganath2016hierarchical, titsias2018unbiased, yin2018semi}.
Such methods are widely applicable, but not as widely used as their mean-field counterparts. This is because they are computationally expensive and slow to converge. Our work attempts to improve these aspects for a flexible class of approximations. 

%


\section{Natural-Gradient Variational Inference}
\label{sec:ngvi}
We begin with a description of natural-gradient descent for variational inference in probabilistic models.
Given a probabilistic model $p(\data,\vlat)$ to model data $\data$ using latent vector $\vlat$,
the goal of Bayesian inference is to compute the posterior distribution:
$p(\vlat|\data) = p(\data,\vlat)/p(\data)$. This requires computation of the \emph{marginal likelihood} $p(\data) = \int p(\data,\vlat) d\vlat$, which is a high dimensional integral and difficult to compute.
VI simplifies this problem by approximating the posterior distribution $p(\vlat|\data)$ by another distribution whose normalizing constant is easier to compute.
For example, a common choice to approximate the posterior is to use a \emph{regular}\footnote{An EF is regular when $\Omega$ is an open set.} exponential-family (EF) approximation,
\begin{align}
   q(\vlat|\vvarpar_\lat) := h_\lat(\vlat)\exp\sqr{ \myang{\vphi_{\lat}(\vlat), \vvarpar_{\lat}} - A_\lat(\vvarpar_\lat)} ,
   \label{eq:expfamdef}
\end{align}
where $q$ denotes the approximating distribution
, $\vphi_{\lat}(\vlat)$ are the sufficient statistics,  $h_\lat(\vlat)$ is the base measure, and $\vvarpar_\lat\in \Omega$ is the natural parameter with $\Omega$ being the set of valid natural-parameters (the set of $\vvarpar_\lat$ where the log-partition function $A_\lat(\vvarpar_\lat)$ \footnote{We assume $A_\lat(\vvarpar_\lat)$ can be efficiently computed.} is
finite) and $\myang{\cdot,\cdot}$ denotes an inner product.
Such parametrized approximations can be estimated by maximizing the variational lower bound:
\begin{align}
   \elbofinal(\vvarpar_\lat) := \myexpect_q \sqr{ \log p(\data,\vlat) - \log q(\vlat|\vvarpar_\lat) },
\end{align}
which can be solved by gradient descent, as shown below:
\begin{align}
   \textrm{GD : } \quad \vvarpar_\lat &\leftarrow \vvarpar_\lat + \alpha \nabla_{\varpar_\lat} \elbofinal(\vvarpar_\lat) ,
\end{align}
where $\nabla$ denotes the gradient and $\alpha>0$ is a scalar learning rate.
The GD algorithm is simple and convenient to implement by using modern automatic-differentiation methods and the reparameterization trick \citep{ranganath2014black, titsias2014doubly}.
Unfortunately, such \emph{first-order} methods can show suboptimal rates of convergence and be slow in practice.

An alternative approach is to use  natural-gradient descent which exploits the information geometry of $q$ to speed-up convergence. Assuming that the Fisher information matrix (FIM) of $q(\vlat|\vvarpar_\lat)$, denoted by $\vF_z(\vvarpar_\lat)$, is positive-definite for all $\vvarpar_\lat \in \Omega$, the natural-gradient descent (NGD) for VI  in the natural-parameter space is given as follows:
\begin{align}
   \textrm{NGD : } \quad \vvarpar_\lat &\leftarrow \vvarpar_\lat + \beta\,\, \sqr{ \vfim_\lat(\vvarpar_\lat) }^{-1} \nabla_{\varpar_\lat} \elbofinal(\vvarpar_\lat),
   \label{eq:ngd}
\end{align}
The preconditioning of the gradients by the FIM leads to a proper scaling of the gradient in each dimension, and takes into account dependencies between variables. This often leads to faster convergence, particularly when the FIM is well-conditioned.

A naive implementation of \eqref{eq:ngd} would require computation and inversion of the FIM. However, for specific types of models and approximations, NGVI could be simpler to compute than GD. This is true for mean-field VI in conjugate-exponential family models \citep{hoffman2013stochastic} as well as Bayesian neural networks \citep{khan18a, zhang2017noisy}. 
This computational efficiency is a result of a simple NGD update to estimate EF approximations, as shown by \citet{khan2018fast}. 
For EFs, we can use the expectation-parameter, defined as the function $\vvarmean_\lat(\vvarpar_\lat) := \myexpect_q \sqr{\vphi_{\lat}(\vlat)}$ from $\Omega\to \mathcal{M}$, to compute natural-gradients.
This is possible because of the following relation,
\begin{align}
   \nabla_{\varpar_\lat}\elbofinal &= \sqr{\nabla_{\varpar_\lat} \vvarmean_\lat^T} \nabla_{\varmean_\lat} \elbofinal = \sqr{\vfim_\lat(\vvarpar_\lat)}  \nabla_{\varmean_\lat}\elbofinal 
   \label{eq:chainrule1}
\end{align}
where the first equality is obtained by applying the chain rule and second equality is obtained by noting that $\nabla_{\varpar_\lat} \vvarmean_\lat^T  =  \nabla^2_{\varpar_\lat}  A(\vvarpar_\lat) =  \vfim_\lat(\vvarpar_\lat)$.
When the FIM is invertible, we get a simple update for NGD,
\begin{align}
   \vvarpar_\lat \leftarrow \vvarpar_\lat + \beta\,\, \nabla_{\varmean_\lat} \elbofinal(\vvarpar_\lat), \label{eq:ngd_meanpar}
\end{align}
where the gradient is computed with respect to $\vvarmean_\lat$. When the above gradient is easier to compute than the gradient with respect to $\vvarpar_\lat$, NGD admits a simpler form than GD. This is the case for many existing works on NGD for VI.

To rewrite NGD as in \eqref{eq:ngd_meanpar}, we need the FIM to be invertible. For EFs, a sufficient condition for invertibility is to use a \emph{minimal} representation which is defined\footnote{A more complete definition is given in Definition 1.3~of \citet{johansen1979introduction}. } below using the definition given in \citet{WainwrightJordan08}.
\begin{defn}
   {\bf (Minimal EF)} A regular EF representation is said to be minimal when there does not exist a nonzero vectors $\vvarpar$ such that $\myang{ \vphi_\lat(\vlat), \vvarpar}$ is equal to a constant.
\end{defn}
This essentially means that there are no linear dependencies in the parameterization of the distribution. When such a nonzero vector exists, we can add/subtract it from $\vvarpar$ without changing the distribution.
Minimality ensures that this never happens and the parametrization $\vvarpar_\lat$ is unique and identifiable up to multiplication with a nonsingular affine transformation.

Under a minimal representation, the log-partition function $A_\lat(\vvarpar_\lat)$ is \emph{strictly} convex, implying that the FIM is positive-definite.\footnote{A formal proof can be found in
\citet{johansen1979introduction}.}
For other types of representation, like \emph{curved} EFs, the FIM may not be positive-definite. Minimality ensures that the FIM is positive-definite and that NGD is well defined \footnote{In some cases, when FIM is not invertible, it is still possible to perform NGD by, for example, ignoring the zero eigenvalues or using damping, but the simplification shown
in \eqref{eq:ngd_meanpar} may not be possible.}.
The NGD update can be then carried out using the expectation parameter which
are a one-to-one function, as stated below \citep{WainwrightJordan08}.
\begin{thm}
   \label{thm:minimalEF}
   The representation \eqref{eq:expfamdef} is minimal  if and only if the mapping $\vvarmean_\lat(\cdot): \Omega \to \mathcal{M}$ is one-to-one.
\end{thm}
Unfortunately, minimal EF approximations are not always appropriate. Such approximations, especially unimodal ones, usually yield poor approximations of multimodal posterior distributions. Structured approximations, such as mixtures of EF distributions, are more suitable for approximating multimodal posteriors. 
Unfortunately, for such approximations, there is no straightforward way to define a minimal EF representation which can be exploited to derive a simple NGD update.
For example, a mixture of EF distributions expressed as
\begin{align}
   q(\vlat) := \int q(\vlat|\vmix)q(\vmix) d\vmix,
   \label{eq:mixexp}
\end{align}
with $q(\vlat|\vmix)$ as the component and $q(\vmix)$ as the mixing distribution, may not even have an EF form even when both the terms above are in the EF. A famous example is the finite mixture of Gaussians.
The conditions under which the FIM of \eqref{eq:mixexp} is invertible are also difficult to characterize in general.
Due to these reasons, it is difficult to simplify the NGD update for such structured distributions.

In this paper, we propose a new way to derive simple natural-gradients for structured approximations \eqref{eq:mixexp} using a minimal conditional representation which we define next.

\section{Minimal Conditional-EF Representation}
\label{sec:cmef}

In this section we define a minimal conditional-EF (MCEF) representation of the joint distribution $q(\vlat,\vmix)$ for structured approximations that take the form \eqref{eq:mixexp}.
Using this, we derive conditions under which the FIM of the joint is invertible, and show that it leads to a simple NGD similar to \eqref{eq:ngd_meanpar}.

We begin with a definition of the conditional EF distribution.
\begin{defn}
   {\bf (Conditional EF) }
   We call the joint distribution $q(\vlat,\vmix)$ defined in \eqref{eq:mixexp} a conditional EF when its components take the following form:\footnote{We assume that $A_\lat(\vvarpar_\lat,\vmix)$ and $A_\mix(\vvarpar_\mix)$ can be efficiently computed.}
   \begin{align}
      q(\vlat|\vmix) &:= h_\lat(\vlat,\vmix)\exp\sqr{\myang{\vphi_\lat(\vlat,\vmix), \vvarpar_\lat} - A_\lat(\vvarpar_\lat,\vmix)} , \nonumber\\
      q(\vmix) &:= h_\mix(\vmix)\exp\sqr{\myang{\vphi_\mix(\vmix), \vvarpar_\mix} - A_\mix(\vvarpar_\mix)} ,  \label{eq:qzw}
   \end{align}
   with $\vvarpar_\mix, \vphi_\mix(\vmix)$ and $A_\mix(\vvarpar_\mix)$ being the natural parameters, sufficient statistics, and log-partition function of $q(\vmix)$, and $\vvarpar_\lat, \vphi_\lat(\vlat,\vmix)$ and $A_z(\vvarpar_\lat,\vmix)$ are the same for $q(\vlat|\vmix)$. 
   We denote the set of natural parameters for $q(\vlat|\vmix)$ and $q(\vmix)$ by $\Omega_\lat$ and $\Omega_\mix$ respectively, and assume them to be open.\footnote{It is possible for $\Omega_\lat$ to be non-open since it is an intersection of all of the valid natural-parameter of $q(\vlat|\vmix)$ for each $\vmix\sim q(\vmix)$. The set is open when the set of valid natural-parameters for $q(\vlat|\vmix)$ conditioned on $\vmix$ does not depend on $\vmix$, or when the cardinality of the support of $\vmix$ is finite. For all the examples given in this paper, the set is an open set.}
   
   \label{def:cef}
\end{defn}
Note that the sufficient statistics $\vphi_\lat(\vlat,\vmix)$ and log-partition $A_\lat(\vvarpar_\lat,\vmix)$ both depend on $\vmix$, but, conditioned on $\vmix$,  the distribution is parametrized by the natural parameter $\vvarpar_\lat$.
For a well-defined conditional EF distribution,  $q(\vmix)$ is a regular EF distribution and conditioned on $\vmix$, $q(\vlat|\vmix)$ is also a regular EF distribution.
This is a type of conditional exponential-family distribution \citep{xing2002generalized,
liang2009learning,lindsey1996parametric,feigin1981conditional} with a special conditional structure.

We are interested in an NGD update that can exploit the FIM of the joint distribution. We denote the set of natural parameters by $\vvarpar := \{ \vvarpar_\lat, \vvarpar_\mix\}$ and the set of valid $\vvarpar_\lat$ and $\vvarpar_\mix$ by $\Omega_\lat$
and $\Omega_\mix$ 
respectively. We define the following FIM in the natural-parameter space $\Omega_\lat\times\Omega_\mix$  as follows: 
\begin{align}
   \vfim_{\mix\lat}(\vvarpar) := - \Unmyexpect{q(\lat,\mix)}\sqr{\nabla_\varpar^2 \log q(\vlat,\vmix)}.
   \label{eq:fimwz}
\end{align}
This is the FIM of the joint distribution $q(\vlat,\vmix)$ which is different from the one for the marginal distribution $q(\vlat)$.
Similar to the minimal EF case, our goal now is to find representations where the above FIM is invertible and can be exploited to compute NGD using expectations of the sufficient statistics. The expectation parameters of a CEF can be defined as shown below:
\begin{align}
   \vvarmean_\lat(\vvarpar_\lat, \vvarpar_\mix) & := \Unmyexpect{q(\lat|\mix)q(\mix)} \sqr{ \vphi_\lat(\vlat,\vmix) }, \label{eq:mz}\\
   \vvarmean_\mix(\vvarpar_\mix) & := \Unmyexpect{q(\mix)} \sqr{ \vphi_\mix(\vmix) } .  \label{eq:mw}
\end{align}
We denote the ranges of $\vvarmean_\lat$ and $\vvarmean_\mix$ by $\mathcal{M}_\lat$ and $\mathcal{M}_\mix$ respectively, and the whole set of expectation parameters by $\vvarmean(\vvarpar) := \{\vvarmean_\lat, \vvarmean_\mix\}$. Since $\vphi_\lat(\vlat,\vmix)$ and $A_\lat(\vvarpar_\lat,\vmix)$ both depend on $\vmix$, we may or may not be able to perform NGD using these expectation parameters. However, we next show NGD is always possible by restricting to a minimal representation of CEFs.

Below we define a minimal representation for CEF which ensures that the $\vfim_{\mix\lat}(\vvarpar)$ is positive definite, and NGD can be performed using $\vvarmean$.
\begin{defn}
   \label{eq:mcef}
   {\bf (Minimal Conditional-EF (MCEF))}
   A conditional EF defined in Definition \ref{def:cef} is said to have a minimal representation when $\vvarmean_\mix(\cdot):\Omega_\mix\to \mathcal{M}_\mix$ and  $\vvarmean_\lat(\cdot, \vvarpar_\mix):\Omega_\lat\to \mathcal{M}_\lat$ are both one-to-one, $\forall \vvarpar_\mix \in \Omega_\mix$.
\end{defn}
In the next section, we will show that all the examples shown in Figure~\ref{fig:fig1} (See ``Examples of MCEFs'') have an MCEF representation.
Similar to minimal EFs, an MCEF representation implies that $\vfim_{\mix\lat}(\vvarpar)$ is positive-definite and invertible, as stated in the following theorem (see a proof in Appendix \ref{app:theorem1}).
\begin{thm}
   \label{thm:fimmcef}
   For an MCEF representation given in Definition \ref{eq:mcef}, the FIM $\vfim_{\mix\lat}(\vvarpar)$ given in \eqref{eq:fimwz} is positive-definite and invertible for all $\vvarpar\in\Omega$.
\end{thm}

Since the FIM is well-defined, it is reasonable to perform natural-gradient steps in the Riemannian manifold defined by $\vfim_{wz}(\vvarpar)$. The steps can be taken using an update that takes a simple form, similar to the one shown in \eqref{eq:ngd_meanpar}.
As shown in Lemma \ref{lemma:ng_update_fim} in Appendix \ref{app:theorem1}, similar to \eqref{eq:chainrule1} for minimal EF, we have the following relationship for MCEFs:
\begin{align}
   \nabla_{\varpar}\elbofinal &= \sqr{\nabla_{\varpar} \vvarmean^T} \nabla_{\varmean} \elbofinal = \sqr{\vfim_{wz}(\vvarpar)}  \nabla_{\varmean}\elbofinal 
\end{align}
Since the FIM is invertible, we can compute the natural-gradient by using gradients with respect to $\vvarmean$. The following theorem shows the simplicity of NGD.
\begin{thm}
   \label{thm:mixexp}
   For minimal conditional-EF approximations, the following updates are equivalent: 
   \begin{align}
      \vvarpar &\leftarrow \vvarpar + \beta\,\, \sqr{\fim_{wz}(\vvarpar)}^{-1} \nabla_{\varpar} \elbofinal(\vvarpar) \\
      \vvarpar & \leftarrow \vvarpar + \beta\,\, \nabla_{\varmean} \mathcal{L}(\vvarpar).
      \label{eq:simplengvi}
   \end{align}
\end{thm}
The above theorem generalizes the previous result for minimal EF approximations to minimal conditional-EF approximations. It also implies that using the expectation parameterization $\vvarmean$ enables us to exploit the geometry of the joint $q(\vlat,\vmix)$ to improve convergence.

\section{Examples}
\label{sec:examples}
In this section, we give examples of approximations with a minimal conditional-EF form. We discuss finite mixtures of EFs, scale mixtures of Gaussians, and multivariate Skew-Gaussians. 
We give updates that can be implemented efficiently by using existing NGD implementations, e.g., the variational online Newton (VON) method of \citet{khan18a}. We also propose many new versions.
To derive these updates, we use the extended Bonnet's and Price's theorems \citep{wu-report} for Gaussian mixtures. 
For exponential-family mixtures, we use the implicit reparameterization trick \citep{salimans2013fixed,figurnov2018implicit}. \citet{wu-report} discuss a weaker version of the reparameterization trick for exponential-family mixtures.

\begin{table*}[!t]
\center
   \caption{We give expressions for various components of $q(\vlat,\vmix)$ as specified in \eqref{eq:qzw} for a minimal conditional-EF representation. We give three examples. More examples can be found at the Appendix. The first two rows give expressions for the sufficient statistics, and the subsequent rows show natural and expectation parameters. For the
first column, $\mathbb{I}_c(\mix)$ denotes the indicator function which is 1 when $w=c$ and 0 otherwise. For the second column, $\psi$ is the digamma function. For the third column, $ c:=\sqrt{2/\pi }$. Note that both the natural and expectation parameters may lie in a constrained set. Due to space limitations, we have not explicitly given the description of these sets.}
\begin{tabular}{l|l|l|l}
    & Mixture of Gaussians  
    & T-Distribution
    & Skew-Gaussian
    \\
    \hline
          
   $\phi_\mix(\mix)$ 
      & $\crl{ \mathbb{I}_c(\mix) }_{c=1}^{K-1} $%
      & $-1/\mix-\log \mix $  %
      & $ \mix, \,  \mix^2  $ %
      \\

   $\vphi_\lat(\vlat,\mix)$
      & $\crl{ \mathbb{I}_c(\mix) \vz,\,\, \mathbb{I}_c(\mix) \vz\vz^T }_{c=1}^K $%
      & $\crl{ \vlat/\mix,\, \vlat \vlat^{T}/\mix } $ %
      & $\crl{  \vlat,\, |\mix|\vlat, \, \vlat \vlat^{T} } $  %
      \\

   $\vvarpar_\mix $
      & $\crl{\log \rnd{\pi_c/\pi_K} }_{c=1}^{K-1} $%
      & $a$ %
      & constants $\crl{0, \, -\half}$ 
      \\
  
   $\vvarmean_\mix(\vvarpar_\mix) $
      & $\{\pi_c\}_{c=1}^{K-1}$
      & $-1 - \log a + \psi(a)$
      & constants $\crl{0, \, 1}$
     \\
     
  $\vvarpar_\lat$ 
      & $\crl{\vSigma_c^{-1}\vmu_c,\,\, -\half\vSigma_c^{-1} }_{c=1}^K$%
      & $\crl{ \vSigma^{-1}\vmu, -\half \vSigma^{-1}  } $
      & $\crl{ \vSigma^{-1}\vmu, \,\, \vSigma^{-1}\valpha , \,\, -\half \vSigma^{-1}  } $
      \\

   $\vvarmean_\lat(\vvarpar) $
      & $\crl{ \pi_c \vmu_c, \pi_c \rnd{\vmu_c\vmu_c^T + \vSigma_c} }_{c=1}^K$ %
      & $\crl{\vmu, \vmu \vmu^T + \vSigma } $ 
      & $\big\{\vmu + c \valpha, \,\, \valpha + c  \vmu,$ 
      \\
      
     ~
     & ~
     & ~
     &  $ \vmu \vmu^T + \valpha \valpha^T+ \vSigma + c \left( \vmu \valpha^T + \valpha \vmu^T \right) \big\} $ \\
     \hline
     
\end{tabular}
\label{tab:examples}
\end{table*}

\subsection{Finite Mixture of Exponential Family Distributions}
\label{sec:finiteef}
Finite mixtures of EFs are a powerful approximation where components in EF form are mixed using a discrete distributions such as a multinomial distribution, as shown below:
\begin{align}
   q(\vlat) = \sum_{c=1}^K  \pi_c q(\vlat|\vvarpar_c), \,\, \textrm{such that } \sum_{c=1}^K \pi_c = 1,
   \label{eq:finitemixexp1}
\end{align}
where $q(\vlat|\vvarpar_c)$ are EF distributions with natural parameters $\vvarpar_c$. The support of $\vlat$ 
of all the components is assumed to be the same, and $\pi_K$ is fixed to $1 - \sum_{c=1}^{K-1} \pi_c$.

This distribution cannot be written in an exponential form in general, and therefore none of the existing NGD methods can be used to derive a simple expression for NGVI. Directly applying NGD using the FIM of $q(\vlat)$ would be too expensive since the number of parameters are in $O(D^2K)$, and the FIM could be extremely large when computed naively.
Fortunately, the joint distribution does take an MCEF form when all $q(\vlat|\vvarpar_c)$ are minimal EFs.
A formal statement is given in Appendix \ref{app:finitemix}.
Using this conditional EF form, we can derive a much simpler NGD update which reduces to simple parallel updates on mixture components.

We will now demonstrate the simplicity of our update for a finite mixture of Gaussians (MOG) where components $q(\vlat|\vvarpar_c) := \gauss(\vlat|\vmu_c,\vSigma_c )$ are Gaussian.
The sufficient statistics, natural parameters, and expectation parameters of the corresponding CEF are given in Table \ref{tab:examples}.

We consider the following model where the likelihood $p(\data_n|\vlat)$ is defined using neural-network weights $\vlat$ with a prior $p(\vlat)$: $p(\data,\vlat) = \prod_{n=1}^N p(\data_n|\vlat) p(\vlat)$.
We approximate the posterior with a MOG by optimizing the following variational lower bound:
$\mathcal{L}(\vvarpar) := \myexpect_{q(\lat)}[- h(\vlat)]$ where $h(\vlat) := \log \sqr{q(\vlat)/p(\vlat)} -\sum_{n}\log p(\data_n|\vlat)$.

We now summarize the NGD update derived in Appendix \ref{app:finitemix}. We first generate samples $(\vlat,\mix)$ from $q(\vlat,\mix)$. 
The mean $\vmu_c$ and covariance $\vSigma_c$ are similarly updated by the variational online Newton (VON) algorithm \citep{khan18a}:
\begin{align}
   \vSigma_c^{-1} &\leftarrow \vSigma_c^{-1} + \beta  \delta_c \, \sqr{  \nabla_{\lat}^2 h(\vlat) } , \nonumber\\
   \vmu_c &\leftarrow \vmu_c - \beta \delta_c \vSigma_c  \sqr{\nabla_{\lat}  h(\vlat) } ,  \label{eq:dis_mix}\\
   \log \rnd{ \pi_{c}/\pi_K} &\leftarrow  \log \rnd{ \pi_{c}/\pi_K}  - \beta  (\delta_{c} - \delta_{K})  h(\vlat) \nonumber
\end{align}
where $\delta_{c}:= \gauss(\vlat|\vmu_c,\vSigma_c) / \sum_{k=1}^K \pi_{k} \gauss(\vlat|\vmu_{k},\vSigma_{k}) $. This update can be implemented efficiently using approximations discussed in \citet{khan18a, zhang2017noisy, mishkin2018slang}. For example, implementation of the Adam optimizer can be utilized. This is simpler and more efficient than the update which requires computation of the FIM.

A similar example to MOG is the fatigue life distribution \citep{birnbaum1969new} discussed in Appendix \ref{app:BS}.

\subsection{Scale Mixture of Gaussians}
\label{sec:tdis}
A multivariate scale-mixture of Gaussian (SMG) distribution \citep{andrews1974scale, eltoft2006multivariate} takes the following form where the covariance matrix is ``scaled'' by a vector $\vmix$ sampled from $q(\vmix)$ such that $w_i>0$:
\begin{align*}
   q(\vlat) = \int \gauss(\vlat|\vmu,\vL \vW \vL^T) \prod_{i=1}^d q(w_i) d\vw ,
\end{align*}
where $\vW := \diag(\vw)$ is a diagonal matrix containing $\vw$ as the diagonal and $\vL$ is a matrix with determinant 1 (e.g., a Cholesky factor) that determines the covariance matrix.
SMG includes many well-known distributions like Student's t, Laplace, logistic, doubly-exponential, normal-gamma, normal-inverse Gaussian, normal-Jeffreys, and their nonparametric extensions \cite{caron2008sparse}.
For example, the multivariate Student's t-distribution is obtained by using a scalar $w$ from an inverse-gamma distribution:
\begin{align}
   \mathcal{T}(\vlat|\vmu,\vSigma,a) := \int \gauss(\vlat|\vmu,\mix\vSigma) \,\, \mathcal{IG}(\mix|a,a) d \mix, 
\label{eq:studentt}
\end{align}
where $\vSigma := \vL\vL^T$ and both the shape and scale parameters of the inverse-gamma distribution are equal to $a$. We assume $a>1$, since for $a\leq1$ the variance of $q(\vlat)$ does not exist.

SMG is another class of approximations where existing methods cannot be applied to obtain a simple natural-gradient update.~For example, for Student's t the joint-distribution $q(\vz,w)$ takes an EF form, but, as we show in Lemma \ref{lemma:curvedTdist} in Appendix \ref{app:tdist}, the number of free parameters is equal to 3, while the number of natural parameters is equal to 4. Such representations are not minimal EFs but \emph{curved} where the FIM is not positive-definite because the number of free parameters is less than the number of natural parameters. Therefore, the update \eqref{eq:ngd_meanpar} does not apply.
In contrast, our update can be applied since the joint distribution is a minimal CEF. This can be verified from Table \ref{tab:examples} where both $\varmean_\mix(\varpar_\mix)$ and $\vvarmean_\lat(\vvarpar)$ are one-to-one functions. A formal proof is given in Lemma \ref{lemma:tdistmcefproof} in Appendix \ref{app:tdist}.

We now demonstrate the simplicity of our update to obtain a Student's t-approximation on a Bayesian neural network.
We assume the following model with a likelihood $p(\data_n|\vlat)$ specified using a neural network with weights $\vlat$ and a Student's t-prior: $p(\data,\vlat) = \prod_{n=1}^N p(\data_n|\vlat) \mathcal{T}(\vlat|\mathbf{0}, \vI, a_0)$.
The Student's t-prior is better than a Gaussian one if we expect the weights to follow a heavy-tailed distribution.
We approximate the posterior by the approximation \eqref{eq:studentt} using parameters $\vSigma$, $\vmu$, and $a$.
We use the variational lower bound defined in the joint distribution $p(\data,\vlat,w)$ space:
\vspace{-0.1cm}
\begin{align}
   \mathcal{L}(\vvarpar) = \myexpect_{q(\lat, \mix|\lambda)} \Big[ \sum_{n=1}^N \overbrace{\log p(\data_n|\vlat) }^{:= - f_n(\boldsymbol{\lat})} + \log \frac{p(\vlat,\mix)}{q(\vlat, \mix)} \Big]  \label{eq:t_elbo} 
\end{align}
Below, we summarize the NGD updates derived in Appendix \ref{app:tdist}. We first sample $(\vlat,\mix)$ from $q$ and randomly sampled an example $n$. The update is then a small modification of the VON update \cite{khan18a}:
\begin{align}
   \vSigma^{-1} &\leftarrow (1-\beta) \vSigma^{-1} +  \beta \sqr{ u \nabla_z^2 f_n(\vlat) + \vI/N },  \\
   \vmu &\leftarrow \vmu -\beta \vSigma \sqr{ \nabla_z f_n(\vlat) + \vmu /N }, \\
   a &\leftarrow (1-\beta) a + \beta \sqr{ a_0 - \delta \mathrm{Tr}\left(\nabla_z^2 f_n(\vlat) \vSigma \right) }, \label{eq:updatea}
\end{align} 
where $\beta>0$, $\delta \leftarrow N w^2/(2(1-w))$, and $u$ is a pre-multiplier defined below ($d$ is length of $\vlat$):
$   u = (a-1 + d/2)^{-1} \sqr{a + \left( \vlat-\vmu \right)^T \vSigma^{-1} \left( \vlat-\vmu \right)/2  }$.
Similarly to the previous section, this update can be implemented efficiently using the method of \citet{khan18a}.

Another example is given at Appendix \ref{app:sym_nig}.
Extensions using an Adam-like optimizer for this kind of mixtures are given in Appendix \ref{app:vadam}.

\begin{figure*}[t]
	\centering
   \includegraphics[width=0.33\linewidth]{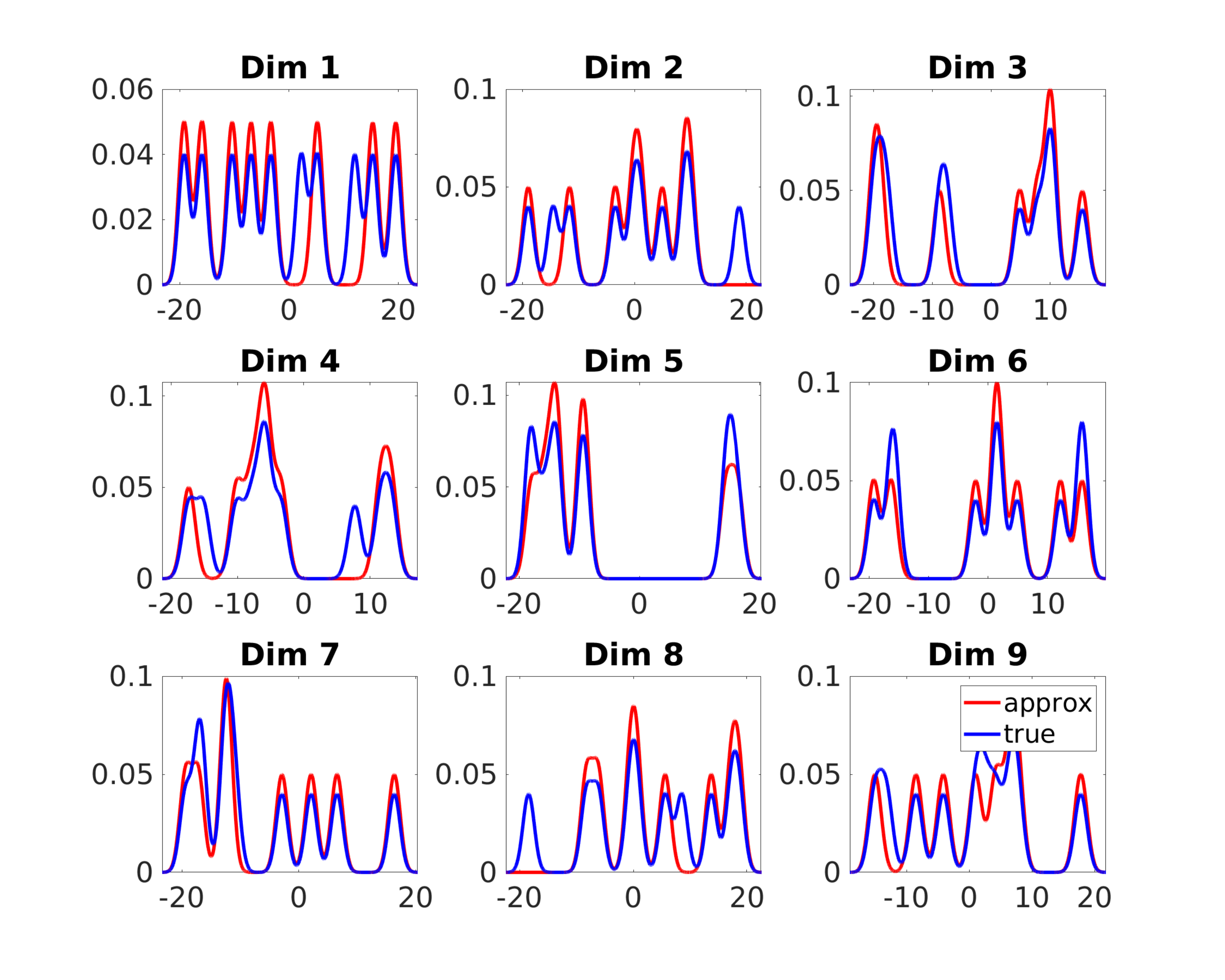}
	\includegraphics[width=0.33\linewidth]{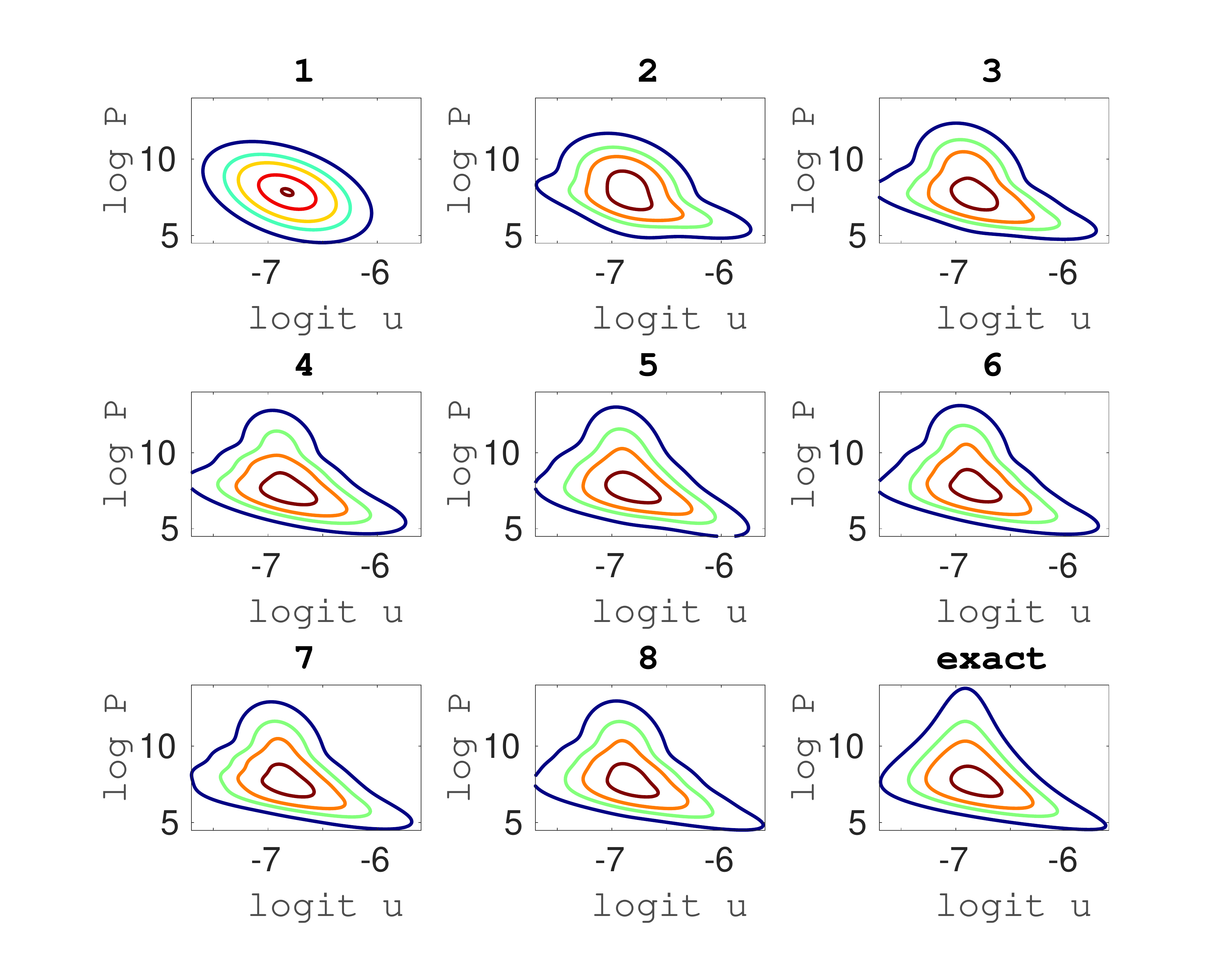}
	\includegraphics[width=0.33\linewidth]{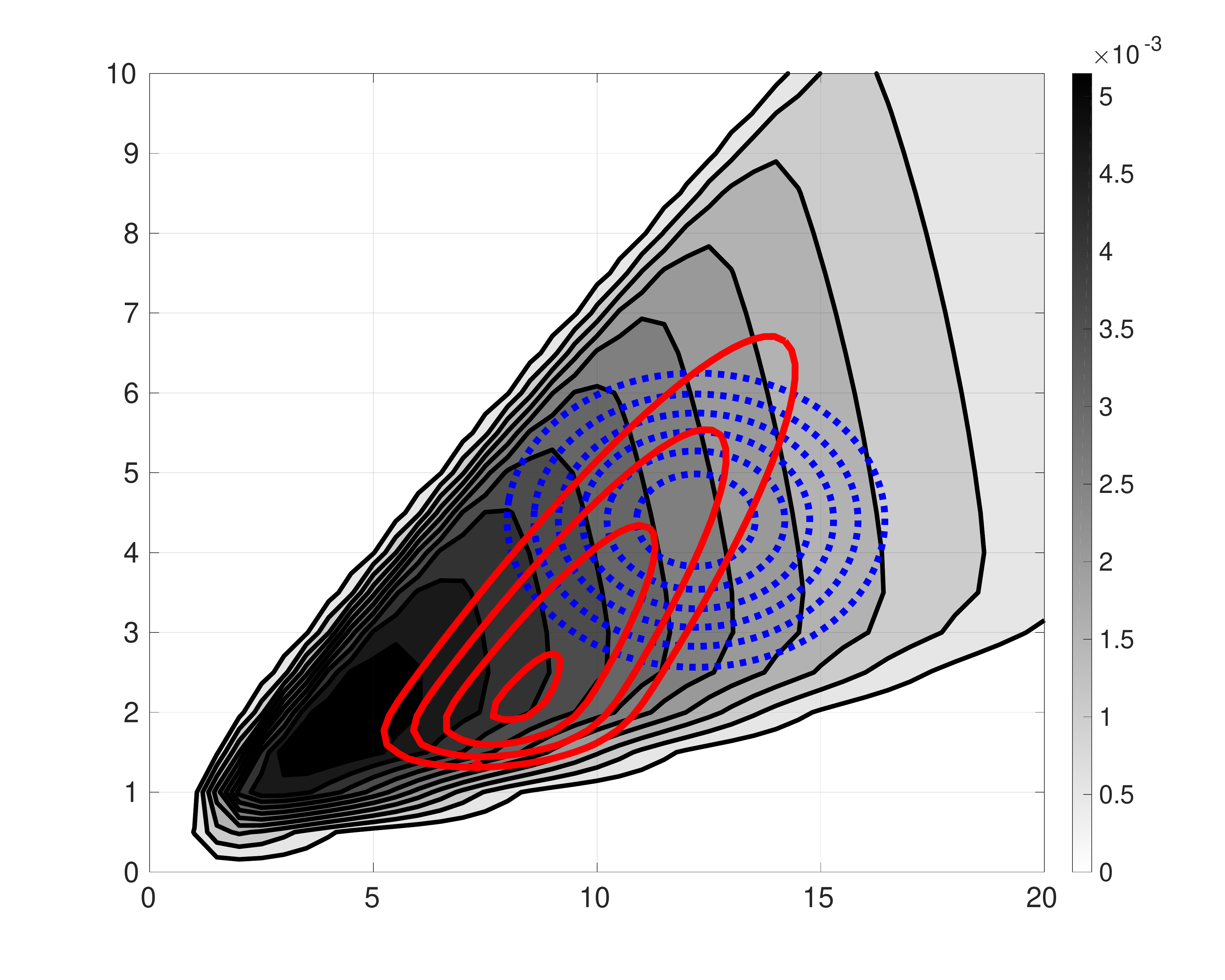}
   \caption{Quantitative results on three toy examples showing the flexibility obtained by using structured approximations considered in this paper. The leftmost figure shows the MOG approximation (with $K=20$) to fit an MOG model with 10 components in a 20 dimensional problem. The first 9 dimensions are shown in the figure where we see that MOG approximation fits the marginals well. The middle figure shows MOG approximation fit to a beta-binomial model for a 2-D problem. The number indicates
   the number of mixture components used. By increasing the number of components in our approximation, we get better results. The last figure shows a Skew Gaussian (in red) and a Gaussian (in blue) fit on a 2-D logistic regression posterior. We see that skew-Gaussian captures the skewness in the distribution in the right direction, and gives better approximation than a single Gaussian.}
	\label{figure:toy_examples}
\end{figure*}

\subsection{Gaussian Mean Mixture}
\label{sec:skewGdist}
We consider the following Gaussian mixture.
\begin{align*}
 q(\vlat| \vmu, \valpha, \vSigma) = \int \gauss(\vlat| \vmu+ \sum_{i=1}^k u(\mix_i) \valpha_i , \vSigma) \prod_{i=1}^{k} q(\mix_i) d\vmix,
\end{align*} where $\valpha$ is a $d$-by-$k$ matrix and $\vlat \in \mathcal{R}^d$.
An example is the rank-1 covariance Gaussian with $k=1$, $u(\mix)=\mix$, and $\vD$ as a diagonal covariance matrix: $q(\vlat|\vmu,\valpha,\vD)=\int \gauss(\vlat| \vmu+ \mix \valpha , \vD) \gauss(\mix|0,1)d\mix=\gauss(\vlat|\vmu,\valpha\valpha^T+\vD)$.

The multivariate skew Gaussian \citep{azzalini1996multivariate,genton2004skew}, defined below, is another example and allows for non-zero skewness (asymmetric approximations):
\begin{align*}
 q(\vlat| \vmu, \valpha, \vSigma) = \int \gauss(\vlat| \vmu+ |\mix| \valpha , \vSigma) \gauss(\mix|0,1) d\mix.  \label{eq:skewgauss}
\end{align*}
This distribution is not a minimal EF and the FIM of the marginal $q(\vlat)$ can be singular \citep{azzalini2013skew}. However, the joint distribution is a minimal conditional-EF distribution as we show in Lemma \ref{lemma:jointcef_skew} and \ref{lemma:skew_g_proof} in Appendix \ref{app:skew_gauss}. The sufficient statistics, natural parameters, expectation parameters of the conditional EF form are given in Table \ref{tab:examples}.

Similar to other examples, we get a simple and efficient NGD update. We summarize the updates for a model with a neural-network likelihood $p(\data_n|\vlat)$ using weights $\vlat$ and a Gaussian prior $\gauss(\vlat|0,\vI)$.
Denoting $f_n(\vlat)= -\log p(\data_n|\vlat)$, the lower bound is $\mathcal{L}(\vvarpar) :=\Unmyexpect{q(\lat)} [ - \sum_n  f_n(\vlat) + \log p(\vlat) - \log q(\vlat)  ]$.
The updates derived in Appendix \ref{app:skew_gauss} is a variant of the VON update:
\begin{align}
 \vSigma^{-1} & \leftarrow (1-\beta) \vSigma^{-1} + \beta (\vI + N \vg_{S}^n) \\
   \vmu & \leftarrow  \vmu - \beta \vSigma \sqr{\bar{c} ( \vg_{\mu}^n - c \vg_{\alpha}^n) + \vmu} \\
   \valpha & \leftarrow  \valpha - \beta \vSigma \sqr{ \bar{c} ( \vg_{\alpha}^n - c \vg_{\mu}^n) + \valpha }
\end{align}
where $\bar{c} := N/(1-2/\pi)$ and $\vg_S^n, \vg_{\mu}^n, \vg_{\alpha}^n$ are gradients obtained by gradient and Hessian of $f_n(\vlat)$ at a sample of $q(\vlat)$. The gradients are defined in \eqref{eq:ske_g_update_mu}-\eqref{eq:ske_g_update_sigma} in Appendix \ref{app:skew_gauss}.

Another example of the mixture is the exponentially modified Gaussian distribution \citep{grushka1972characterization,carr2009saddlepoint} given in Appendix \ref{app:exp_gauss}.
Extensions using an Adam optimizer for the class of mixtures are given in Appendix \ref{app:vadam}.

 

\section{Generalization to Multilinear EF}
\label{sec:generalization}
We now extend the approach to an approximation with multilinear EFs which contain blocks of natural parameters.
We start by specifying a distribution over $\vlat$ by a function $f(\cdot)$:  
\begin{align}
   q(\vlat| \vvarpar) = h_\lat(\vlat) \exp\sqr{f\rnd{\vlat, \vvarpar} - A_\lat(\vvarpar)}.
   \label{eq:qz_multilinear}
\end{align}
Then we divide the vector $\vvarpar := \{\vvarpar_1,\vvarpar_2,\ldots,\vvarpar_K\}$ into $K$ blocks with $\vvarpar_j \in \Omega_j$ being the $j$-th block of parameters.
In EF, the function $f$ is just linear in  $\vvarpar_\lat$ at \eqref{eq:expfamdef}.
We can generalize the notion of linearity to multiple blocks of parameters by considering $f$ to be a \emph{multilinear} function.
\begin{defn}
   {\bf (Minimal Multilinear-EF)} We call $f$ a multilinear function when, for each block $j$, there exist functions $\vphi_j$ and $r_j$ such that $f$ is linear with respect to $\vvarpar_j$, i.e.,
\begin{align}
   f(\vlat, \vvarpar) = \myang{\vvarpar_j, \vphi_j\rnd{\vlat, \vvarpar_{-j} }} + r_j\rnd{\vlat, \vvarpar_{-j} }, 
\end{align}
   where $\vvarpar_{-j}$ is the parameter vector containing all $\vvarpar$ except $\vvarpar_j$.
   A distribution $q(\vlat| \vvarpar)$ defined as in \eqref{eq:qz_multilinear}, but with a multilinear $f$, is called a multilinear EF. Additionally, when $\Omega_j$ is open and the following expectation parameters
      $\vvarmean_j(\vvarpar) :=\Unmyexpect{q(\lat)}\sqr{ \vphi_{j}\left(\vlat, \vvarpar_{-j} \right) }$
   are one-to-one, we call the distribution a minimal multilinear EF distribution.
\end{defn}
Clearly, minimal EFs are minimal multilinear EFs.
The following theorem give a result about a \emph{block} NGD update performed on individual blocks of parameters $\vvarpar_j$.
\begin{thm}
For approximations with multilinearly-minimal EF representation, the following updates are equivalent: 
   \begin{align}
      \vvarpar_j &\leftarrow \vvarpar_j + \beta\,\, \sqr{\vfim_j(\vvarpar)}^{-1} \nabla_{\varpar_j} \elbofinal(\vvarpar) \\
      \vvarpar_j &\leftarrow \vvarpar_j + \beta \,\, \nabla_{\varmean_j} \mathcal{L}(\vvarpar) 
   \end{align}
\end{thm}
The proof of this theorem is similar to Theorem \ref{thm:mixexp}. We now give an example and demonstrate the simplicity of the NGD update. Let's consider the Matrix-Variate Gaussian (MVG) distribution  defined as follows:
\begin{align}
   \mgauss(\vLat|\vW,\vU,\vV) := \gauss(\textrm{vec}(\vLat)| \textrm{vec}(\vW), \vV \otimes \vU). \nonumber
\end{align}
This distribution has been used for Bayesian neural networks \cite{louizos2016structured, sun2017learning}. An approximate NGD update is also derived by \citet{zhang2017noisy} where the FIM is approximated by a block-diagonal matrix and K-FAC approximation.
Our update have a similar block-diagonal approximation, but the update for the each block is an exact NGD unlike \citet{zhang2017noisy} where the steps are approximated by K-FAC.

In Appendix \ref{app:MGauss}, we show that the MVG distribution $\mgauss(\vLat|\vW,\vU,\vV)$ can be written in the minimal multi-linear form.
The NGD update, derived in Appendix \ref{app:MGauss}, is summarized below to optimize the lower bound as $\elbofinal(\vvarpar) = E_q[- h(\vLat)]$ where $h(\vLat):=-\log p(\data, \vLat) + \log q(\vLat)$. To simplify our implementation, we use the Gauss-Newton approximation \citep{graves2011practical} although it is not necessary to do so. The resulting block NGD update is shown below,
\begin{align}
\vW  & \leftarrow \vW - \beta_1 \vU\vG\vV, \\
\vU^{-1} & \leftarrow \vU^{-1} + \beta_2 \vG\vV\vG^\top, \\
\vV^{-1} & \leftarrow \vV^{-1} + \beta_2 \vG^\top \vU\vG ,
\end{align}
where we sample $\vLat$ from the MVG distribution and evaluate the gradient $\vG := \nabla_{\Lat} h(\vLat)$.
These updates extend the VON update obtained in \citet{khan18a} to MVG approximations. The gradient $\vG$ is pre-conditioned, which is very similar to other preconditioned algorithms, such as K-FAC \citep{martens2015optimizing,zhang2017noisy} and Shampoo \citep{gupta18shampoo}. The update can be extended to Tensor-Variate Gaussian \citep{ohlson2013multilinear}.


\section{Experimental Results }
\label{sec:results}
The code is available at:\\\url{https://github.com/yorkerlin/VB-MixEF}.
\subsection{Qualitative Results on Synthetic Examples}
First, we show qualitative results on three toy examples and visualize the results obtained by structured approximations.

The first toy example is the Gaussian mixture example from \citet{wang2018variational}.
In this example, the true distribution is a finite mixture of Gaussians (MOG) $p(\vlat)=\sum_{i=1}^{C} \frac{1}{C} \gauss(\vlat|\vu_i, \vI), \vlat \in \mathcal{R}^d$, where each element of $\vu_i$ is uniformly drawn from the interval $[-s,s]$. 
We approximate the posterior distribution by an MOG approximation described in Section \ref{sec:finiteef}.
We consider a case with $K=20$, $C=10$, $s=20$, and $d=20$.
We initialize $\pi_c=\frac{1}{K}$ and $\vSigma_c=100 \vI$. Each element of $\vmu_c$ is randomly initialized by Gaussian noise with mean $0$ and variance $100$.
We use 10 Monte Carlo samples to compute the gradients.
The leftmost plot in Figure \ref{figure:toy_examples} shows the first $9$ marginal distributions of the true distribution and its approximations, where we clearly see that MOG closely matches the marginals.
All 20 marginal distributions are in Figure \ref{figure:toy_examples_all} in Appendix \ref{sec:fmGauss_plot}.

In the second toy example we approximate the beta-binomial model for overdispersion considered in \citet{salimans2013fixed,salimans2015markov}  by using MOG ($N=20$, $d=2$).
The model is to used to estimate the rates of death from stomach cancer for cities in Missouri.
The exact posterior of the model is non-standard and extremely skewed.
In the middle plot in Figure \ref{figure:toy_examples}, we see that our MOG approximation approximates the true posterior better and better as we increase the number of mixture components.

In the last toy example, we visualize the skew-Gaussian approximations for the two-dimensional Bayesian Logistic regression example taken from \citet{murphy2013machine} ($N=60$, $d=2$).
In the rightmost plot in Figure \ref{figure:toy_examples}, we can see that the skewness of the true posterior is captured well by the skew-Gaussian distribution. The Gaussian approximation results in a worse approximation than the skew-Gaussian.

\begin{figure}[t]
    \includegraphics[width=.23\textwidth]{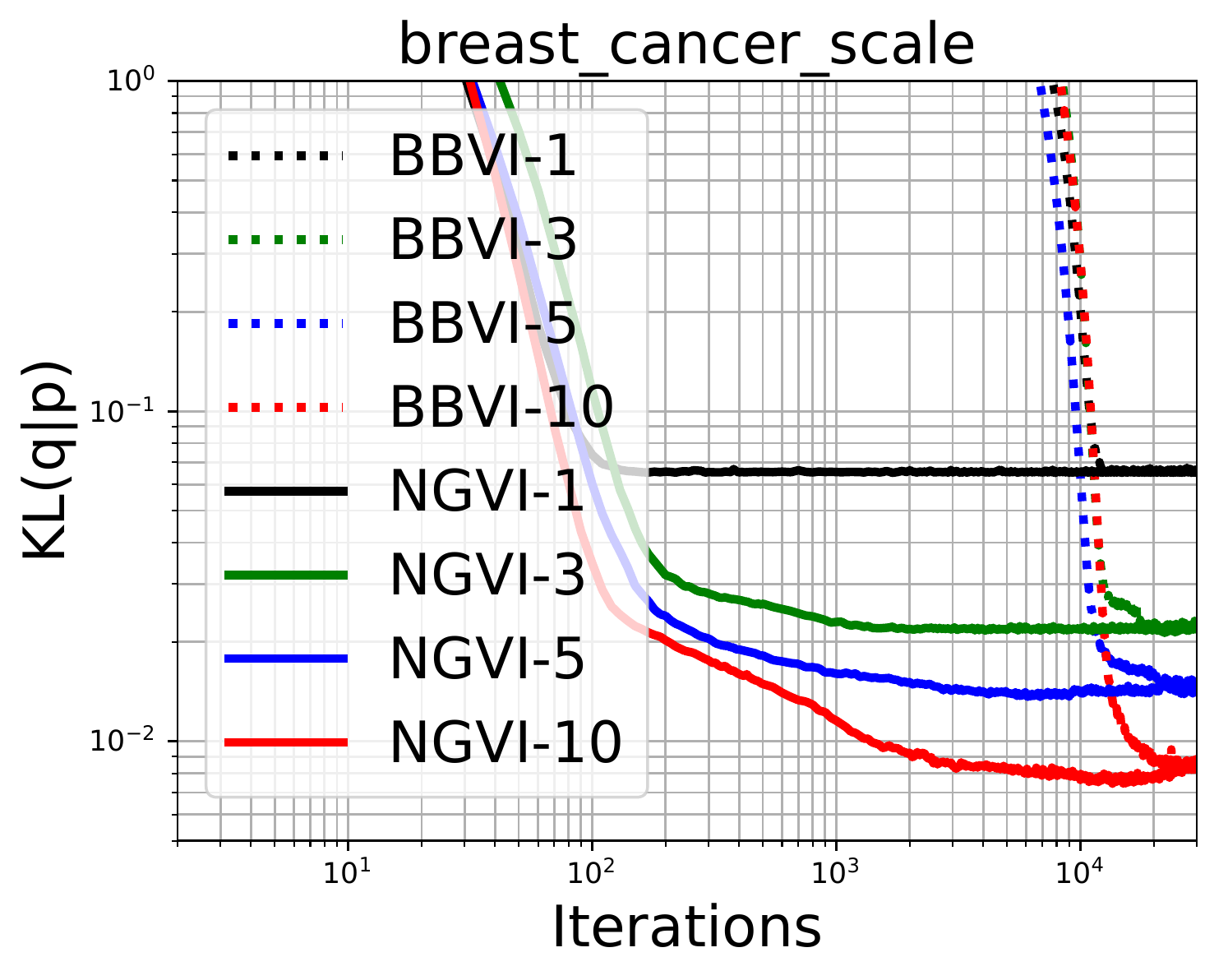} \hfill
    \includegraphics[width=.23\textwidth]{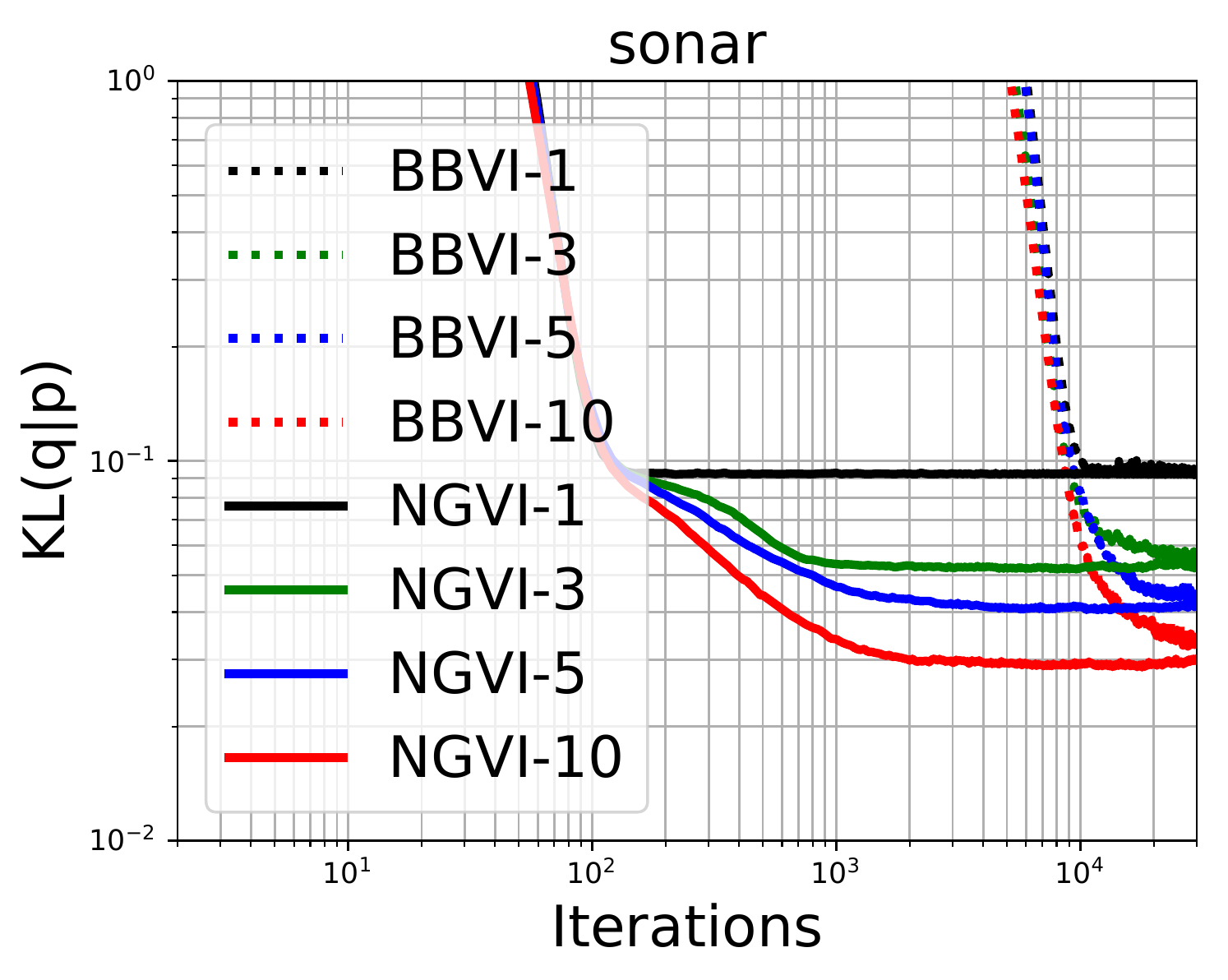}
	\caption{Bayesian logistic regression approximated by MoG: This figure demonstrates a fast convergence of NGVI over BBVI. We use a mixture of Gaussians with full covariance matrix as the approximating distribution. The number next to the method name indicates the number of mixture components used. The plot shows the KL obtained using $10^6$ MC samples, where $p$ is the true posterior distribution. For both algorithms, we used full-batches by using 20 MC samples to compute stochastic approximations. }
	\label{figure:blr_mog}
\end{figure}

\begin{figure}[t]
   \center
\includegraphics[width=.24\textwidth]{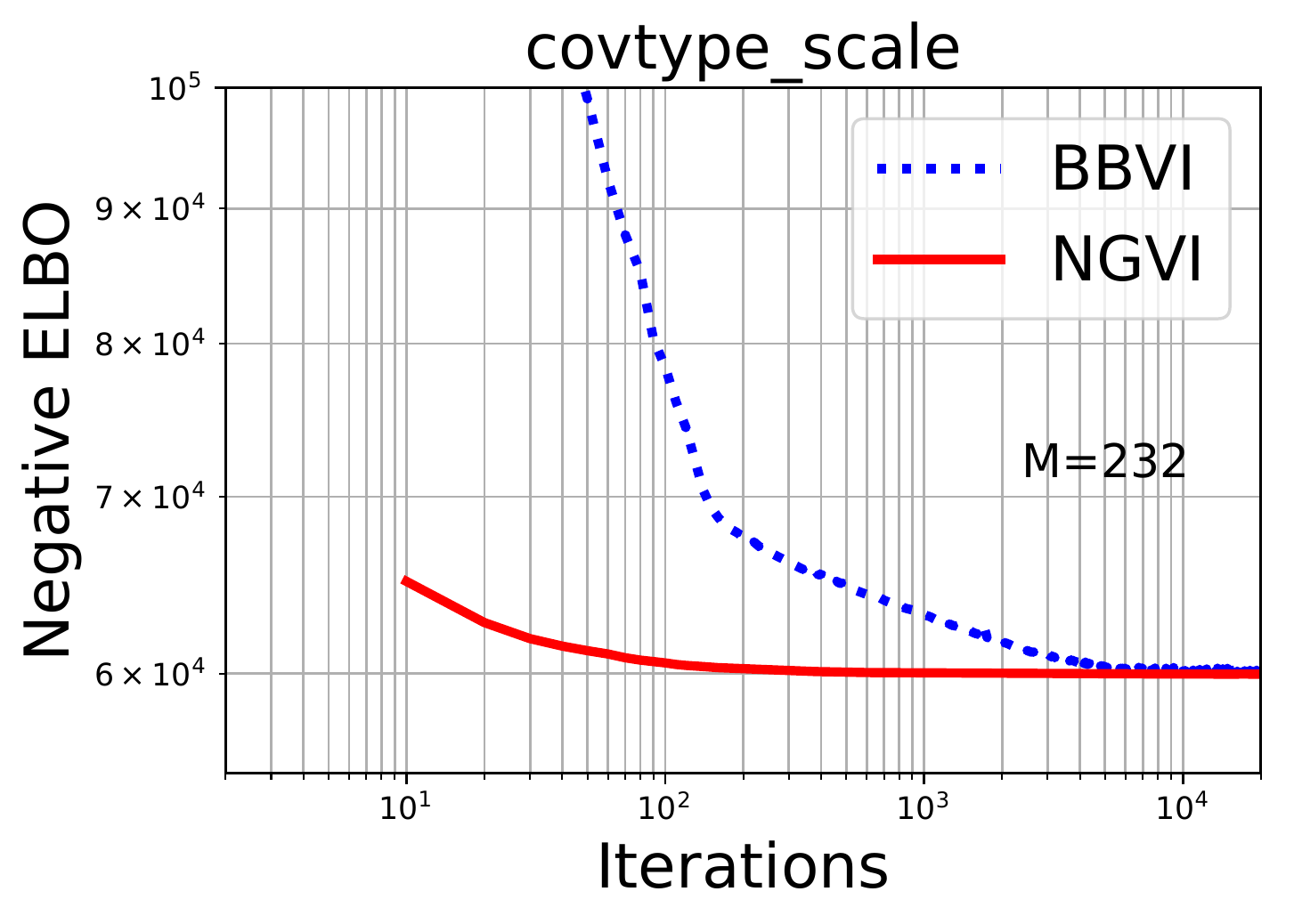}
\includegraphics[width=.23\textwidth]{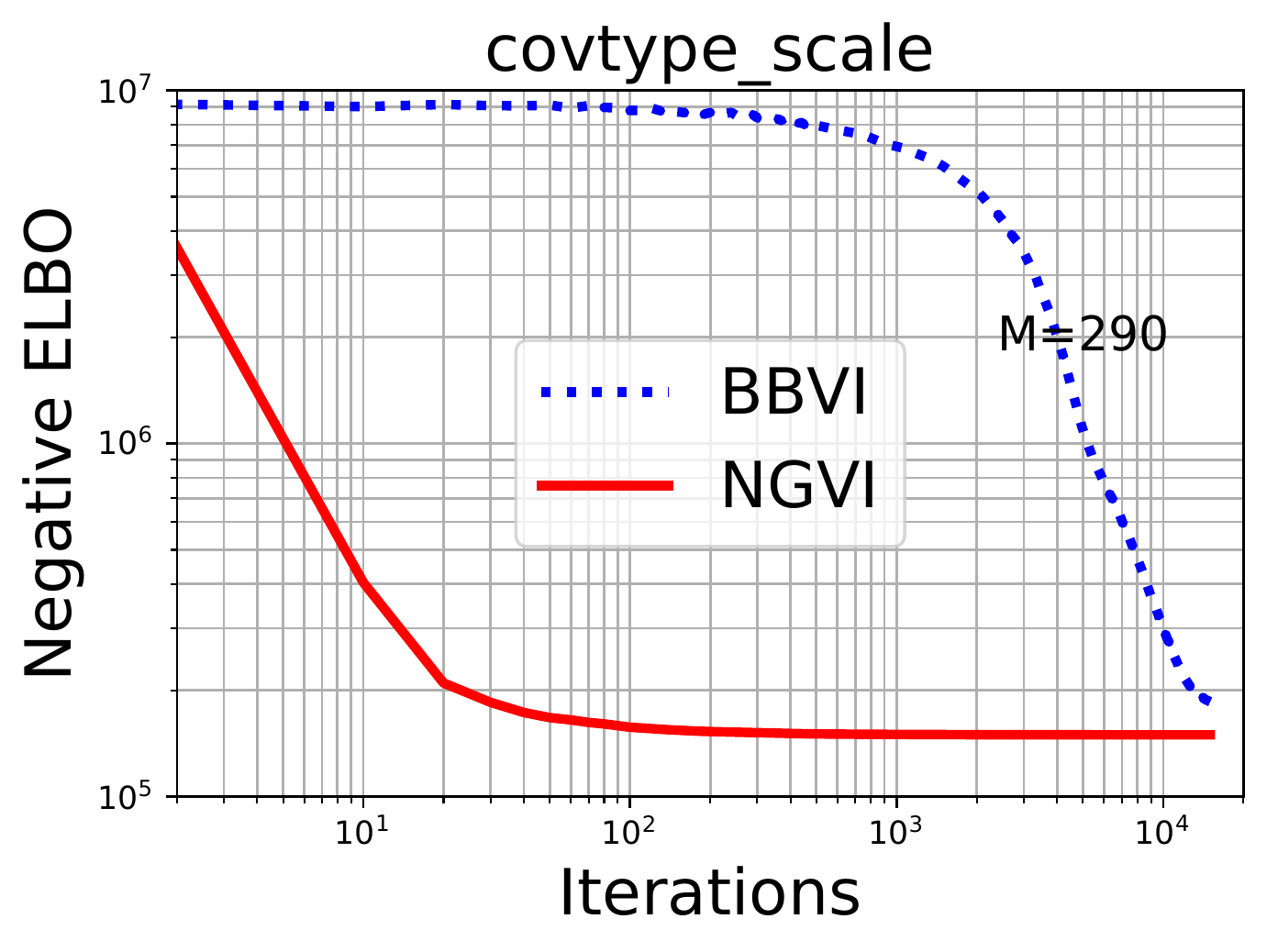}
   \caption{Bayesian logistic regression with Student's t (left) and skew-Gaussian approximations (right). We use $10$ MC samples for training, and $M$ denotes the mini-batch size.}
\vspace{-0.3cm}
	\label{figure:blr_t_skew}
\end{figure}

\subsection{Results on Real Data}
Next, we show results on real-world datasets. We consider two models in our experiments. 
We start with Bayesian Logistic regression (BLR) 
and present results for MOG approximations on two small UCI datasets.
The `Breast-Cancer' dataset has $N=683,d=10$ with $341$ chosen for training, and regularization parameter is set to $1.0$.
The `Sonar' dataset has $N=208,d=6$ with $100$ chosen for training, and regularization parameter is set to $0.204$.
We vary the number of mixture components to $K=1,3,5,10$.
In Figure \ref{figure:blr_mog}, we plot the KL divergence between the true posterior and the MOG approximation, and compare our method (referred to as `NGVI') proposed in Section \ref{sec:finiteef} to the black-box gradient method (referred to as `BBVI') with the re-parametrization trick \citep{salimans2013fixed,kingma2013auto,figurnov2018implicit}.
For both methods, we use a full batch for each update.
We observe that NGVI always converges faster than BBVI.
We also see that MoG is a better approximation than the single Gaussian, and the quality of the posterior approximation improves as the number of mixture increases.

Next, we show results on a larger dataset using two other kinds of variational approximation: Skew-Gaussian and Student's t. 
We use the UCI dataset ``covtype-binary-scale'' with $d=54, N=581,012$ with $464,809$ chosen for training and regularization parameter $0.002$.
We use the algorithm discussed in Section \ref{sec:tdis} and \ref{sec:skewGdist}.
For black-box methods, we use the Adam optimizer and refer to it as BBVI.
In the skew-Gaussian case, we use a Gaussian prior, and, for the Student's t-distribution, we use a Student's t-prior as shown in \eqref{eq:t_elbo}.
Figure \ref{figure:blr_t_skew} demonstrates the fast convergence of our method compared to BBVI.

Finally, we discuss results on Bayesian neural network (BNN) with a standard normal prior on weights. We use one
hidden layer, 50 hidden units, and ReLU activation functions.
We approximate the posterior by a skew-Gaussian distribution using NGVI. 
We also compare to two other methods where we used BBVI to fit a skew-Gaussian approximation as well as a Gaussian approximation.
For scalability reasons, we use of a diagonal covariance for all methods.
We use 10 Monte Carlo (MC) samples and minibatch size of 32.
For NGVI, we use the gradient magnitude-approximation as explained in Appendix \ref{app:vadam}.
Figure \ref{figure:bnn_skew} shows the performance of all methods in terms of the test RMSE. We can see that our method converges faster than BBVI, although the performance of skew-Gaussian methods seem to be similar to a Gaussian. 

\begin{figure}[t]
\includegraphics[width=.23\textwidth]{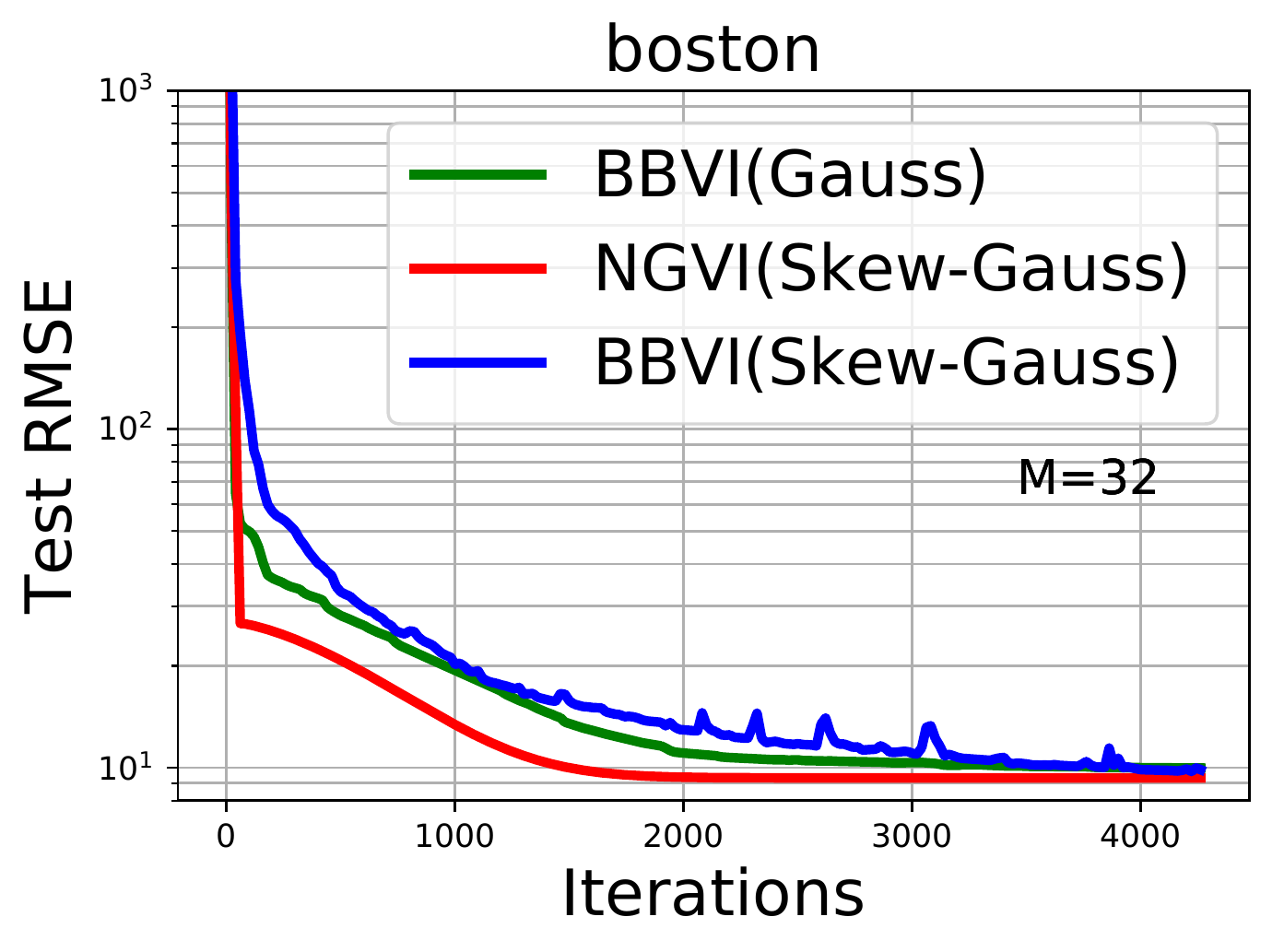}
\includegraphics[width=.23\textwidth]{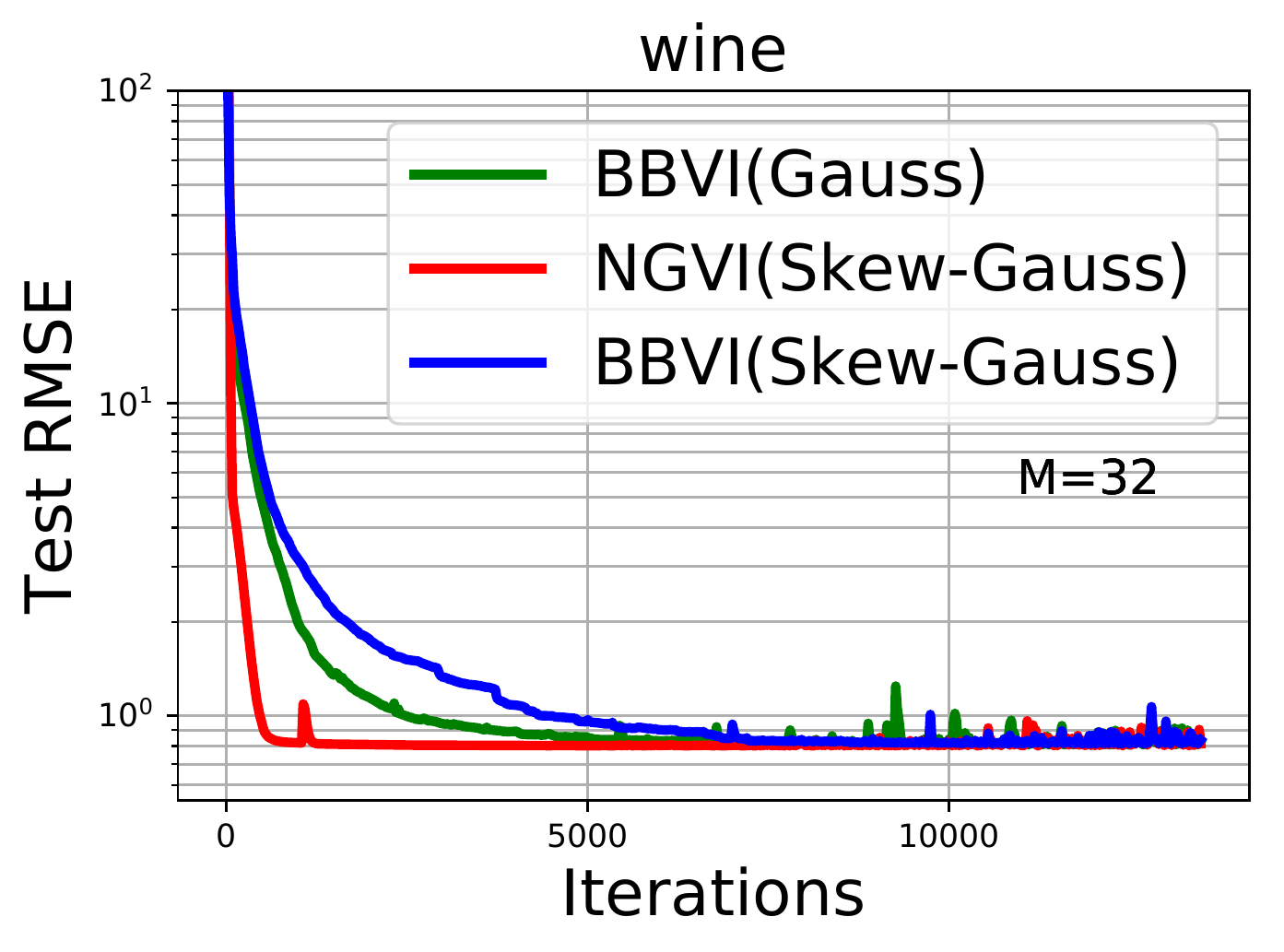}
   \caption{BNN using skew-Gaussian approximation: This figure shows a fast convergence of NGVI over BBVI to approximate the posterior of BNN. For all methods, the prior is a Gaussian. We use 10 MC samples for training. $M$ in the figures denotes the size of a mini-batch. BBVI (Gauss) uses a Gaussian approximation while BBVI (Skew-Gauss) uses a skew-Gaussian approximation. Our NGVI method with skew-Gaussian approximation converges faster than the other two methods.}
	\label{figure:bnn_skew}
\vspace{-1.0cm}
\end{figure}

\section{Discussion}
\label{sec:discussion}
In this paper, we present fast and simple NGD for VI with structured approximations. The approximations we have considered are currently beyond the reach of existing methods, and our approach extends these existing approaches to perform NGD updates with a simple update which can also be implemented efficiently in some cases.
Our current proposal is limited to a certain class of approximations, and further work is needed to generalize our results to many other types of structured approximations. The minimality condition we proposed uses one-to-one mappings of the expectation parameterization. We believe that this condition can be relaxed which will enable simple NGD update for many types of approximations.

Our main focus has been on the derivation of simple updates. We have presented examples where the updates can also be implemented efficiently. There are however implementation bottlenecks in existing software frameworks to implement some of the reparameterization tricks used in our algorithms. 
It is important to find ways to enable efficient updates by modifying the existing software frameworks. Another issue is that the NGD update needs to make sure that the parameters stay inside $\Omega$, and this issue deserves further exploration.
Another important direction is to apply structured approximations to large problems, especially those involving deep networks. We hope to perform such extensive experiments in the future to establish the benefits obtained by our NGD update for Bayesian deep learning.

\section*{Acknowledgements}
We would like to thank Voot Tangkaratt for useful discussions during this work.
WL is supported by a UBC International Doctoral Fellowship.

\bibliography{refs}

\begin{thebibliography}{69}
\providecommand{\natexlab}[1]{#1}
\providecommand{\url}[1]{\texttt{#1}}
\expandafter\ifx\csname urlstyle\endcsname\relax
  \providecommand{\doi}[1]{doi: #1}\else
  \providecommand{\doi}{doi: \begingroup \urlstyle{rm}\Url}\fi

\bibitem[Andrews \& Mallows(1974)Andrews and Mallows]{andrews1974scale}
Andrews, D.~F. and Mallows, C.~L.
\newblock {Scale mixtures of normal distributions}.
\newblock \emph{Journal of the Royal Statistical Society. Series B
  (Methodological)}, pp.\  99--102, 1974.

\bibitem[Arellano-Valle et~al.(2013)Arellano-Valle, Contreras-Reyes, and
  Genton]{arellano2013shannon}
Arellano-Valle, R.~B., Contreras-Reyes, J.~E., and Genton, M.~G.
\newblock {Shannon Entropy and Mutual Information for Multivariate
  Skew-Elliptical Distributions}.
\newblock \emph{Scandinavian Journal of Statistics}, 40\penalty0 (1):\penalty0
  42--62, 2013.

\bibitem[Azzalini(2005)]{azzalini2005skew}
Azzalini, A.
\newblock {The skew-normal distribution and related multivariate families}.
\newblock \emph{Scandinavian Journal of Statistics}, 32\penalty0 (2):\penalty0
  159--188, 2005.

\bibitem[Azzalini(2013)]{azzalini2013skew}
Azzalini, A.
\newblock \emph{{The skew-normal and related families}}, volume~3.
\newblock Cambridge University Press, 2013.

\bibitem[Azzalini \& Valle(1996)Azzalini and Valle]{azzalini1996multivariate}
Azzalini, A. and Valle, A.~D.
\newblock {The multivariate skew-normal distribution}.
\newblock \emph{Biometrika}, 83\penalty0 (4):\penalty0 715--726, 1996.

\bibitem[Balakrishnan \& Kundu(2019)Balakrishnan and
  Kundu]{balakrishnan2019birnbaum}
Balakrishnan, N. and Kundu, D.
\newblock {Birnbaum-Saunders distribution: A review of models, analysis, and
  applications}.
\newblock \emph{Applied Stochastic Models in Business and Industry},
  35\penalty0 (1):\penalty0 4--49, 2019.

\bibitem[Barndorff-Nielsen(1997)]{barndorff1997normal}
Barndorff-Nielsen, O.~E.
\newblock {Normal inverse Gaussian distributions and stochastic volatility
  modelling}.
\newblock \emph{Scandinavian Journal of statistics}, 24\penalty0 (1):\penalty0
  1--13, 1997.

\bibitem[Batir(2005)]{batir2005some}
Batir, N.
\newblock {Some new inequalities for gamma and polygamma functions}.
\newblock \emph{J. Inequal. Pure Appl. Math}, 6\penalty0 (4):\penalty0 1--9,
  2005.

\bibitem[Birnbaum \& Saunders(1969)Birnbaum and Saunders]{birnbaum1969new}
Birnbaum, Z.~W. and Saunders, S.~C.
\newblock {A new family of life distributions}.
\newblock \emph{Journal of applied probability}, 6\penalty0 (2):\penalty0
  319--327, 1969.

\bibitem[Bonnet(1964)]{bonnet1964transformations}
Bonnet, G.
\newblock {Transformations des signaux al{\'e}atoires a travers les systemes
  non lin{\'e}aires sans m{\'e}moire}.
\newblock In \emph{Annales des T{\'e}l{\'e}communications}, volume~19, pp.\
  203--220. Springer, 1964.

\bibitem[Caron \& Doucet(2008)Caron and Doucet]{caron2008sparse}
Caron, F. and Doucet, A.
\newblock {Sparse Bayesian nonparametric regression}.
\newblock In \emph{Proceedings of the 25th international conference on Machine
  learning}, pp.\  88--95. ACM, 2008.

\bibitem[Carr \& Madan(2009)Carr and Madan]{carr2009saddlepoint}
Carr, P. and Madan, D.
\newblock {Saddlepoint methods for option pricing}.
\newblock \emph{The Journal of Computational Finance}, 13\penalty0
  (1):\penalty0 49, 2009.

\bibitem[Choi et~al.(2018)Choi, Du, and Song]{choi2018high}
Choi, J., Du, Y., and Song, Q.
\newblock {Inverse Gaussian quadrature and finite normal-mixture approximation
  of generalized hyperbolic distribution}.
\newblock \emph{arXiv preprint arXiv:1810.01116}, 2018.

\bibitem[Contreras-Reyes \& Arellano-Valle(2012)Contreras-Reyes and
  Arellano-Valle]{contreras2012kullback}
Contreras-Reyes, J.~E. and Arellano-Valle, R.~B.
\newblock {Kullback--Leibler divergence measure for multivariate skew-normal
  distributions}.
\newblock \emph{Entropy}, 14\penalty0 (9):\penalty0 1606--1626, 2012.

\bibitem[Culham(2004)]{Richard-notes}
Culham, R.
\newblock {Lecture Notes: Advance differential equations and special
  functions}.
\newblock \url{www.mhtlab.uwaterloo.ca/courses/me755/web_chap4.pdf}, 2004.
\newblock Accessed: 2019/03/25.

\bibitem[Desmond(1986)]{desmond1986relationship}
Desmond, A.
\newblock {On the relationship between two fatigue-life models}.
\newblock \emph{IEEE Transactions on Reliability}, 35\penalty0 (2):\penalty0
  167--169, 1986.

\bibitem[Eltoft et~al.(2006)Eltoft, Kim, and Lee]{eltoft2006multivariate}
Eltoft, T., Kim, T., and Lee, T.-W.
\newblock {Multivariate scale mixture of Gaussians modeling}.
\newblock In \emph{International Conference on Independent Component Analysis
  and Signal Separation}, pp.\  799--806. Springer, 2006.

\bibitem[Feigin(1981)]{feigin1981conditional}
Feigin, P.~D.
\newblock {Conditional Exponential Families and a Representation Theorem for
  Asympotic Inference}.
\newblock \emph{The Annals of Statistics}, pp.\  597--603, 1981.

\bibitem[Figurnov et~al.(2018)Figurnov, Mohamed, and
  Mnih]{figurnov2018implicit}
Figurnov, M., Mohamed, S., and Mnih, A.
\newblock {Implicit Reparameterization Gradients}.
\newblock 2018.

\bibitem[Furmston \& Barber(2010)Furmston and Barber]{furmston2010variational}
Furmston, T. and Barber, D.
\newblock {Variational methods for reinforcement learning}.
\newblock In \emph{Proceedings of the Thirteenth International Conference on
  Artificial Intelligence and Statistics}, pp.\  241--248, 2010.

\bibitem[Genton(2004)]{genton2004skew}
Genton, M.~G.
\newblock \emph{{Skew-elliptical distributions and their applications: a
  journey beyond normality}}.
\newblock CRC Press, 2004.

\bibitem[Graves(2011)]{graves2011practical}
Graves, A.
\newblock {Practical variational inference for neural networks}.
\newblock In \emph{Advances in neural information processing systems}, pp.\
  2348--2356, 2011.

\bibitem[Grushka(1972)]{grushka1972characterization}
Grushka, E.
\newblock {Characterization of exponentially modified Gaussian peaks in
  chromatography}.
\newblock \emph{Analytical Chemistry}, 44\penalty0 (11):\penalty0 1733--1738,
  1972.

\bibitem[Gupta et~al.(2018)Gupta, Koren, and Singer]{gupta18shampoo}
Gupta, V., Koren, T., and Singer, Y.
\newblock {Shampoo: Preconditioned Stochastic Tensor Optimization}.
\newblock In \emph{Proceedings of the 35th International Conference on Machine
  Learning}, pp.\  1842--1850, 2018.

\bibitem[Hensman et~al.(2012)Hensman, Rattray, and Lawrence]{hensman2012fast}
Hensman, J., Rattray, M., and Lawrence, N.~D.
\newblock {Fast variational inference in the conjugate exponential family}.
\newblock In \emph{Advances in neural information processing systems}, pp.\
  2888--2896, 2012.

\bibitem[Hensman et~al.(2013)Hensman, Fusi, and Lawrence]{hensman2013gaussian}
Hensman, J., Fusi, N., and Lawrence, N.~D.
\newblock {Gaussian processes for big data}.
\newblock \emph{arXiv preprint arXiv:1309.6835}, 2013.

\bibitem[Hoffman \& Blei(2015)Hoffman and Blei]{hoffman2015structured}
Hoffman, M.~D. and Blei, D.~M.
\newblock {Structured stochastic variational inference}.
\newblock In \emph{Artificial Intelligence and Statistics}, 2015.

\bibitem[Hoffman et~al.(2013)Hoffman, Blei, Wang, and
  Paisley]{hoffman2013stochastic}
Hoffman, M.~D., Blei, D.~M., Wang, C., and Paisley, J.
\newblock {Stochastic variational inference}.
\newblock \emph{The Journal of Machine Learning Research}, 14\penalty0
  (1):\penalty0 1303--1347, 2013.

\bibitem[Honkela et~al.(2007)Honkela, Tornio, Raiko, and
  Karhunen]{honkela2007natural}
Honkela, A., Tornio, M., Raiko, T., and Karhunen, J.
\newblock {Natural conjugate gradient in variational inference}.
\newblock In \emph{International Conference on Neural Information Processing},
  pp.\  305--314. Springer, 2007.

\bibitem[Honkela et~al.(2011)Honkela, Raiko, Kuusela, Tornio, and
  Karhunen]{Honkela:11}
Honkela, A., Raiko, T., Kuusela, M., Tornio, M., and Karhunen, J.
\newblock {Approximate {Riemannian} Conjugate Gradient Learning for Fixed-form
  Variational {Bayes}}.
\newblock \emph{Journal of Machine Learning Research (JMLR))}, 11:\penalty0
  3235--3268, 2011.

\bibitem[Johansen(1979)]{johansen1979introduction}
Johansen, S.
\newblock {Introduction to the Theory of Regular Exponential Famelies}.
\newblock 1979.

\bibitem[J{\o}rgensen et~al.(1991)J{\o}rgensen, Seshadri, and
  Whitmore]{jorgensen1991mixture}
J{\o}rgensen, B., Seshadri, V., and Whitmore, G.
\newblock {On the mixture of the inverse Gaussian distribution with its
  complementary reciprocal}.
\newblock \emph{Scandinavian Journal of Statistics}, pp.\  77--89, 1991.

\bibitem[Khan \& Lin(2017)Khan and Lin]{khan2017conjugate}
Khan, M. and Lin, W.
\newblock Conjugate-computation variational inference: Converting variational
  inference in non-conjugate models to inferences in conjugate models.
\newblock In \emph{Artificial Intelligence and Statistics}, pp.\  878--887,
  2017.

\bibitem[Khan \& Nielsen(2018)Khan and Nielsen]{khan2018fast}
Khan, M.~E. and Nielsen, D.
\newblock {Fast yet Simple Natural-Gradient Descent for Variational Inference
  in Complex Models}.
\newblock \emph{arXiv preprint arXiv:1807.04489}, 2018.

\bibitem[Khan et~al.(2016)Khan, Babanezhad, Lin, Schmidt, and
  Sugiyama]{khan2016faster}
Khan, M.~E., Babanezhad, R., Lin, W., Schmidt, M., and Sugiyama, M.
\newblock {Faster stochastic variational inference using Proximal-Gradient
  methods with general divergence functions}.
\newblock In \emph{Proceedings of the Thirty-Second Conference on Uncertainty
  in Artificial Intelligence}, pp.\  319--328. AUAI Press, 2016.

\bibitem[Khan et~al.(2018)Khan, Nielsen, Tangkaratt, Lin, Gal, and
  Srivastava]{khan18a}
Khan, M.~E., Nielsen, D., Tangkaratt, V., Lin, W., Gal, Y., and Srivastava, A.
\newblock {Fast and Scalable {B}ayesian Deep Learning by Weight-Perturbation in
  {A}dam}.
\newblock In \emph{Proceedings of the 35th International Conference on Machine
  Learning}, pp.\  2611--2620, 2018.

\bibitem[Kingma \& Welling(2013)Kingma and Welling]{kingma2013auto}
Kingma, D.~P. and Welling, M.
\newblock {Auto-encoding variational bayes}.
\newblock \emph{arXiv preprint arXiv:1312.6114}, 2013.

\bibitem[Kotz \& Nadarajah(2004)Kotz and Nadarajah]{kotz2004multivariate}
Kotz, S. and Nadarajah, S.
\newblock \emph{{Multivariate t-distributions and their applications}}.
\newblock Cambridge University Press, 2004.

\bibitem[Kumar \& Tsvetkov(2018)Kumar and Tsvetkov]{kumar2018mises}
Kumar, S. and Tsvetkov, Y.
\newblock {Von Mises-Fisher Loss for Training Sequence to Sequence Models with
  Continuous Outputs}.
\newblock \emph{arXiv preprint arXiv:1812.04616}, 2018.

\bibitem[Liang et~al.(2009)Liang, Jordan, and Klein]{liang2009learning}
Liang, P., Jordan, M.~I., and Klein, D.
\newblock {Learning from measurements in exponential families}.
\newblock In \emph{Proceedings of the 26th annual international conference on
  machine learning}, pp.\  641--648. ACM, 2009.

\bibitem[Lin et~al.(2019)Lin, Khan, and Schmidt]{wu-report}
Lin, W., Khan, M.~E., and Schmidt, M.
\newblock {Stein's Lemma for the Reparameterization Trick with
  Exponential-family Mixtures}.
\newblock \emph{arXiv preprint arXiv:1910.13398}, 2019.

\bibitem[Lindsey(1996)]{lindsey1996parametric}
Lindsey, J.~K.
\newblock \emph{{Parametric statistical inference}}.
\newblock Oxford University Press, 1996.

\bibitem[Louizos \& Welling(2016)Louizos and Welling]{louizos2016structured}
Louizos, C. and Welling, M.
\newblock {Structured and efficient variational deep learning with matrix
  gaussian posteriors}.
\newblock In \emph{International Conference on Machine Learning}, pp.\
  1708--1716, 2016.

\bibitem[Martens \& Grosse(2015)Martens and Grosse]{martens2015optimizing}
Martens, J. and Grosse, R.
\newblock {Optimizing neural networks with {K}ronecker-factored approximate
  curvature}.
\newblock In \emph{International Conference on Machine Learning}, pp.\
  2408--2417, 2015.

\bibitem[Mishkin et~al.(2018)Mishkin, Kunstner, Nielsen, Schmidt, and
  Khan]{mishkin2018slang}
Mishkin, A., Kunstner, F., Nielsen, D., Schmidt, M., and Khan, M.~E.
\newblock {SLANG: Fast Structured Covariance Approximations for Bayesian Deep
  Learning with Natural Gradient}.
\newblock In \emph{Advances in Neural Information Processing Systems}, pp.\
  6246--6256, 2018.

\bibitem[Murphy(2013)]{murphy2013machine}
Murphy, K.~P.
\newblock \emph{{Machine learning : a probabilistic perspective}}.
\newblock MIT Press, Cambridge, Mass. [u.a.], 2013.
\newblock ISBN 9780262018029 0262018020.

\bibitem[Nguyen et~al.(2017)Nguyen, Li, Bui, and Turner]{nguyen2017variational}
Nguyen, C.~V., Li, Y., Bui, T.~D., and Turner, R.~E.
\newblock {Variational Continual Learning}.
\newblock \emph{arXiv preprint arXiv:1710.10628}, 2017.

\bibitem[Oh et~al.(2019)Oh, Adamczewski, and Park]{oh2019radial}
Oh, C., Adamczewski, K., and Park, M.
\newblock {Radial and Directional Posteriors for Bayesian Neural Networks}.
\newblock \emph{arXiv preprint arXiv:1902.02603}, 2019.

\bibitem[Ohlson et~al.(2013)Ohlson, Ahmad, and
  Von~Rosen]{ohlson2013multilinear}
Ohlson, M., Ahmad, M.~R., and Von~Rosen, D.
\newblock {The multilinear normal distribution: Introduction and some basic
  properties}.
\newblock \emph{Journal of Multivariate Analysis}, 113:\penalty0 37--47, 2013.

\bibitem[Opper \& Archambeau(2009)Opper and Archambeau]{opper2009variational}
Opper, M. and Archambeau, C.
\newblock {The variational Gaussian approximation revisited}.
\newblock \emph{Neural computation}, 21\penalty0 (3):\penalty0 786--792, 2009.

\bibitem[Price(1958)]{price1958useful}
Price, R.
\newblock {A useful theorem for nonlinear devices having Gaussian inputs}.
\newblock \emph{IRE Transactions on Information Theory}, 4\penalty0
  (2):\penalty0 69--72, 1958.

\bibitem[Ranganath et~al.(2014)Ranganath, Gerrish, and
  Blei]{ranganath2014black}
Ranganath, R., Gerrish, S., and Blei, D.
\newblock {Black box variational inference}.
\newblock In \emph{Artificial Intelligence and Statistics}, pp.\  814--822,
  2014.

\bibitem[Ranganath et~al.(2016)Ranganath, Tran, and
  Blei]{ranganath2016hierarchical}
Ranganath, R., Tran, D., and Blei, D.
\newblock {Hierarchical variational models}.
\newblock In \emph{International Conference on Machine Learning}, pp.\
  324--333, 2016.

\bibitem[Rezende et~al.(2014)Rezende, Mohamed, and
  Wierstra]{rezende2014stochastic}
Rezende, D.~J., Mohamed, S., and Wierstra, D.
\newblock {Stochastic backpropagation and approximate inference in deep
  generative models}.
\newblock \emph{arXiv preprint arXiv:1401.4082}, 2014.

\bibitem[Ruiz-Antol{\'\i}n \& Segura(2016)Ruiz-Antol{\'\i}n and
  Segura]{ruiz2016new}
Ruiz-Antol{\'\i}n, D. and Segura, J.
\newblock {A new type of sharp bounds for ratios of modified Bessel functions}.
\newblock \emph{Journal of Mathematical Analysis and Applications},
  443\penalty0 (2):\penalty0 1232--1246, 2016.

\bibitem[Salimans \& Knowles(2013)Salimans and Knowles]{salimans2013fixed}
Salimans, T. and Knowles, D.
\newblock {Fixed-form variational posterior approximation through stochastic
  linear regression}.
\newblock \emph{Bayesian Analysis}, 8\penalty0 (4):\penalty0 837--882, 2013.

\bibitem[Salimans et~al.(2015)Salimans, Kingma, and
  Welling]{salimans2015markov}
Salimans, T., Kingma, D., and Welling, M.
\newblock {Markov chain monte carlo and variational inference: Bridging the
  gap}.
\newblock In \emph{International Conference on Machine Learning}, pp.\
  1218--1226, 2015.

\bibitem[Salimbeni et~al.(2018)Salimbeni, Eleftheriadis, and
  Hensman]{salimbeni2018natural}
Salimbeni, H., Eleftheriadis, S., and Hensman, J.
\newblock {Natural Gradients in Practice: Non-Conjugate Variational Inference
  in Gaussian Process Models}.
\newblock \emph{International Conference on Artificial Intelligence and
  Statistics (AISTATS)}, 2018.

\bibitem[Sato(2001)]{sato2001online}
Sato, M.-A.
\newblock {Online model selection based on the variational Bayes}.
\newblock \emph{Neural computation}, 13\penalty0 (7):\penalty0 1649--1681,
  2001.

\bibitem[Staines \& Barber(2012)Staines and Barber]{staines2012variational}
Staines, J. and Barber, D.
\newblock {Variational optimization}.
\newblock \emph{arXiv preprint arXiv:1212.4507}, 2012.

\bibitem[Sun et~al.(2017)Sun, Chen, and Carin]{sun2017learning}
Sun, S., Chen, C., and Carin, L.
\newblock {Learning Structured Weight Uncertainty in Bayesian Neural Networks}.
\newblock In \emph{Proceedings of the 20th International Conference on
  Artificial Intelligence and Statistics, {AISTATS}}, pp.\  1283--1292, 2017.

\bibitem[Titsias \& L{\'a}zaro-Gredilla(2014)Titsias and
  L{\'a}zaro-Gredilla]{titsias2014doubly}
Titsias, M. and L{\'a}zaro-Gredilla, M.
\newblock {Doubly stochastic variational Bayes for non-conjugate inference}.
\newblock In \emph{International Conference on Machine Learning}, pp.\
  1971--1979, 2014.

\bibitem[Titsias \& Ruiz(2018)Titsias and Ruiz]{titsias2018unbiased}
Titsias, M.~K. and Ruiz, F.~J.
\newblock {Unbiased Implicit Variational Inference}.
\newblock \emph{arXiv preprint arXiv:1808.02078}, 2018.

\bibitem[Wainwright \& Jordan(2008)Wainwright and Jordan]{WainwrightJordan08}
Wainwright, M.~J. and Jordan, M.~I.
\newblock Graphical models, exponential families, and variational inference.
\newblock \emph{Foundations and Trends in Machine Learning}, 1--2:\penalty0
  1--305, 2008.

\bibitem[Wang et~al.(2018)Wang, Liu, and Liu]{wang2018variational}
Wang, D., Liu, H., and Liu, Q.
\newblock {Variational Inference with Tail-adaptive f-Divergence}.
\newblock In \emph{Advances in Neural Information Processing Systems}, pp.\
  5742--5752, 2018.

\bibitem[Xing et~al.(2002)Xing, Jordan, and Russell]{xing2002generalized}
Xing, E.~P., Jordan, M.~I., and Russell, S.
\newblock {A generalized mean field algorithm for variational inference in
  exponential families}.
\newblock In \emph{Proceedings of the Nineteenth conference on Uncertainty in
  Artificial Intelligence}, pp.\  583--591. Morgan Kaufmann Publishers Inc.,
  2002.

\bibitem[Yang \& Chu(2017)Yang and Chu]{yang2017approximating}
Yang, Z.-H. and Chu, Y.-M.
\newblock {On approximating the modified Bessel function of the second kind}.
\newblock \emph{Journal of Inequalities and Applications}, 2017\penalty0
  (1):\penalty0 41, 2017.

\bibitem[Yin \& Zhou(2018)Yin and Zhou]{yin2018semi}
Yin, M. and Zhou, M.
\newblock {Semi-Implicit Variational Inference}.
\newblock \emph{arXiv preprint arXiv:1805.11183}, 2018.

\bibitem[Zhang et~al.(2018)Zhang, Sun, Duvenaud, and Grosse]{zhang2017noisy}
Zhang, G., Sun, S., Duvenaud, D., and Grosse, R.
\newblock {Noisy Natural Gradient as Variational Inference}.
\newblock In \emph{Proceedings of the 35th International Conference on Machine
  Learning}, 2018.

\end{thebibliography}
\bibliographystyle{icml2019}


\appendix
\onecolumn
\begin{appendix}
   \section{Proof of Theorem \ref{thm:fimmcef}}
   \label{app:theorem1}
      In this section, we provide a proof for Theorem \ref{thm:fimmcef}:

\noindent\fbox{%
   \parbox{0.9\textwidth}{%
      For an MCEF representation given in Definition \ref{eq:mcef}, the FIM $\vfim_{\mix\lat}(\vvarpar)$ given in \eqref{eq:fimwz} is positive-definite and invertible for all $\vvarpar\in\Omega$.
      }
}

We prove this using a sequence of lemmas.
\begin{lemma}
\label{lemma:diff}
$\log q(\vmix,\vlat|\vvarpar)$ is twice differentiable with respect to $\vvarpar$.
\end{lemma}
\begin{proof}
   From Definition \ref{def:cef}, we see that the $\log q(\vlat,\vmix)$ is differentiable when $A_\mix(\vvarpar_\mix)$ and $A_\lat(\vvarpar_\lat, \vmix)$ are twice differentiable for each $\vmix$ sampled from $q(\vmix)$.
   Since $q(\vmix)$ is an EF, $A_\mix(\vvarpar_\mix)$ is twice differentiable~\cite{johansen1979introduction}.  Similarly, since conditioned on $\vmix$, $q(\vlat|\vmix)$ is also an EF,  $A_\lat(\vvarpar_\lat, \vmix)$ is twice differentiable too. Therefore, the log of the joint distribution is twice differentiable.
\end{proof}

\begin{lemma}
\label{lemma:fim}
The FIM $\vfim_{wz}(\vvarpar)$ is block-diagonal with two blocks: 
\begin{align}
   \vfim_{wz}(\vvarpar) = \begin{bmatrix}
   \vfim_{z}(\vvarpar) & \mathbf{0} \\
   \mathbf{0} & \vfim_{w}(\vvarpar_\mix)
   \end{bmatrix}, 
   \label{eq:blockdiag}
\end{align}
   where $\vfim_\mix(\vvarpar_\mix)$ is the FIM of $q(\vmix)$ and $\vfim_{\lat}(\vvarpar)$ is the expected of the FIM of $q(\vlat|\vmix)$ where expectation is taken under $q(\vmix)$ as shown below:
\begin{align*}
   \vfim_{\mix}(\vvarpar_\mix) &:= - \Unmyexpect{q(\mix)} \sqr{\nabla_{\varpar_\mix}^2 \log q(\vmix|\vvarpar_\mix)} \\
   \vfim_{\lat}(\vvarpar) &:= - \Unmyexpect{q(\mix)} \sqr{ \Unmyexpect{q(\lat|\mix)}  \sqr{\nabla_{\varpar_\lat}^2 \log q(\vlat|\vmix,\vvarpar_\lat)} },
\end{align*}
\end{lemma}
\begin{proof}
   By Lemma \ref{lemma:diff}, $\log q(\vmix,\vlat|\vvarpar)$ is twice differentiable, so the FIM is well defined. Below, we simplify the FIM to show that it has a block-diagonal structure. The first step below follows from the definition of the FIM. The second step is simply writing the FIM in a $2\times 2$ block corresponding to $\vvarpar_\lat$ and $\vvarpar_\mix$. In the third step, we write the joint as the product of $q(\vlat|\vmix)$ and
   $q(\vmix)$. The fourth step is obtained since the two blocks are separable in $\vvarpar_\lat$ and $\vvarpar_\mix$. In the fifth step, we take the expectation inside which give us the desired result in the last step.
\begin{align*}
 \vfim_{wz}(\vvarpar) &= -\Unmyexpect{q(\lat,\mix|\varpar)} \sqr{ \nabla_{\varpar}^2 \log q(\vlat,\vmix|\vvarpar) } \\
 &= -\Unmyexpect{q(\lat,\mix|\varpar)} \begin{bmatrix}
  \nabla_{\varpar_\lat}^2 \log q(\vlat,\vmix|\vvarpar) &  \nabla_{\varpar_\mix} \nabla_{\varpar_\lat^T} \log q(\vlat,\vmix|\vvarpar)\\
  \nabla_{\varpar_\lat} \nabla_{\varpar_\mix^T} \log q(\vlat,\vmix|\vvarpar) & \nabla_{\varpar_\mix}^2 \log q(\vlat,\vmix|\vvarpar) 
\end{bmatrix}  \\
 &= -\Unmyexpect{q(\lat,\mix|\varpar)} \begin{bmatrix}
  \nabla_{\varpar_\lat}^2 \rnd{ \log q(\vlat|\vmix,\vvarpar_\lat) +\log q(\vmix|\vvarpar_\mix)  } &  \nabla_{\varpar_\mix} \nabla_{\varpar_\lat^T}  \rnd{ \log q(\vlat|\vmix,\vvarpar_\lat) +\log q(\vmix|\vvarpar_\mix)  } \\
  \nabla_{\varpar_\lat} \nabla_{\varpar_\mix^T}  \rnd{ \log q(\vlat|\vmix,\vvarpar_\lat) +\log q(\vmix|\vvarpar_\mix)  }  & \nabla_{\varpar_\mix}^2  \rnd{ \log q(\vlat|\vmix,\vvarpar_\lat) +\log q(\vmix|\vvarpar_\mix)  } 
\end{bmatrix} \\ 
 &= -\Unmyexpect{q(\lat,\mix|\varpar)} \begin{bmatrix}
  \nabla_{\varpar_\lat}^2  \log q(\vlat|\vmix,\vvarpar_\lat)  &  \mathbf{0} \\
  \mathbf{0}  & \nabla_{\varpar_\mix}^2 \log q(\vmix|\vvarpar_\mix) 
\end{bmatrix} \\ 
 &= -\begin{bmatrix}
  \Unmyexpect{q(\lat,\mix|\varpar)} \sqr{ \nabla_{\varpar_\lat}^2  \log q(\vlat|\vmix,\vvarpar_\lat) } &  \mathbf{0} \\
  \mathbf{0}  & \Unmyexpect{q(\mix|\varpar_\mix)} \sqr{ \nabla_{\varpar_\mix}^2 \log q(\vmix|\vvarpar_\mix) } 
\end{bmatrix}  \\
&= \begin{bmatrix}
\vfim_{z}(\vvarpar) & \mathbf{0} \\
\mathbf{0} & \vfim_{w}(\vvarpar_\mix),
\end{bmatrix} 
\end{align*}
\end{proof}

\begin{lemma}
   The first block of the FIM matrix $\vfim_{\lat}$ is equal to the derivative of the expectation parameter $\vvarmean_\lat(\vvarpar)$:
   \begin{align*}
   \vfim_{\lat}(\vvarpar) := \nabla_{\varpar_\lat} \vvarmean_\lat^T(\vvarpar) 
   \end{align*}
   \label{lemma:fimwrtmean}
\end{lemma}

\begin{proof}
   We first show that the gradient of $A_\lat(\vvarpar_\lat,\vmix)$ with respect to $\vvarpar_\lat$ is equal to $\Unmyexpect{ q(\lat|\mix) } \sqr{ \vphi_\lat(\vlat,\vmix)  }$. By using the definition of $A_\lat(\vvarpar_\lat,\vmix)$, this is straightforward to show:
\begin{align}
\nabla_{\varpar_\lat} A_\lat(\vvarpar_\lat,\vmix) &= \nabla_{\varpar_\lat} \log \int  h_\lat(\vlat,\vmix)\exp\sqr{\myang{\vphi_\lat(\vlat,\vmix), \vvarpar_\lat}}  d\vlat \nonumber \\
&= \frac{ \int \nabla_{\varpar_\lat} h_\lat(\vlat,\vmix)\exp\sqr{\myang{\vphi_\lat(\vlat,\vmix), \vvarpar_\lat}}  d\vlat } { \int  h_\lat(\vlat,\vmix)\exp\sqr{\myang{\vphi_\lat(\vlat,\vmix), \vvarpar_\lat}}  d\vlat } \nonumber \\
&= \frac{ \int \vphi_\lat(\vlat,\vmix)  h_\lat(\vlat,\vmix)\exp\sqr{\myang{\vphi_\lat(\vlat,\vmix), \vvarpar_\lat}}  d\vlat } { \int  h_\lat(\vlat,\vmix)\exp\sqr{\myang{\vphi_\lat(\vlat,\vmix), \vvarpar_\lat}}  d\vlat } \label{eq:fim_psd1} \\
&= \Unmyexpect{ q(\lat|\mix) } \sqr{ \vphi_\lat(\vlat,\vmix)  } \label{eq:fim_psd2}
\end{align}
Using this, the expectation parameter $\vvarmean_\lat$ is simply the expected value of the gradient of the log-partition function.
\begin{align}
\vvarmean_\lat & =  \Unmyexpect{ q(\mix) q(\lat|\mix) } \sqr{ \vphi_\lat(\vlat,\vmix)  } = \Unmyexpect{ q(\mix)  } \sqr{ \nabla_{\varpar_\lat} A_\lat(\vvarpar_\lat,\vmix) } \label{eq:diff_cond_mean}
\end{align}
Using this, it is easy to show the result by simply using the definition of the conditional EF, as shown below:
\begin{align*}
\vfim_{\lat}(\vvarpar) &= -\Unmyexpect{q(\lat,\mix)} \sqr{\nabla_{\varpar_\lat}^2 \log q(\vlat|\vmix,\vvarpar_\lat)}\\
&= -\Unmyexpect{ q(\lat,\mix|\varpar) } \sqr{ \nabla_{\varpar_\lat}^2 \rnd{ \log h_\lat(\vlat,\vmix) + \myang{\vphi_\lat(\vlat,\vmix), \vvarpar_\lat} - A_\lat(\vvarpar_\lat,\vmix)} } \\
&= \Unmyexpect{ q(\lat,\mix|\varpar) } \sqr{ \nabla_{\varpar_\lat}^2  A_\lat(\vvarpar_\lat,\vmix)}  \\
&= \Unmyexpect{ q(\mix|\varpar_\mix) } \sqr{ \nabla_{\varpar_\lat}^2  A_\lat(\vvarpar_\lat,\vmix)} \\ 
&= \nabla_{\varpar_\lat} \Unmyexpect{ q(\mix|\varpar_\mix) } \sqr{ \nabla_{\varpar_\lat^T}  A_\lat(\vvarpar_\lat,\vmix)}  \\
&= \nabla_{\varpar_\lat} \vvarmean_\lat^T 
\end{align*}
\end{proof}


\begin{lemma}
\label{fim_mini_cond}
Let $\Omega_\mix \times \Omega_\lat$ be relatively open.
If the mapping $\vvarmean_\mix(\cdot):  \Omega_\mix \to \mathcal{M}_\mix$ is one-to-one, and, 
given every $\vvarpar_\mix \in \Omega_\mix$, the conditional mapping $\vvarmean_{\lat}(\cdot,\vvarpar_\mix): \Omega_\lat  \to \mathcal{M}_\lat$ is one-to-one,
then $\vfim_{\mix\lat}(\vvarpar)$ is positive-definite in $\Omega_\lat\times\Omega_\mix$.
\end{lemma}
\begin{proof}


   When the mapping  $\vvarmean_\mix$  is one-to-one, $q(\vmix|\vvarpar_\mix)$ is a minimal EF, and given that $\Omega_\mix$ is open, using the result discussed in Section \ref{sec:ngvi}, we conclude that the second block $\vfim_{\mix}(\vvarpar_\mix)$ of $\vfim_{\mix\lat}(\vvarpar)$ given in \eqref{eq:blockdiag} 
   is positive definite and invertible for all $\Omega_\lat$. Now we prove that the first block $\vfim_\lat(\vvarpar)$ is also positive definite.
 
   The steps below establish the positive-\emph{semi} definiteness first. The first step is simply the definition of the FIM, while the second step is obtained by using the fact that $\nabla \log f(\vvarpar) = \nabla f(\vvarpar)/ f(\vvarpar)$. The third step is obtained by using the chain-rule, and the fourth step simply uses the log-trick above to simply the second term. In the fifth step, we take the derivative out of the first term which cancels out
   $q(\vlat|\vmix,\vvarpar_\lat)$. The last step is straightfoward since the outer products are always nonnegative. 
\begin{align}
  \nabla_{\varpar_\lat}^2 A_\lat(\vvarpar_\lat,\vmix) & = -\Unmyexpect{ q(\lat|\mix) } \sqr{ \nabla_{\varpar_\lat}^2 \log q(\vlat|\vmix,\vvarpar_\lat) } , \nonumber \\
&= - \Unmyexpect{ q(\lat|\mix) } \sqr{\nabla_{\varpar_\lat} \left( \frac{ \nabla_{\varpar_\lat^T} q(\vlat|\vmix,\vvarpar_\lat)} { q(\vlat|\vmix,\vvarpar_\lat)} \right)} , \nonumber \\ 
   &= - \Unmyexpect{ q(\lat|\mix) } \sqr{ \frac{ \nabla_{\varpar_\lat}^2 q(\vlat|\vmix,\vvarpar_\lat)} { q(\vlat|\vmix,\vvarpar_\lat)}  - \frac{ \nabla_{\varpar_\lat} q(\vlat|\vmix,\vvarpar_\lat)} { q(\vlat|\vmix,\vvarpar_\lat)}  \frac{ \nabla_{\varpar_\lat^T} q(\vlat|\vmix,\vvarpar_\lat)} { q(\vlat|\vmix,\vvarpar_\lat)}  } \nonumber \\ 
&= \Unmyexpect{ q(\lat|\mix) } \sqr{ -\frac{ \nabla_{\varpar_\lat}^2 q(\vlat|\vmix,\vvarpar_\lat)} { q(\vlat|\vmix,\vvarpar_\lat)}}  + \Unmyexpect{ q(\lat|\mix) } \sqr{ \nabla_{\varpar_\lat} \log q(\vlat|\vmix,\vvarpar_\lat) \nabla_{\varpar_\lat^T} \log q(\vlat|\vmix,\vvarpar_\lat) }, \nonumber \\
&= \int  - \nabla_{\varpar_\lat}^2 q(\vlat|\vmix,\vvarpar_\lat) d\vlat  + \Unmyexpect{ q(\lat|\mix) } \sqr{ \nabla_{\varpar_\lat} \log q(\vlat|\vmix,\vvarpar_\lat) \nabla_{\varpar_\lat^T} \log q(\vlat|\vmix,\vvarpar_\lat) }, \nonumber \\
&= \underbrace{ - \nabla_{\varpar_\lat}^2 \overbrace{\int q(\vlat|\vmix,\vvarpar_\lat)  d\vlat}^{=1} }_{=0} + \Unmyexpect{ q(\lat|\mix) } \sqr{ \nabla_{\varpar_\lat} \log q(\vlat|\vmix,\vvarpar_\lat) \nabla_{\varpar_\lat^T} \log q(\vlat|\vmix,\vvarpar_\lat) }, \nonumber \\
&=  \Unmyexpect{ q(\lat|\mix) } \sqr{ \nabla_{\varpar_\lat} \log q(\vlat|\vmix,\vvarpar_\lat) \nabla_{\varpar_\lat^T} \log q(\vlat|\vmix,\vvarpar_\lat) } \succeq \mathbf{0} \label{ieq:conconvex}.
\end{align}
   Using Lemma \ref{lemma:fimwrtmean} and \eqref{eq:diff_cond_mean}, we see that FIM is positive semi-definite:
   \begin{align*}
      \vfim_{\lat}(\vvarpar) = \nabla_{\varpar_\lat} \vvarmean_\lat^T = \nabla_{\varpar_\lat} \Unmyexpect{ q(\mix|\varpar_\mix)  } \sqr{ \nabla_{\varpar_\lat^T} A_\lat(\vvarpar_\lat,\vmix) } =\Unmyexpect{q(\mix)} \sqr{ \nabla_{\varpar_\lat}^2 A_\lat(\vvarpar_\lat,\vmix) } \succeq \mathbf{0}
   \end{align*}
   Now, we prove the final claim that, for every $\vvarpar_\mix \in \Omega_\mix$, if the conditional mapping $\vvarmean_{\lat}(\cdot,\vvarpar_\mix)$ is one-to-one, then
$\vfim_{\lat}(\vvarpar)$ is positive definite.
   We will prove this statement by contradiction.
Suppose there exists $\vvarpar$ such that $\vfim_{\lat}(\vvarpar)$ is positive semi-definite, since $\vfim_{\lat}(\vvarpar)$ is positive semi-definite,
   there exists a non-zero vector $\va$ such that $\va^T \vfim_{\lat}(\vvarpar) \va = 0$. Simplifying below, we show that this leads to a contradiction. The first and second step are obtained by simply plugging \eqref{ieq:conconvex}, while the third step is obtained by using the definition of $q(\vlat|\vmix,\vvarpar_\lat)$ and the fourth step is obtained by using \eqref{eq:fim_psd2}. The last step is obtained by noting that the quantity is simply the variance of
   $\va^T\vphi_\lat(\vlat,\vmix)$ conditioned on $\vmix$ .  
\begin{align*}
   \va^T \vfim_\lat(\vvarpar) \va  &= \Unmyexpect{q(\mix)} \sqr{\va^T \nabla_{\varpar_\lat}^2 A_\lat(\vvarpar_\lat,\vmix) \va } \\
&= \Unmyexpect{ q(\mix) } \sqr{ \va^T \Unmyexpect{ q(\lat|\mix) } \crl{ \nabla_{\varpar_\lat} \log q(\vlat|\vmix,\vvarpar_\lat) \nabla_{\varpar_\lat^T} \log q(\vlat|\vmix,\vvarpar_\lat) } \va } \\
&= \Unmyexpect{ q(\mix) q(\lat|\mix) } \sqr{ \va^T \left( \vphi_\lat(\vlat,\vmix) - \nabla_{\varpar_\lat} A_\lat(\vvarpar_\lat,\vmix)  \right)  \left( \vphi_\lat(\vlat,\vmix) - \nabla_{\varpar_\lat} A_\lat(\vvarpar_\lat,\vmix)  \right)^T   \va } \\ 
&= \Unmyexpect{ q(\mix) q(\lat|\mix) } \sqr{ \va^T \left( \vphi_\lat(\vlat,\vmix) - \Unmyexpect {q(\lat|\mix) } \sqr{ \vphi_\lat(\vlat,\vmix)}  \right)  \left( \vphi_\lat(\vlat,\vmix) - \Unmyexpect {q(\lat|\mix) } \sqr{ \vphi_\lat(\vlat,\vmix)} \right)^T   \va } \\ 
&= \Unmyexpect{ q(\mix)  } \Unmyvar{q(\lat|\mix)} \sqr{ \va^T\vphi_\lat(\vlat,\vmix) }
\end{align*}
The expectation of a function positive quantity is equal to zero only when each function value is equal to zero, therefore for the above to be zeros, we need 
   $\va^T\vphi_\lat(\vlat,\vmix)= 0$. However, as we show below, this is not possible since the representation $q(\vlat|\vmix)$ is minimal conditioned on $\vmix$.

   Since $\Omega_\lat$ is  open, there exists a small $\delta>0$ to always be able to obtain a perturbed version $\vvarpar_\lat' =  \vvarpar_\lat + \delta \va$, such that $\vvarpar_\lat' \in \Omega_\lat$.
   Since the conditional mapping is one-to-one, $\vvarmean_{\lat}(\vvarpar_\lat', \vvarpar_\mix) \neq \vvarmean_{\lat}(\vvarpar_\lat, \vvarpar_\mix)$.
   By using \eqref{eq:diff_cond_mean} and \eqref{eq:fim_psd1}, when $\va^T\vphi_\lat(\vlat,\vmix)=0$, we get a contradiction:  
\begin{align}
 \vvarmean_{\lat}(\vvarpar_\lat', \vvarpar_\mix) & = \Unmyexpect{q(\mix|\varpar_\mix) } \sqr{ \nabla_{\varpar_\lat'} A_\lat(\vvarpar_\lat', \vmix) } \nonumber \\
&= \Unmyexpect{q(\mix|\varpar_\mix) } \sqr{ \frac{ \int \vphi_\lat(\vlat,\vmix)  h_\lat(\vlat,\vmix)\exp\sqr{\myang{\vphi_\lat(\vlat,\vmix), \vvarpar_\lat'}}  d\vlat } { \int  h_\lat(\vlat,\vmix)\exp\sqr{\myang{\vphi_\lat(\vlat,\vmix), \vvarpar_\lat'}}  d\vlat } } \label{eq:pd_prf1} \\
&= \Unmyexpect{q(\mix|\varpar_\mix) } \sqr{ \frac{ \int \vphi_\lat(\vlat,\vmix)  h_\lat(\vlat,\vmix)\exp\sqr{\myang{\vphi_\lat(\vlat,\vmix), \vvarpar_\lat}}  d\vlat } { \int  h_\lat(\vlat,\vmix)\exp\sqr{\myang{\vphi_\lat(\vlat,\vmix), \vvarpar_\lat}}  d\vlat } } \label{eq:pd_prf2} \\
&=\vvarmean_{\lat}(\vvarpar_\lat, \vvarpar_\mix) \nonumber
\end{align} where we can 
 move from \eqref{eq:pd_prf1} to \eqref{eq:pd_prf2},
since  $\myang{\vphi_\lat(\vlat,\vmix), \vvarpar_\lat'} = \myang{\vphi_\lat(\vlat,\vmix), \vvarpar_\lat+ \delta \va}=\myang{\vphi_\lat(\vlat,\vmix), \vvarpar_\lat}$, 
   Due to the contradiction,  $\vfim_{\lat}(\vvarpar)$  must be positive definite. This proves that both the blocks are positive definite and invertible.

\end{proof}

\begin{lemma}
   The gradient with respect to $\vvarpar$ can be expressed as the gradient with respect to $\vvarmean$:
   \begin{align}
      \nabla_{\varpar}\elbofinal &= \sqr{\nabla_{\varpar} \vvarmean^T} \nabla_{\varmean} \elbofinal = \sqr{\vfim_{wz}(\vvarpar)}  \nabla_{\varmean}\elbofinal 
      \label{eq:ng_update_fim}
   \end{align}
   \label{lemma:ng_update_fim}
\end{lemma}
\begin{proof}
   Using Lemma \ref{lemma:fimwrtmean}, and chain rule, we can establish the results for $\vvarpar_\lat$:
   \begin{align*}
      \nabla_{\varpar_\lat} \elbofinal = \nabla_{\varpar_\lat} \vvarmean_\lat^T(\vvarpar) \sqr{ \nabla_{\varmean_\lat} \elbofinal } = \vfim_{\lat}(\vvarpar) \sqr{ \nabla_{\varmean_\lat} \elbofinal } 
   \end{align*}
   For $\vvarpar_\mix$, this result holds trivially, which proves the statement.
\end{proof}

   \section{Finite Mixture of Gaussians}
   \label{app:finitemix}
      The finite mixture of EF distribution has the following conditional distribution $q(\vlat|\mix)$:
\begin{align*}
q(\vlat|\mix) &= \sum_{c=1}^{K} \mathbb{I}_c(\mix) q(\vlat|\vvarpar_{c}) =  \sum_{c=1}^{K} \mathbb{I}_c(\mix) h_{\lat}(\vlat)\exp \sqr{ \myang{ \vvarpar_{c},  \vphi_{\lat}(\vlat) } - A_{\lat}(\vvarpar_{c}) }  \\
&= h_{\lat}(\vlat)\exp \crl{ \sum_{c=1}^{K} \myang{ \mathbb{I}_c(\mix) \vphi_{\lat}(\vlat) , \vvarpar_{c} } -  \sum_{c'=1}^{K}\mathbb{I}_{c'}(\mix) A_{\lat}(\vvarpar_{c'}) } 
\end{align*} where we assume each component admits the same parametric form with distinct parameter $\vvarpar_c$.
For a mixture using distinct parametric forms with tied parameters, please see Appendix \ref{app:BS}.

From the above expression and using the EF form for the multinomial distribution, we can write the sufficient statistics, natural parameters, and expectation parameters as shown below, where $\vvarmean_c:=\Unmyexpect{q(\lat|\mix=c)}\sqr{ \vphi_\lat(\vlat) }$ is the expectation parameter of a component $q(\vlat|\mix=c)$.
\begin{equation*}
\left[
\begin{array}{c}
\mathbb{I}_1(w)\\
\mathbb{I}_2(w)\\
\vdots\\
\mathbb{I}_{K-1}(w)\\
\mathbb{I}_{1}(w)\vphi_\lat(\vlat)\\
\mathbb{I}_{2}(w)\vphi_\lat(\vlat)\\
\vdots\\
\mathbb{I}_{K}(w)\vphi_\lat(\vlat)\\
\end{array}
\right]
\quad
\left[
\begin{array}{c}
\log (\pi_1/\pi_K)\\
\log (\pi_2/\pi_K)\\
\vdots\\
\log (\pi_{K-1}/\pi_K)\\
\vvarpar_1\\
\vvarpar_2\\
\vdots\\
\vvarpar_K\\
\end{array}
\right]
\quad
\left[
\begin{array}{c}
\pi_1\\
\pi_2\\
\vdots\\
\pi_{K-1}\\
\pi_1 \vvarmean_1\\
\pi_2 \vvarmean_2\\
\vdots\\
\pi_K \vvarmean_K\\
\end{array}
\right]
\end{equation*}
From the last two vectors, we can see that the mapping between $\vvarpar$ and $\vvarmean$ is going to be one-to-one, when each EF $q(\vlat|\vvarpar_c)$ is minimal (which makes sure that mapping $\vvarpar_c$ and $\vvarmean_c$ is one-to-one), and all $\vvarpar_c$ are distinct.

\subsection{The Model and ELBO}
\label{app:fmg_elbo}
We consider the following model: $p(\mathcal{D},\vlat) = \prod_{n=1}^{N} p(\mathcal{D}_n|\vlat)p(\vlat)$. We approximate the posterior by using the finite mixture of EFs whose marginal is denoted by $q(\vlat)$ as given in \eqref{eq:finitemixexp1}. 
The variational lower bound is given by the following: 
\begin{align*}
\elbofinal (\vvarpar ) &=  \sqr{ \sum_{n=1}^N \sqr{ \log p(\mathcal{D}_n|\vlat)} +  \log \frac{p(\vlat)}{q(\vlat)} } \\
&= \Unmyexpect{q(\lat)}\sqr{ - h(\vlat)}, \textrm{ where } h(\vlat)  :=-\sqr{  \log \frac{p(\vlat)}{q(\vlat)} + \sum_{n} \log p(\mathcal{D}_n|\vlat)   } .
\end{align*}
Note that the lower bound is defined with the marginal $q(\vlat)$ and the variable $w$ is not part of the model but only the variational approximation $q(\vlat,\mix)$.

\subsection{Finite Mixture of Gaussians Approximation}
\label{sec:fmGauss}
We now give details about the NGD update for finite mixture of Gaussians. Note that the NGD update for $\vvarpar_\lat$ and $\vvarpar_\mix$ can be computed separately since the FIM is block-diagonal. We first derive the NGD update for each component $q(\vlat|\mix=c)$, and then give an update for $\vvarpar_\mix$.

As shown in Table \ref{tab:examples}, the natural and expectation parameters of the $c$'th component is given as follows:
\begin{align*}
   \vLambda_c := - \half \vSigma_c^{-1} &\quad\quad \vM_c := \pi_c (\vmu_c\vmu_c^T + \vSigma_c) \\
   \vlambda_c := \vSigma_c^{-1} \vmu_c &\quad\quad \vm_c := \pi_c \vmu_c
\end{align*}

The expectation parameters $\vm_c$ and $\vM_c$ are functions of $\pi_c,\vmu_c,\vSigma_c$ and its gradient can be obtained in terms of the gradient with respect to these quantities by using the chain rule. The final expressions are shown below: 
\begin{align*}
\nabla_{m_{c}}  \elbofinal  &= \frac{1}{\pi_c} \left( \nabla_{\mu_c}  \elbofinal  - 2 \sqr{\nabla_{\Sigma_c}  \elbofinal} \vmu_c \right) \\
\nabla_{M_{c}} \elbofinal &= \frac{1}{\pi_c}\left( \nabla_{\Sigma_c} \elbofinal \right) 
\end{align*}

We can compute gradients with respect to $\vmu_c$ and $\vSigma_c$ by using the gradient and Hessian of $h(\vlat)$ at a sample $\vlat$ from $q(\vlat,\mix)$. This can be done by using the Bonnet's and Price's theorems \citep{bonnet1964transformations,price1958useful,opper2009variational, rezende2014stochastic}.
\citet{staines2012variational} and \citet{wu-report}
discuss the conditions of the target function $h(\vlat)$ when it comes to applying these theorems.
Firstly, we define $\delta_{c} = q(\vlat|\mix=c)/q(\vlat) := \gauss(\vlat|\vmu_c,\vSigma_c) / \sum_{c'=1}^K \pi_{c'} \gauss(\vlat|\vmu_{c'},\vSigma_{c'}) $.

Notice that
\begin{align*}
  \nabla_{\mu_c}   \elbofinal (\vvarpar ) &=  \nabla_{\mu_c}   \Unmyexpect{q(\lat)} \sqr{ -  h(\vlat) }  
  =  \Unmyexpect{ q(\lat|\mix=c)} \sqr{ -  \pi_c \nabla_{\lat} h(\vlat) }
  =  \Unmyexpect{ q(\lat)} \sqr{ - \frac{ q(\vlat|\mix=c) \pi_c}{ q(\vlat) } \nabla_{\lat} h(\vlat) } 
  =  \Unmyexpect{ q(\lat)} \sqr{ - \frac{ q(\vlat|\mix=c) q(\mix=c) }{ q(\vlat) } \nabla_{\lat} h(\vlat) }
\end{align*}

Using these theorems, we obtain the following stochastic-gradient estimations for the mixture of Gaussians:
\begin{align*}
  \nabla_{\mu_c}   \elbofinal (\vvarpar )
=& -  \Unmyexpect{q(\lat)} \sqr{  q(\mix=c|\vlat)  \nabla_{\lat} h(\vlat)  } \quad \approx - \pi_c \delta_c \nabla_\lat h(\vlat) \\
  \nabla_{\Sigma_c}   \elbofinal (\vvarpar )
=& - \Unmyexpect{q(\lat)} \sqr{  q(\mix=c|\vlat)  \nabla_{\lat}^2 h(\vlat)  } \quad \approx - \frac{\pi_c \delta_c}{2} \nabla_\lat^2 h(\vlat) . 
\end{align*}
where $q(\mix=c|\vlat)=\pi_c \delta_c$ and $\vlat$ is sampled from $q(\vlat)$.

For $\vLambda_c$, we can then plug these gradient estimations into the natural-gradient update and obtain the following update:
\begin{align*}
   -\half \sqr{ \vSigma_c^{\textrm{(new)}} }^{-1} &\leftarrow -\half \vSigma_c^{-1} + \beta \nabla_{M_c} \elbofinal \quad\quad \Rightarrow \sqr{ \vSigma_c^{\textrm{(new)}} }^{-1} \leftarrow \vSigma_c^{-1} + \beta \delta_c   \nabla_z^2 h(\vlat) 
\end{align*}
Similarly, for $\mathbb{\lambda}_c$, we have the following expression:
\begin{align*}
   \sqr{\vSigma_c^{\textrm{(new)}} }^{-1}\vmu_c^{\textrm{(new)}} &\leftarrow  \vSigma_c^{-1} \vmu_c + \beta \nabla_{m_c} \elbofinal \\
    &\leftarrow  \vSigma_c^{-1} \vmu_c + \beta  \frac{1}{\pi_c} \left( \nabla_{\mu_c}  \elbofinal  - 2 \sqr{\nabla_{\Sigma_c}  \elbofinal} \vmu_c \right)\\
    &\leftarrow  \sqr{ \vSigma_c^{-1}  - 2 \frac{\beta}{\pi_c} \sqr{\nabla_{\Sigma_c}  \elbofinal}}\vmu_c + \beta  \frac{1}{\pi_c} \left( \nabla_{\mu_c}  \elbofinal  \right)\\
    &\leftarrow  \sqr{ \vSigma_c^{-1}  +  \beta \delta_c \sqr{\nabla_\lat^2 h(\vlat)  } }\vmu_c + \beta  \frac{1}{\pi_c} \left( \nabla_{\mu_c}  \elbofinal  \right)\\
    &\leftarrow  \sqr{\vSigma_c^{\textrm{(new)}} }^{-1} \vmu_c + \beta  \frac{1}{\pi_c} \left( \nabla_{\mu_c}  \elbofinal  \right) 
\end{align*}
This gives the following update (by using the stochastic gradients):
\begin{align*}
    \vmu_c^{\textrm{(new)}} \leftarrow  \vmu_c - \beta \delta_c \vSigma_c^{\textrm{(new)}}  \nabla_z h(\vlat)
\end{align*}

%

\subsection{Natural Gradients for $q(\mix)$}
Now, we give the update for $q(\mix|\vvarpar_\mix)$.
Its natural parameter and expectation parameter are
\begin{align*}
\vvarpar_\mix &= \crl{ \log \frac{\pi_c}{\pi_K} }_{c=1}^{K-1}
\quad\quad\quad
\vvarmean_\mix = \crl{ \Unmyexpect{q(\mix)} \sqr{ \mathbb{I}_c(\mix) } }_{c=1}^{K-1} = \{ \pi_c \}_{c=1}^{K-1}
\end{align*}
To derive the gradients, we note that only $q(\vlat)$ depends on $\vpi$ since the model does not contain this as a parameter. Therefore, we need the gradient of the variational approximation which can be written as follows:
\begin{align*}
   \nabla_{\pi_c} q(\vlat) = \nabla_{\pi_c} \sum_{k=1}^K \pi_k q(\vlat|\mix=k) = q(\vlat|\mix=c) - q(\vlat|\mix=K).
\end{align*} where $q(\vlat|\mix=c)=q(\vlat|\vvarpar_c)=\gauss(\vlat|\vmu_c,\vSigma_c)$.
The second term appears in the above expression because  $\pi_K=1-\sum_{c=1}^{K-1} \pi_c$.

For the convenience of our derivation, we will separate the lower bound into terms that depend on $q(\vlat)$ and the rest of the terms, as shown below:
\begin{align*}
\nabla_{\pi_c} \elbofinal (\vvarpar ) &=\nabla_{\pi_c}   \Unmyexpect{q(\lat)}\sqr{  \sum_{n=1}^N \log p(\mathcal{D}_n|\vlat) + \log p(\vlat) - \log q(\vlat) }\\
&=  \int \underbrace{\nabla_{\pi_c} q(\vlat)}_{q(\boldsymbol{\lat}|\mix=c) - q(\boldsymbol{\lat}|\mix=K)}  \sqr{ \sum_{n=1}^N \log p(\mathcal{D}_n|\vlat) + \log p(\vlat)  -\log q(\vlat)}  d\vlat - \underbrace{ \int q(\vlat) \nabla_{\pi_c} \log q(\vlat) d\vlat}_{0} \nonumber \\
&= \int \sqr{ q(\vlat|\mix=c) - q(\vlat|\mix=K) } \sqr{ \sum_{n=1}^N \rnd{ \log p(\mathcal{D}_n|\vlat) +  \log \frac{p(\vlat)}{q(\vlat)} }} d\vlat   \\
&= \int q(\vlat) \sqr{ \frac{q(\vlat|\mix=c)}{q(\vlat)} - \frac{q(\vlat|\mix=K)}{q(\vlat)} } \sqr{  \sum_{n=1}^N \rnd{ \log p(\mathcal{D}_n|\vlat) + \log \frac{p(\vlat)}{q(\vlat)} }} d\vlat   \\
&= \Unmyexpect{q(z)} \Big[ \left( \delta_c - \delta_K \right) \big[  \underbrace{\sum_{n=1}^N \rnd{ \log p(\mathcal{D}_n|\vlat) +  \log \frac{p(\vlat)}{q(\vlat)} }}_{-h(\boldsymbol{\lat})} \big]  \Big]  \\
&\approx - (\delta_c - \delta_K) h(\vlat)
\end{align*}
where $\vlat$ is a sample from $q(\vlat)$.

Using this, we can perform the following NGD update:
\begin{align*}
   \log \rnd{ \pi_{c}/\pi_K} &\leftarrow  \log \rnd{ \pi_{c}/\pi_K}  - \beta  (\delta_{c} - \delta_{K})  h(\vlat)
\end{align*}

\subsection{Updates for Unnormalized Likelihoods}
When $p(\mathcal{D},\vlat)=\frac{\hat{p}(\mathcal{D},\vlat)}{\mathcal{C}}$ is unnormalized, where $\mathcal{C}$ is the normalizing constant, the NGD updates are still valid.
We define the following functions.
\begin{align*}
 \hat{h}(\vlat)&:=-\sqr{ \log  \hat{p}(\mathcal{D},\vlat) - \log q(\vlat)} \\
 h(\vlat)&:=\hat{h}(\vlat) + \log \mathcal{C} 
\end{align*}

It is easy to see that  the update for $q(\vlat|\mix)$ remains the same since $\nabla_\lat h(\vlat)=\nabla_\lat \hat{h}(\vlat)$ and $\nabla_\lat^2 h(\vlat)=\nabla_\lat^2 \hat{h}(\vlat)$.

For the update for $q(\mix)$, we use the following expression.
\begin{align*}
\nabla_{\pi_c} \elbofinal (\vvarpar ) &=\nabla_{\pi_c}   \Unmyexpect{q(\lat)}\sqr{ \log \hat{p}(\mathcal{D},\vlat) -  \log \mathcal{C}  - \log q(\vlat) }\\
&=  \int \underbrace{\nabla_{\pi_c} q(\vlat)}_{q(\boldsymbol{\lat}|\mix=c) - q(\boldsymbol{\lat}|\mix=K)}  \sqr{ \log \hat{p}(\mathcal{D},\vlat) -  \log \mathcal{C}   -\log q(\vlat)}  d\vlat - \underbrace{ \int q(\vlat) \nabla_{\pi_c} \log q(\vlat) d\vlat}_{0} \nonumber \\
&=  \int \left(q(\vlat|\mix=c) - q(\vlat|\mix=K)\right)  \sqr{  \log \hat{p}(\mathcal{D},\vlat)   -\log q(\vlat)}  d\vlat
- \underbrace{\int \left(q(\vlat|\mix=c) - q(\vlat|\mix=K)\right)   \log \mathcal{C} d\vlat}_{0} \\
&= \Unmyexpect{q(z)} \Big[ \left( \delta_c - \delta_K \right)  \rnd{ \log \hat{p}(\mathcal{D},\vlat) -  \log q(\vlat) }   \Big]  \\
&\approx - (\delta_c - \delta_K) \hat{h}(\vlat)
\end{align*} where the second term in the third step is $0$ since $\mathcal{C}$ does not depend on $\vlat$ and $\int q(\vlat|\mix=c) d\vlat=\int q(\vlat|\mix=K) d\vlat = 1$.

\subsection{Extension to Finite Mixture of EFs}
The algorithm presented in Section \ref{sec:finiteef} can be extended to handle generic minimal EF components. We now present a general gradient estimator to do so. The update of $\pi_c$ remains unaltered, so we do not discuss them here. We only discuss how to update natural parameters $\vvarpar_c$ of $q(\vlat|\vvarpar_c)$.

The natural parameter and sufficient statistics are $\vvarpar_{c}$  and $\mathbb{I}_c(\mix) \vphi_{\lat}(\vlat)$ respectively. We wish to perform the following update:
$\vvarpar_{c} \leftarrow \vvarpar_{c} + \beta_\lat \nabla_{\varmean_{c}} \elbofinal (\vvarpar ) $.
In general, we can compute the gradient $\nabla \varmean_{c} \elbofinal (\vvarpar ) $ by computing the FIM of each component as shown below:
\begin{align*}
\nabla_{\varmean_{c}} \elbofinal (\vvarpar ) &=  \left( \nabla_{\varpar_{c}} \vvarmean_{c} \right)^{-1} \nabla_{\varpar_{c}} \elbofinal (\vvarpar ) =   \left( \nabla_{\varpar_{c}} \Unmyexpect{q(\mix,\lat)}\sqr{ \mathbb{I}_c(\mix) \vphi_{\lat}(\vlat) }  \right)^{-1} \nabla_{\varpar_{c}}\elbofinal (\vvarpar ),
\end{align*}
Both of these gradients can be obtained given $\nabla_{\varpar_\lat} \vlat$, where $\vlat$ is a sample from $q(\vlat)$ as shown below.
\begin{align*}
 \nabla_{\varpar_{c}} \Unmyexpect{q(\mix,\lat)}\sqr{ \mathbb{I}_c(\mix) \vphi_\lat(\vlat) } =&  \nabla_{\varpar_{c}} \int \pi_c q(\vlat|\mix=c) \vphi_\lat(\vlat) d\vlat 
= \int \pi_c q(\vlat|\mix=c) \sqr{\nabla_{\lat} \vphi_\lat(\vlat)} \sqr{ \nabla_{\varpar_{c}} \vlat } d\vlat \\  
 \nabla_{\varpar_{c}}   \elbofinal (\vvarpar ) =&  \int \pi_c \nabla_{\varpar_{c}} q(\vlat|\mix=c)  \sqr{ - h(\vlat)} d \vlat + \underbrace{  \int   q(\vlat) \sqr{\nabla_{\varpar_{c}} \log q(\vlat) } d\vlat}_{0}  \\
=&  \int \pi_c q(\vlat|\mix=c) \sqr{\nabla_{\lat} \left( - h(\vlat) \right) } \sqr{ \nabla_{\varpar_{c}} \vlat} d\vlat 
\end{align*}
If we assume that $q(\lat|\mix)$ is an univariate continuous exponential family distribution, we can use the implicit re-parameterization trick \citep{salimans2013fixed,figurnov2018implicit} to get the following gradient. 
\citet{wu-report} discuss the trick under a weaker assumption than \citet{salimans2013fixed,figurnov2018implicit}.

\begin{align*}
 \nabla_{\varpar_{c}} \lat = - \frac{ \nabla_{\varpar_c} Q_c(\lat|\vvarpar_c) } { q(\lat|\mix=c) },
\end{align*} where $Q_c(\cdot|\vvarpar_c)$ is the cumulative distribution function (CDF) of $q(\lat|\mix=c)$.
Therefore, we now can compute the required gradient as below:
\begin{align*}
 \nabla_{\varpar_{c}}   \elbofinal (\vvarpar )
=&  \int \pi_c q(\lat|\mix=c) \sqr{ - \nabla_{\lat} h(\lat) } \sqr{ - \frac{ \nabla_{\varpar_c} Q_c(\lat|\vvarpar_c) } { q(\lat|\mix=c) }} d\lat \\
=&  \Unmyexpect{q(\lat)} \sqr{ \pi_c \frac{ q(\lat|\mix=c) }{ q(\lat) } \sqr{ - \nabla_{\lat}   h(\lat)  } \sqr{ - \frac{ \nabla_{\varpar_c} Q_c(\lat|\vvarpar_c) } { q(\lat|\mix=c) }} } \\
=&  \Unmyexpect{q(\lat)} \sqr{  \frac{ \pi_c \nabla_{\varpar_c} Q_c(\lat|\vvarpar_c) }{ q(\lat) } \sqr{\nabla_{\lat}  h(\lat)  }  },
\end{align*}
and 
\begin{align*}
 \nabla_{\varpar_{c}} \Unmyexpect{q(\mix,\lat)}\sqr{ \mathbb{I}_c(\mix) \vphi_\lat(\lat) }
= & \int \pi_c q(\vlat|\mix=c) \nabla_{\lat} \sqr{\vphi_\lat(\lat)} \sqr{  - \frac{ \nabla_{\varpar_c} Q_c(\lat|\vvarpar_c) } { q(\lat|\mix=c) } } d\vlat  \\
= & \Unmyexpect{q(\lat)} \sqr{  \frac{ -\pi_c \nabla_{\varpar_c} Q_c(\lat|\vvarpar_c)}{q(\lat)}  \nabla_{\lat} \sqr{\vphi_\lat(\lat)}  }.
\end{align*}
This is not the most efficient way to compute NGDs, however, for the specific cases (e.g., Gaussian, exponential, inverse Gaussian), we can get simplifications whenever gradients with respect to the expectation parameters are easy to compute.

The Birnbaum-Saunders distribution, which is a finite mixture of inverse Gaussians, is presented in Appendix \ref{app:BS}.
This example is different from examples given in this section since we allow each mixing component takes a distinct but tied parametric form.

\subsection{Result for the Toy Example}
\label{sec:fmGauss_plot}
\begin{figure*}[t]
	\centering
  \includegraphics[width=0.95\linewidth]{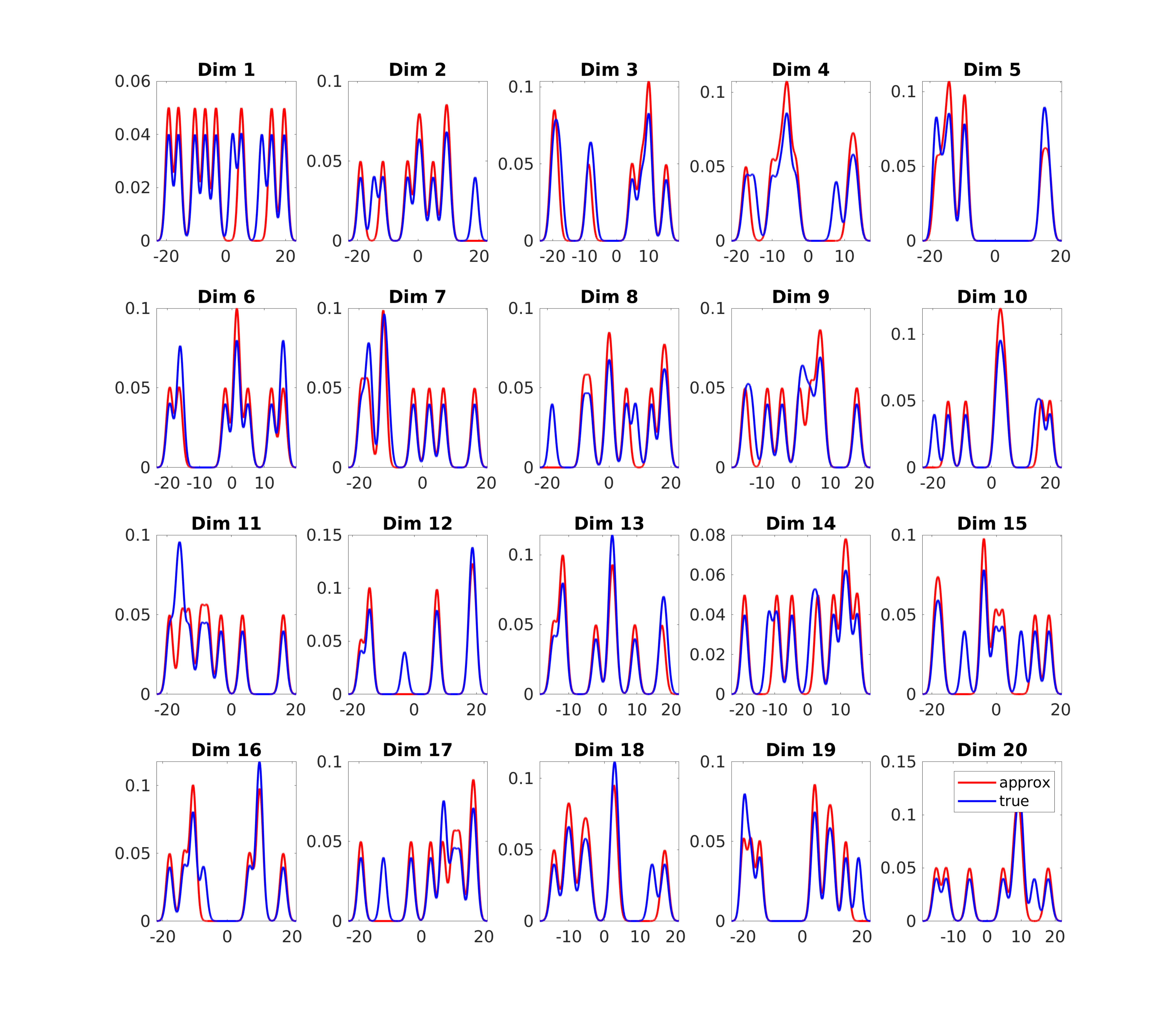}
   \caption{ This is a complete version of the leftmost figure in Figure \ref{figure:toy_examples}. The figure shows MOG approximation (with $K=20$) to fit an MOG model with 10 components in a 20 dimensional problem.}
	
	\label{figure:toy_examples_all}
\end{figure*}

See Figure \ref{figure:toy_examples_all}

   \section{Birnbaum-Saunders Distribution}
   \label{app:BS}
   Firstly, we denote the inverse Gaussian distribution by
$\IGauss(\lat;\mu,v)=\left( \frac{v}{2\pi \lat^3} \right)^{1/2} \exp\sqr{ -\frac{v \lat}{2\mu^2}  - \frac{v}{2\lat}  + \frac{v}{\mu}}$,
where  $\lat>0$, $v>0$, and $\mu>0$.
We consider the following mixture distribution.
\begin{align}
q(\mix) & = p^{\mathbb{I}(\mix=0)} (1-p)^{\mathbb{I}(\mix=1)} \nonumber \\
q(\lat|\mix) &= \mathbb{I}(\mix=0) \IGauss(\lat;\mu,v) + \mathbb{I}(\mix=1)\frac{\lat\IGauss(\lat;\mu,v)}{\mu}, \label{eq:bsdist}
\end{align} where $\frac{\lat}{\mu}\IGauss(\lat;\mu,v)$ is a normalized distribution since it is the distribution of $\lat=y^{-1}$ where $y$ is distributed by $\IGauss(y;\mu^{-1}, v/\mu^2)$.

As we can observe from Eq \eqref{eq:bsdist}, each mixing component has a distinct parametric form and variational parameters are shared between the components, which is different from examples discussed in Appendix \ref{app:finitemix}.
According to \citet{desmond1986relationship,jorgensen1991mixture,balakrishnan2019birnbaum},
the marginal distribution is known as
the Birnbaum-Saunders distribution \citep{birnbaum1969new} shown as below, where $p=\half$.
\begin{align*}
q(\lat|v,\mu)
&= \sum_{\mix} q(\mix) q(\lat|\mix)  \\
&= \half \crl{ \left( \frac{v}{2\pi \lat^3} \right)^{1/2} \exp\sqr{ -\frac{v \lat}{2\mu^2}  - \frac{v}{2\lat}  + \frac{v}{\mu} } \left(1+\frac{\lat}{\mu}\right) } \\
&= \frac{\sqrt{v}}{2 \sqrt{2\pi} }  \sqr{ \frac{1}{\lat^{3/2}} + \frac{1}{\mu \lat^{1/2}} } \exp\sqr{ -\frac{v \lat}{2\mu^2}  - \frac{v}{2\lat}  + \frac{v}{\mu} }  \\
&= \frac{\sqrt{v}}{2 \mu \sqrt{2\pi \mu}  } \sqr{ \left(\frac{\mu}{\lat}\right)^{1/2} + \left(\frac{\mu}{\lat}\right)^{3/2} } \exp\crl{ -\frac{v(\frac{\lat}{\mu}+\frac{\mu}{\lat}-2)}{2\mu}  }
\end{align*}

\begin{lemma}
   The joint distribution of the Birnbaum-Saunders distribution given in \eqref{eq:bsdist} can be written in a conditional EF form.
\end{lemma}

\begin{proof}
It is obvious that $q(\mix)$ is  Bernoulli distribution with $p=\half$, which is an EF distribution.
Now, we show that $q(\lat|\mix)$ is a conditional EF distribution as below.
\begin{align*}
q(\lat|\mix)
&= \exp\crl{ \mathbb{I}(\mix=0) \sqr{ -\frac{v \lat}{2\mu^2}  - \frac{v}{2\lat}  + \frac{v}{\mu} +\half \log\frac{v}{2\pi \lat^3} }  + \mathbb{I}(\mix=1) \sqr{ -\frac{v \lat}{2\mu^2}  - \frac{v}{2\lat}  + \frac{v}{\mu} +\half \log\frac{v}{2\pi \lat } -\log \mu }   } \\
&= \exp\crl{  -\frac{v \lat}{2\mu^2}  - \frac{v}{2\lat}  + \frac{v}{\mu} +\mathbb{I}(\mix=0) \sqr{\half \log\frac{v}{2\pi \lat^3} }  + \mathbb{I}(\mix=1) \sqr{\half \log\frac{v}{2\pi \lat } -\log \mu }   } \\
&= \frac{1}{\sqrt{2\pi}} \lat^{-3\mathbb{I}(\mix=0)/2 } \lat^{-\mathbb{I}(\mix=1)/2 } \exp\crl{  -\frac{v \lat}{2\mu^2}  - \frac{v}{2\lat}  + \frac{v}{\mu} + \half \log(v)  - \mathbb{I}(\mix=1) \log(\mu)  } 
\end{align*}
The natural parameters and sufficient statistics  are $\{ -\frac{v}{2\mu^2},  - \frac{v}{2} \}$ and $\{ \lat, \frac{1}{\lat} \}$ respectively.
\end{proof}

According to \citet{balakrishnan2019birnbaum}, the expectation parameters are
\begin{align*}
m_1 &=\Unmyexpect{q(\lat)} \sqr{ \lat} 
=  \mu +  \frac{\mu^2}{2v} \\
m_2 &= \Unmyexpect{q(\lat)} \sqr{ \lat^{-1}} 
=  \mu^{-1} +  \frac{1}{2v}
\end{align*}

The sufficient statistics, natural parameters, and expectation parameters are summarized below:
\begin{equation*}
\left[
\begin{array}{c}
\lat\\
\frac{1}{\lat}
\end{array}
\right]
\quad
\left[
\begin{array}{c}
-\frac{v}{2\mu^2}\\
- \frac{v}{2}
\end{array}
\right]
\quad
\left[
\begin{array}{c}
\mu +  \frac{\mu^2}{2v} \\
\mu^{-1} +  \frac{1}{2v}
\end{array}
\right]
\end{equation*}

\begin{lemma}
   The joint distribution  given in \eqref{eq:bsdist} is a minimal conditional-EF.
\end{lemma}

\begin{proof}
Since $\vvarpar_\mix$ is known in this case, we only need to show
there exists an one-to-one mapping between the natural parameter  and the expectation parameter.
Just by observing the parameters given above, we can see that there exists an one-to-one mapping between the natural parameter and $\crl{\mu,v}$ since $\mu>0$ and $v>0$.
Furthermore, we know that $m_1 m_2>1$ and $m_1>0$.
We can show that there also exists an one-to-one mapping between  $\crl{\mu,v}$ and the expectation parameter by noticing that 
\begin{align*}
\mu &= \sqrt{ m_1 /m_2 } \\
v &= \frac{1}{2 ( m_2 -\sqrt{m_2/m_1}) }
\end{align*}
Since one-to-one mapping is transitive, we know that mapping between natural and expectation parameters is one-to-one. Hence proved.
\end{proof}

Note that we can use the implicit re-parametrization trick to compute the gradient w.r.t. $\mu$ and $v$.
Furthermore, the expectation parameters $m_1$ and $m_2$ are functions of $\mu,v$ and the gradients can be obtained in terms of the gradient with respect to $\mu$ and $v$ by using the chain rule.

   \section{Student's t-distribution}
   \label{app:tdist}
      
\begin{lemma}
   The joint distribution $\gauss(\vlat|\vmu,\mix\vSigma)\mathcal{IG}(\mix|a,a)$, where $\vlat \in \mathcal{R}^d$, is a curved exponential family distribution. 
   \label{lemma:curvedTdist}
\end{lemma}

\begin{proof}
The joint-distribution can be expressed as a four-parameter exponential form as shown below:
 \begin{align*}
  \gauss(\vlat|\vmu,\mix\vSigma)\mathcal{IG}(\mix|a,a) 
    &=  \mathrm{det} \rnd{2\pi \mix\vSigma}^{-\half} \exp \{-\half (\vlat-\vmu)^T \left(\mix\vSigma\right)^{-1}  (\vlat-\vmu) \} \frac{a^a}{\Gamma(a)} \mix^{-a-1} \exp\{ - \frac{a}{\mix} \} \\
  &=  \rnd{2\pi \mix}^{-d/2} \mix^{-1}\exp \{  -\half (\vlat-\vmu)^T \left(\mix\vSigma\right)^{-1}  (\vlat-\vmu) -\half \log \mathrm{det}(\vSigma)  \\
  & \quad \quad \quad  \quad \quad  \quad - \frac{a}{\mix}  - a \log\mix -\rnd{ \log \Gamma(a) - a\log a } \} \\
  &=  \rnd{2\pi \mix}^{-d/2} \mix^{-1}\exp \{  \myang{ -\half \vSigma^{-1},  \mix \vlat \vlat^T } + \myang{ \vSigma^{-1}\vmu,  \mix \vlat  } + \myang{ -\half \vmu^T \vSigma^{-1} \vmu,  \mix^{-1}} \nonumber  \\
  & \quad \quad \quad  \quad \quad  \quad + \myang{-a, \mix^{-1} +\log\mix } -\sqr{ \log \Gamma(a) - a\log a + \half \log \mathrm{det}(\vSigma)} \} \\
  &=  \rnd{2\pi \mix}^{-d/2} \mix^{-1}\exp \{  \myang{ \vVarpar_1,  \mix \vlat \vlat^T } + \myang{ \vvarpar_2,  \mix \vlat  } + \myang{ \varpar_3,  \mix^{-1}}   \\
  & \quad \quad \quad  \quad \quad  + \myang{\varpar_4, \mix^{-1} +\log\mix } -\sqr{ \log \Gamma(-\varpar_4) + \varpar_4 \log (-\varpar_4) - \half \log \mathrm{det}( -2 \vVarpar_1 )} \} ,
  \end{align*}
   where the following are the natural parameters:
   \begin{align*}
   \vVarpar_1:= -\half \vSigma^{-1}, \quad  \vvarpar_2 := \vSigma^{-1}\vmu, \quad \varpar_3 := -\half \vmu^T \vSigma^{-1} \vmu, \quad \varpar_4:=-a
   \end{align*}
 We can see that $\varpar_3$ is fully determined by $\vVarpar_1$ and $\vvarpar_2$, i.e.,
 \begin{align*}
\varpar_3 &=-\half \vmu^T \vSigma^{-1} \vmu = -\half \rnd{ \vSigma \vSigma^{-1}\vmu }^T \rnd{ \vSigma^{-1} \vmu } \nonumber = -\half \rnd{ \rnd{ -2 \vVarpar_1 }^{-1} \vvarpar_2}^T \vvarpar_2 
 \end{align*}
   In a minimal representation we can specify all four parameters freely, but in this case we have less degree of freedom. Therefore, this is a curved EF representation.
 
\end{proof}

Instead of using the above 4 parameter form, we can write the distribution in the conditional EF form given in Definition \ref{def:cef}.

\begin{lemma}
   The joint distribution of  Student's t-distribution  given in \eqref{eq:studentt} can be written in a conditional EF form.
\end{lemma}
\begin{proof}
   We can rewrite the conditional $q(\vlat|\mix)$ in an EF-form as follows:
\begin{align*}
q(\vlat|\mix) &= \gauss(\vlat|\vmu, \mix\vSigma) \\
&= (2\pi)^{-d/2} \exp(  \crl{ \mathrm{Tr} \left( -\half\vSigma^{-1} \mix^{-1}\vlat\vlat^T \right) + \vmu^T \vSigma^{-1} \mix^{-1}\vlat- \half(  \mix^{-1} \vmu^T \vSigma^{-1}\vmu +     \log \mathrm{det}(\mix\vSigma)  } )
\end{align*}
   The sufficient statistics $\vphi_\lat(\vlat,\mix)=\crl{ \mix^{-1} \vlat, \mix^{-1}\vlat\vlat^T} $ . The natural parameter is $\vvarpar_\lat=\crl{ \vSigma^{-1}\vmu, -\half \vSigma^{-1} }$. Since  $q(\mix)$ is a inverse Gamma distribution, which is a EF distribution as shown below, the joint $q(\vlat,\mix)$ is a conditional EF. The EF form of the inverse gamma distribution is shown below: 
\begin{align*}
q(\mix|a,a) = \mix^{-1} \exp( -( \log(\mix) + \frac{1}{\mix}  )a  - ( \log \Gamma(a) - a\log(a) )  )
\end{align*}
We can read the sufficient statistics $\vphi_\mix(\mix)=-\log(\mix)-\frac{1}{\mix}$ and the natural parameter  $\varpar_\mix=a$ from this form.
\end{proof}

Using the fact that $\Unmyexpect{q(\mix|a)}\sqr{ 1/\mix } = a/a = 1$, and $\Unmyexpect{q(\mix|a)}\sqr{ \log\mix }=\log a-\psi(a)$, we can derive the 
 expectation parameter shown below:
\begin{align*}
   \vm &:= \Unmyexpect{q(\mix|a)q(\lat|\mu, \mix\Sigma) } \sqr{ \mix^{-1} \vlat } = \vmu ,\\
   \vM &:= \Unmyexpect{q(\mix|a)q(\lat|\mu, \mix\Sigma) } \sqr{ \mix^{-1} \vlat \vlat^T } = \vmu \vmu^T + \vSigma, \\
   m_a &:= \Unmyexpect{q(\mix|a)} \sqr{-\frac{1}{\mix} -\log(\mix)} = -1 - \log a + \psi(a)
\end{align*}
The sufficient statistics, natural parameters, and expectation parameters are summarized below:
\begin{equation*}
\left[
\begin{array}{c}
-1/\mix -\log \mix\\
\vlat/\mix\\
\vlat\vlat^T/\mix
\end{array}
\right]
\quad
\left[
\begin{array}{c}
a\\
\vSigma^{-1}\vmu\\
-\half \vSigma^{-1}
\end{array}
\right]
\quad
\left[
\begin{array}{c}
-1-\log a +\psi(a)\\
\vmu\\
\vmu\vmu^T + \vSigma
\end{array}
\right]
\end{equation*}
The following lemma shows that the Student's t-distribution is an MCEF obtained by establishing one-to-one mapping between natural and expectation parameters.

\begin{lemma}
   The joint distribution of Student's t-distribution shown in \eqref{eq:studentt} is a minimal conditional-EF.
   \label{lemma:tdistmcefproof}
\end{lemma}
\begin{proof}
The proof is rather simple. First we note that the expectation parameters for $q(\vlat|\mix)$ do not depend on $\vvarpar_\mix:= a$. In fact, the mapping between the last two natural and expectation parameter is one-to-one since they correspond to a Gaussian distribution which has a minimal representation.

The only thing remaining is to show that the mapping between $a$ and $\varmean_\mix(a) := -1-\log a +\psi(a)$ is one-to-one. Since $\nabla_a \varmean_\mix(a)$ is the Fisher information of $q(w)$, we can show this when $\nabla_a \varmean_\mix(a)>0$.  The gradient $\nabla_a \varmean_\mix(a) = \nabla_{a} \psi(a) -1/a$.
According to Eq. 1.4 in \citet{batir2005some}, we have $\nabla_{\alpha} \psi(\alpha) -  1/a - 1/(2a^2) >0$ when $a>0$.
Therefore, $\nabla_{a} \psi(a) - 1/a >0$ which establishes that the Fisher information is positive, therefore the distribution is a minimal EF. This completes the proof.
\end{proof}

\subsection{Derivation of the NGD Update}
Let's consider $q(\mix)=\mathcal{IG}(\mix|a,a)$ and $q(\vlat|\mix)=\gauss(\vlat|\vmu, \mix\vSigma)$.
We denote the log-likelihood for the $n$'th data point by $f_n(\vlat) :=-\log p(\data_n|\vlat)$ with a Student's t-prior on $\vlat$ expressed as a scale mixture of Gaussians
$p(\vlat,\mix)=\mathcal{IG}(\mix|a_0,a_0) \gauss(\vlat|\mathbf{0}, \mix \vI)$.
We use the lower bound defined in the joint-distribution $p(\data, \vlat, w)$: 
\begin{align*}
   \elbofinal (\vvarpar ) & = \Unmyexpect{q(\lat,\mix)} \sqr{\log p(\data,\vlat,\mix) - \log q(\data, \vlat, \mix)} \\ 
   & = \Unmyexpect{q(\lat,\mix)}\sqr{ \sum_{n=1}^{N} \underbrace{ \log p(\data_n|\vlat) }_{:= -f_n(\vlat)} +\log \frac{\gauss(\vlat| \mathbf{0}, \mix \vI)}{ \gauss(\vlat| \vmu, \mix\vSigma )} + \log \frac{ \mathcal{IG}(\mix|a_0,a_0)}{ \mathcal{IG}(\mix|a,a)}},
\end{align*}
Our goal is to compute the gradient of this ELBO with respect to the expectation parameters. 

Since the expectation parameters $\vvarmean_\lat$ only depend on $\vmu$ and $\vSigma$, we can write the gradient with respect to them using the chain rule (similar to the finite mixture of Gaussians case):
\begin{align*}
\nabla_{m}  \elbofinal (\vvarpar )  &=  \nabla_{\mu}  \elbofinal (\vvarpar )  - 2 \nabla_{\Sigma}  \elbofinal (\vvarpar ) \vmu \\ 
\nabla_{M} \elbofinal (\vvarpar )&= \nabla_{\Sigma} \elbofinal (\vvarpar) 
\end{align*}
These gradients of the lower bound can be obtained as follows:
\begin{align*}
\nabla_{\mu}  \elbofinal (\vvarpar )  &= -\sum_{n=1}^{N}\nabla_{\mu}  \Unmyexpect{q(\lat,\mix)}\sqr{ f_n(\vlat)} -  \vmu \\ 
\nabla_{\Sigma} \elbofinal (\vvarpar ) &= -\sum_{n=1}^{N}\nabla_{\Sigma} \Unmyexpect{q(\lat, \mix)}\sqr{f_n(\vlat)} - \half \vI + \half \vSigma^{-1} 
\end{align*}
Plugging these in the natural-gradient updates \eqref{eq:simplengvi}, 
 we get the following updates (we have simplified these in the same way as explained in Appendix \ref{sec:fmGauss}; more details in \citet{khan18a}):
\begin{align*}
\vSigma^{-1} &\leftarrow  (1-\beta)  \vSigma^{-1} + 2 \beta \sum_{n=1}^{N} \nabla_{\Sigma} \Unmyexpect{q(\lat,\mix)}\sqr{f_n(\vlat)} + \beta  \vI \\
\vmu &\leftarrow \vmu -\beta  \sum_{n=1}^{N}\vSigma \left( \nabla_{\mu}  \Unmyexpect{q(\lat)}\sqr{f_n(\vlat)} + \vmu \right) 
\end{align*}
To compute the gradients, the reparametrization trick \cite{kingma2013auto} can be used. However, we can do better by using the extended Bonnet's and Price's theorems for Student's t-distribution \citep{wu-report}.
Assuming that $f_n(\vlat)$ satisfies the assumptions needed for these two theorems, we can use the following stochastic-gradient approximations for the graidents:
\begin{align*}
\nabla_{\mu} \Unmyexpect{q(\mix)\mathcal{N}(\lat|\mu,\mix\Sigma)} \sqr{ f(\vlat) } =& \Unmyexpect{ q(\lat)} \sqr{ \nabla_{\lat} f(\vlat)}  \quad \approx \nabla_z f(\vz),  \\
\nabla_{\Sigma} \Unmyexpect{q(\mix)\mathcal{N}(\lat|\mu,\mix\Sigma)} \sqr{ f(\vlat) } =& \half \Unmyexpect{q(\lat)}\sqr{ u(\vlat) \nabla_{\lat}^2 f(\vlat)  } \quad \approx \half u(\vlat) \nabla_z^2 f(\vlat), \\
=& \half \Unmyexpect{q(\mix,\lat)}\sqr{ \mix \nabla_{\lat}^2 f(\vlat)  } \quad \approx \half \mix \nabla_z^2 f(\vlat),
\end{align*} 
where  $\vlat \in \mathcal{R}^d$ is generated from $q(\vlat)$, $\mix$ is generated from $q(\mix)$ , and
\begin{align*}
   u(\vlat) := \frac{a + \half \left( \vlat-\vmu \right)^T \vSigma^{-1} \left( \vlat-\vmu \right) }{(a+d/2-1)} .
\end{align*}
Using these gradient, we get the following update: 
\begin{align*}
   \vSigma^{-1} &\leftarrow (1-\beta ) \vSigma^{-1} +  \beta \sqr{ u(\vlat) \nabla_z^2 f_n(\vlat) + \vI/N },  \\
   \vmu &\leftarrow \vmu -\beta \vSigma \sqr{ \nabla_z f_n(\vlat) + \vmu /N }. 
\end{align*}

Now we derive the NGD update for $a$.
Recall that the natural parameter is $a$ and the expectation parameter is $m_a=-1 - \log a + \psi(a)$.
The gradient of the lower bound can be expressed as
\begin{align*}
\nabla_{m_a} \elbofinal (\vvarpar )
&=  - \sum_{n=1}^{N} \nabla_{m_a} \Unmyexpect{q(\lat,\mix)}\sqr{ f_n(\vlat)  } + a_0 - a
\end{align*}
which gives us the following update:
\begin{align}
a \leftarrow (1-\beta) a + \beta \left( a_0 - \sum_{n=1}^{N} \nabla_{m_a} \Unmyexpect{q(\lat,\mix)}\sqr{  f_n(\vlat)  }   \right)
\end{align} 
While the gradient with respect to the expectation parameter does not admit a closed-form expression, we can compute the gradient using the re-parametrization trick.
According to \eqref{eq:ng_update_fim}, the gradient $\nabla_{m_a} \Unmyexpect{q(\lat,\mix)}\sqr{  f_n(\vlat)  }$ can be computed as
\begin{align*}
\nabla_{m_a} \Unmyexpect{q(\lat,\mix)}\sqr{  f_n(\vlat)  } &= \left( \nabla_{a} m_a \right)^{-1}  \nabla_{a} \Unmyexpect{q(\lat,\mix)}\sqr{  f_n(\vlat)  } = \left( \nabla_{a} \Unmyexpect{q(\mix)}\sqr{  \phi_\mix(\mix)  }  \right)^{-1}  \nabla_{a} \Unmyexpect{q(\lat,\mix)}\sqr{  f_n(\vlat)  } 
\end{align*}
Note that $\nabla_{a} \Unmyexpect{q(\mix)}\sqr{  \phi_\mix(\mix)  }=\nabla_a ( m_a )$ has a closed-form expression. However we have found that using stochastic approximation for both the numerator and denominator works better. \citet{salimans2013fixed} show that such approximation reduces the variance but introduce a bit bias.  
We use the reparameterization trick for both terms.
Since $q(\mix)$ is (implicitly) re-parameterizable \citep{salimans2013fixed,figurnov2018implicit} , the gradient can be computed as
\begin{align*}
\nabla_{a} \Unmyexpect{q(\mix)}\sqr{  \phi_\mix(\mix)} &= -\int \mathcal{IG}(\mix|a,a) \left( \nabla_{\mix} \sqr{\mix^{-1} +\log \mix} \right) \left( \nabla_{a} \mix   \right) d\mix \quad \approx (\mix^{-2} - \mix^{-1}) \nabla_a \mix
\end{align*} where $\mix$ is generated from $\mathcal{IG}(\mix|a,a)$.
Similarly,
\begin{align*}
 \nabla_{a} \Unmyexpect{q(\lat,\mix)}\sqr{  f_n(\vlat)  }   &= \int \int q(\mix) \left( \nabla_{\mix} q(\vlat|\mix) \right) \left( \nabla_{a} \mix  \right) f_n(\vlat) d\mix d\vlat \\
&= \int \int \mathcal{IG}(\mix|a,a) \left( \nabla_{\mix} \gauss(\vlat|\vmu,\mix\vSigma) \right) \left( \nabla_{a} \mix  \right) f_n(\vlat) d\mix d\vlat
\end{align*}
For stochastic approximation, we generate $\mix$ from $q(\mix)$ and 
let $\hat{\vSigma}=\mix\vSigma$.
The above expression can be approximated as below.
\begin{align*}
 \nabla_{a} \Unmyexpect{q(\lat,\mix)}\sqr{  f_n(\vlat)  } 
&\approx  \int  \left( \nabla_{\mix} \gauss(\vlat|\vmu,\mix\vSigma) \right) \left( \nabla_{a} \mix  \right) f_n(\vlat)  d\vlat \\
   & = \int  \mathrm{Tr} \left( \vSigma \nabla_{\hat{\Sigma}} \gauss(\vlat|\vmu,\hat{\vSigma}) \right) \left( \nabla_{a} \mix  \right) f_n(\vlat)  d\vlat \\
   & = \nabla_{a} \mix \mathrm{Tr} \left(  \vSigma \nabla_{\hat{\Sigma}} \Unmyexpect{\gauss(\lat|\mu,\hat{\Sigma})} \sqr{ f_n(\vlat)}  \right) \\
   & = \frac{\nabla_{a} \mix}{2} \mathrm{Tr} \left(  \vSigma \Unmyexpect{\gauss(\lat|\mu,\hat{\Sigma})} \sqr{ \nabla_{\lat}^2 f_n(\vlat)}  \right),
\end{align*} where we use the Price's theorem $\nabla_{\hat{\Sigma}} \Unmyexpect{\gauss(\lat|\mu,\hat{\Sigma})} \sqr{ f_n(\vlat)} =\half \Unmyexpect{\gauss(\lat|\mu,\hat{\Sigma})} \sqr{ \nabla_{\lat}^2 f_n(\vlat)} $.

We then use stochastic approximation to get the desired update.
\begin{align*}
 \nabla_{a} \Unmyexpect{q(\lat,\mix)}\sqr{  f_n(\vlat)  } 
   & \approx \frac{\nabla_{a} \mix}{2} \mathrm{Tr} \left(  \vSigma \nabla_{\lat}^2 f_n(\vlat)  \right)
\end{align*} where $\vlat$ is generated from $q(\vlat)$ and $\mix$ is generated from $q(\mix)$.

Anther example is the symmetric normal inverse-Gaussian distribution, which can be found at Appendix \ref{app:sym_nig}.

   \section{Multivariate Skew-Gaussian Distribution}
   \label{app:skew_gauss}
      We consider the following variational distribution .
\begin{align}
 q(\vlat,\mix) =  \gauss(\vlat|\vmu +  |\mix| \valpha, \vSigma) \gauss(\mix|0,1)  \label{skew_pdf}
\end{align}
The marginal distribution is known as the multivariate skew Gaussian distribution \citep{azzalini2005skew} as shown in the following lemma.

\begin{lemma}
The marginal distribution $q(\vlat)$ is $2 \Phi\left( \frac{ \left( \vlat - \vmu \right)^T \vSigma^{-1} \valpha }{ \sqrt{ 1+\valpha^T\vSigma^{-1}\valpha} } \right)\gauss( \vlat| \vmu, \vSigma+\valpha\valpha^T )$, where  $\Phi(\cdot)$ is the CDF of the standard univariate Gaussian distribution.
\end{lemma}

\begin{proof}
The marginal distribution $\vlat$ is 
 \begin{align*}
q(\vlat)
=&   2 \int_{0}^{+\infty}  \gauss(\mix|0,1)  \gauss(\vlat|\vmu+ \mix \valpha,\vSigma)  d\mix \\
=&   2 \int_{0}^{+\infty}  \gauss(\mix|0,1)  \gauss(\vlat- \mix \valpha|\vmu,\vSigma)  d\mix 
\end{align*} 
By grouping terms related to $\mix$ together and completing a Gaussian form for $\mix$, we obtain the following expression 
\begin{align*}
q(\vlat)
=& 2 \int_{0}^{+\infty} \gauss( \mix| \frac{ \left( \vlat - \vmu \right)^T \vSigma^{-1} \valpha }{1+\valpha^T\vSigma^{-1}\valpha} , \left( 1+\valpha^T \vSigma^{-1} \valpha \right)^{-1} )
  \gauss( \vlat| \vmu, \vSigma+\valpha\valpha^T ) d\mix \\
  =& 2 \Phi\left( \frac{ \left( \vlat - \vmu \right)^T \vSigma^{-1} \valpha }{ \sqrt{ 1+\valpha^T\vSigma^{-1}\valpha} } \right)\gauss( \vlat| \vmu, \vSigma+\valpha\valpha^T ) 
\end{align*} where we move from the first step to the second step using the fact that
\begin{align*}
 \int_{0}^{+\infty} \gauss( \mix|u,\sigma^2) d\mix &= \int_{ -u/\sigma }^{+\infty} \gauss( \mix'|0,1) d\mix' \\
 &= 1 - \Phi(-u/\sigma )\\
 &= \Phi(u/\sigma).
\end{align*}

\end{proof}

\begin{lemma}
 The joint distribution of skew-Gaussian  distribution given at Eq. \eqref{skew_pdf} can be written in a conditional EF form.
   \label{lemma:jointcef_skew}
\end{lemma}
\begin{proof}
   We first rewrite $q(\vlat|\mix)$ in a EF-form as follows:
\begin{align*}
q(\vlat|\mix) &= \gauss(\vlat|\vmu +  |\mix|\valpha, \vSigma) \\
&= (2\pi)^{-d/2} \exp(  \crl{ \mathrm{Tr} \left( -\half\vSigma^{-1} \vlat\vlat^T \right) + |\mix| \valpha^T \vSigma^{-1} \vlat + \vmu^T \vSigma^{-1} \vlat- \half(  (\vmu+|\mix|\valpha)^T \vSigma^{-1}(\vmu+|\mix|\valpha)  +   \log \mathrm{det}(\vSigma)  } )
\end{align*}
   The sufficient statistics $\vphi_\lat(\vlat,\mix)=\crl{  \vlat, |\mix|\vlat , \vlat\vlat^T} $ and the natural parameter  $\vvarpar_\lat=\crl{ \vSigma^{-1}\vmu, \vSigma^{-1}\valpha, -\half \vSigma^{-1} }$ can be read from the form. Since  $q(\mix)$ is a univariate Gaussian distribution with known parameters, which is a EF distribution. Therefore, the joint distribution $q(\vlat,\mix)$ is a conditional EF. 
 \end{proof}
 
Let $c=\sqrt{\frac{2}{\pi}}$.
Using the fact that $\Unmyexpect{q(\mix|0,1)}\sqr{ |\mix| } = c$, we can derive the 
 expectation parameter shown below:
\begin{align*}
   \vm &:= \Unmyexpect{\gauss(\mix|0,1)\gauss(\lat|\mu +|\mix|\alpha, \Sigma) } \sqr{  \vlat } = \vmu + c \valpha ,   \\
 \vm_{\alpha} &:= \Unmyexpect{\gauss(\mix|0,1) \gauss(\lat|\mu+|\mix|\alpha, \Sigma) } \sqr{ |\mix| \vlat } =  c \vmu +  \valpha, \\
   \vM &:= \Unmyexpect{\gauss(\mix|0,1) \gauss(\lat|\mu+|\mix|\alpha, \Sigma) } \sqr{ \vlat \vlat^T } =  \vmu \vmu^T + \valpha\valpha^T + c \left( \vmu\valpha^T + \valpha \vmu^T \right) +  \vSigma 
\end{align*}
The sufficient statistics, natural parameters, and expectation parameters are summarized below:
\begin{equation*}
\left[
\begin{array}{c}
\vlat\\
|\mix| \vlat\\
\vlat\vlat^T
\end{array}
\right]
\quad
\left[
\begin{array}{c}
\vSigma^{-1}\vmu\\
\vSigma^{-1}\valpha\\
-\half \vSigma^{-1}
\end{array}
\right]
\quad
\left[
\begin{array}{c}
\vmu + c \valpha \\
c \vmu + \valpha \\
\vmu \vmu^T + \valpha\valpha^T + c \left( \vmu\valpha^T + \valpha \vmu^T \right) +  \vSigma
\end{array}
\right]
\end{equation*}

The following lemma shows that the skew-Gaussian distribution is indeed a minimal conditional-EF.
\begin{lemma}
\label{lemma:skew_g_proof}
Multivariate skew Gaussians is a minimal conditional-EF.
\end{lemma}

\begin{proof}
Since $\vvarpar_\mix$ is known in this case, we only need to show
there exists an one-to-one mapping between the natural parameter  and the expectation parameter.
Just by observing the parameters given above, we can see that there exists an one-to-one mapping between the natural parameter and $\crl{\vmu,\valpha,\vSigma}$.
We can show that there also exists an one-to-one mapping between  $\crl{\vmu,\valpha,\vSigma}$ and the expectation parameter by noticing that 
\begin{align}
\vmu  & = \frac{\vm- c \vm_{\alpha}}{1-c^2} \label{eq:skew_maping_1}\\ 
\valpha  & = \frac{ \vm_\alpha - c \vm}{1-c^2} \label{eq:skew_maping_2}\\ 
\vSigma  & = \vM  - \frac{ \vm\vm^T +\vm_{\alpha}\vm_{\alpha}^T - c \left( \vm_{\alpha}\vm^T +\vm\vm_{\alpha}^T \right)  }{1-c^2} \label{eq:skew_maping_3}
\end{align}
Since one-to-one mapping is transitive, we know that mapping between natural and expectation parameters is one-to-one. Hence proved.
\end{proof}

\subsection{Derivation of the NGD Update}
\label{app:skew_ngd_deriv}

Let's consider the variational approximation using the skew-Gaussian distribution $q(\vlat|\vvarpar)$.
We consider the following model with a Gaussian-prior $\gauss(\vlat|\mathbf{0},\delta^{-1} \vI)$ on $\vlat$,
where the log-likelihood for the $n$'th data point is denoted by $p(\data_n|\vlat)$.
\begin{align*}
p(\mathcal{D},\vlat) = \prod_{n=1}^{N} p(\mathcal{D}_n|\vlat) \gauss(\vlat|\mathbf{0}, \delta^{-1} \vI)
\end{align*}

We use the lower bound defined in the following distribution $p(\mathcal{D}, \vlat)$:
\begin{align*}
\mathcal{L}(\vvarpar) = \Unmyexpect{q(\lat|\varpar)}\sqr{  \sum_{n=1}^{N} \underbrace{ \log p(\data_n|\vlat)}_{:= -f_n(\vlat) } + \log \gauss(\vlat|\mathbf{0},\delta^{-1} \vI)  -\log q(\vz|\vvarpar) },
\end{align*} where $q(\vlat) = 2 \Phi\left( \frac{ \left( \vlat - \vmu \right)^T \vSigma^{-1} \valpha }{ \sqrt{ 1+\valpha^T\vSigma^{-1}\valpha} } \right)\gauss( \vlat| \vmu, \vSigma+\valpha\valpha^T )$ and recall that $\Phi(\cdot)$ denotes the CDF of the standard univariate normal distribution.

Our goal is to compute the gradient of this ELBO with respect to the expectation parameters. 

\subsection{Natural Gradient for $q(\vlat|\mix)$}
We do not need to compute these gradients with respect to the expectation parameters directly. The gradients can be computed in terms of $\crl{\vmu,\valpha, \vSigma}$.

Using the mapping at \eqref{eq:skew_maping_1}-\eqref{eq:skew_maping_3} and the chain rule, we can express the following gradients with respect to the expectation parameters in terms of the gradients with respect to $\vmu$, $\valpha$, and $\vSigma$.
\begin{align*}
 \nabla_{m} \mathcal{L} &=  \frac{1}{1-c^2} \nabla_{\mu} \mathcal{L} - \frac{c}{1-c^2} \nabla_{\alpha} \mathcal{L} - 2 \left( \nabla_\Sigma \mathcal{L}\right) \vmu \\
 \nabla_{m_{\alpha}} \mathcal{L} &= \frac{1}{1-c^2} \nabla_{\alpha} \mathcal{L} - \frac{c}{1-c^2} \nabla_{\mu} \mathcal{L} - 2 \left( \nabla_\Sigma \mathcal{L} \right)\valpha \\
 \nabla_{M} \mathcal{L} &= \nabla_{\Sigma}\mathcal{L} 
\end{align*}

By plugging the gradients into the update in \eqref{eq:simplengvi} and then re-expressing the update in terms of $\vmu$, $\valpha$, $\vSigma$ (we have simplified these in the same way as explained in Appendix \ref{sec:fmGauss}),
we obtain the natural gradient update in terms of $\vmu$, $\valpha$, and $\vSigma$.
\begin{align}
 \vSigma^{-1} & \leftarrow \vSigma^{-1} - 2 \beta \nabla_{\Sigma}\mathcal{L}  \label{eq:ngd_skew_sigma} \\
 \vmu & \leftarrow \vmu + \beta \vSigma \left( \frac{1}{1-c^2} \nabla_{\mu} \mathcal{L} - \frac{c}{1-c^2} \nabla_{\alpha} \mathcal{L} \right)  \label{eq:ngd_skew_mu} \\
 \valpha &\leftarrow \valpha + \beta \vSigma \left( \frac{1}{1-c^2} \nabla_{\alpha} \mathcal{L} - \frac{c}{1-c^2} \nabla_{\mu} \mathcal{L} \right)  \label{eq:ngd_skew_alpha}
\end{align}

Recall that  the lower bound is
\begin{align*}
\mathcal{L}(\vvarpar) = \Unmyexpect{q(z|\varpar)}\sqr{  - \sum_{n=1}^{N} f_n(\vlat)  + \underbrace{ \log \gauss(\vz|\mathbf{0},\delta^{-1} \vI)}_{\textrm{ prior }} \underbrace{ -\log q(\vz|\vvarpar) }_{ \textrm{ entropy }  }  },
\end{align*} 

For the prior term, there is a closed-form expression for gradient computation.
\begin{align*}
\Unmyexpect{q(\lat)} \sqr{ \log \gauss(\vlat|\mathbf{0},\delta^{-1} \vI) }
&= -\frac{d}{2} \log (2\pi) + \frac{d \delta}{2} - \frac{\delta}{2} \left( \valpha^T\valpha + 2 c  \vmu^T \valpha + \mathrm{Tr}\left(\vSigma  \right) +  \vmu^T\vmu \right) 
\end{align*}
It is easy to show that the gradients about the prior term are 
\begin{align*}
 \vg_\mu^{\text{prior}}  &=  - \delta \left( \vmu+ c \valpha \right) \\
 \vg_\alpha^{\text{prior}}  &=  - \delta \left( \valpha+ c \vmu \right)\\
 \vg_\Sigma^{\text{prior}}  &=   - \frac{\delta}{2} \vI
\end{align*}

For the entropy term, by \citet{contreras2012kullback,arellano2013shannon}, it can be expressed as follows.
\begin{align}
\Unmyexpect{q(\lat)} \sqr{ -\log q(\vlat|\vvarpar) }
 =   \frac{d}{2} (\log(2\pi)+1) + \half \log |\vSigma+\valpha\valpha^T|  -  2 \Unmyexpect{\mathcal{N}(\lat_2|0, \valpha^T \vSigma^{-1} \valpha)} \Big[  \Phi\left(  \lat_2\right)  \log  \left( \Phi\left( \lat_2 \right) \right) \Big] - \log(2) \label{eq:skew_ent_exact}
\end{align}

We can use the re-parametrization trick to approximate the gradients about the entropy term. However, the exact gradients usually works better.
Using the expression in \eqref{eq:skew_ent_exact}, the gradients of the entropy term are given as follows: 
\begin{align}
\vg_\mu^{\text{entropy}}  &= 0 \\
\vg_\alpha^{\text{entropy}} 
 &= \left( \vSigma + \valpha \valpha^T \right)^{-1} \valpha -  \Unmyexpect{\mathcal{N}\left(\lat_3|0, \alpha^T \Sigma^{-1} \alpha  / \left( 1+ \alpha^T \Sigma^{-1} \alpha \right) \right)} \sqr{ \frac{ \log  \left( \Phi\left(\lat_3 \right) \right)} {\sqrt{ 2 \pi \left( 1+\valpha^T \vSigma^{-1} \valpha \right)  } }    \frac{ 2 \lat_3  \vSigma^{-1} \valpha } { \valpha^T \vSigma^{-1} \valpha  } }  \label{eq:skew_ent_grad_alpha} \\
\vg_\Sigma^{\text{entropy}}
  &= \half \left( \vSigma + \valpha \valpha^T \right)^{-1}   +  \Unmyexpect{\mathcal{N}\left(\lat_3|0, \alpha^T \Sigma^{-1} \alpha / \left( 1+ \alpha^T \Sigma^{-1} \alpha \right) \right)} \sqr{ \frac{ \log  \left( \Phi\left( \lat_3 \right) \right)} {\sqrt{ 2 \pi \left( 1+\valpha^T \vSigma^{-1} \valpha \right)  } }    \frac{ \lat_3  \vSigma^{-1} \valpha \valpha^T \vSigma^{-1} } { \valpha^T \vSigma^{-1} \valpha  } }   \label{eq:skew_ent_grad_Sigma}
\end{align}
where the expectations involve 1d integrations, which can be computed by Gauss-Hermite quadrature.

The remaining thing is to compute the gradients about $\Unmyexpect{q(\lat)}\sqr{ f_n(\vlat) }$.
To compute the gradients, the reparametrization trick can be used. However, we can do better by using 
the extended Bonnet's and Price's theorems for skew-Gaussian distribution \citep{wu-report}.
Assuming that $f_n(\vlat)$ satisfies the assumptions needed for these two theorems, we obtain the following gradient expression:
\begin{align*}
\vg_{1}^n &:=  \nabla_{\mu} \Unmyexpect{q(\lat)} \sqr{ f_n(\vlat) } = \Unmyexpect{ q(\lat)} \sqr{ \nabla_{\lat} f_n(\vlat)} \approx   \nabla_{\lat} f_n(\vlat) \\
\vg_{2}^n &:= \nabla_{\alpha} \Unmyexpect{q(\lat)} \sqr{ f_n(\vlat) } \nonumber \\
&= \Unmyexpect{ q(\lat)} \sqr{ u(\vlat) \nabla_{\lat} f_n(\vlat)} +  v \Unmyexpect{ \gauss(\hat{\lat}|\mu, \Sigma)} \sqr{ \nabla_{\hat{\lat}} f_n(\hat{\vlat})}  \approx u(\vlat) \nabla_{\lat} f_n(\vlat) +  v \nabla_{\hat{\lat}} f_n(\hat{\vlat})\\
&= \Unmyexpect{ q(\mix,\lat)} \sqr{ |\mix| \nabla_{\lat} f_n(\vlat)}  \approx |\mix| \nabla_{\lat} f_n(\vlat) \\
\vg_{3}^n &:= 2 \nabla_{\Sigma} \Unmyexpect{q(\lat)} \sqr{ f_n(\vlat) } =  \Unmyexpect{q(\lat)}\sqr{ \nabla_{\lat}^2 f_n(\vlat)  } \approx  \nabla_{\lat}^2 f_n(\vlat)
\end{align*} 
where 
$v = \frac{c}{( 1+\valpha^T \vSigma^{-1} \valpha )}$, 
$u(\vlat)  := \frac{\left( \vlat-\vmu \right)^T \vSigma^{-1} \valpha }{1 + \valpha^T \vSigma^{-1} \valpha }$, and
$\mix  \sim \gauss(\mix|0,1), \,
  \hat{\vlat}  \sim \gauss(\hat{\vlat}|\vmu,\vSigma), \,
  \vlat = \hat{\vlat} + |\mix|\valpha$.

Putting together, 
we can express the gradients in the following form, which will be used for deriving the extended variational Adam update.
\begin{align}
 \nabla_{\mu} \mathcal{L}(\vvarpar) 
 &= - \sum_{n=1}^{N}  \left( \vg_{1}^n \underbrace{- \vg_\mu^{\text{entropy}}/N}_{0} \right) \underbrace{- \delta ( \vmu+c\valpha)}_{ \vg_{\mu}^{\textrm{prior}} }   \nonumber \\
 &= - \sum_{n=1}^{N} \underbrace{ \vg_{1}^n }_{ := \vg_{\mu}^n } - \delta  ( \vmu+c\valpha)   \label{eq:ske_g_update_mu}
 \\
 \nabla_{\alpha} \mathcal{L}(\vvarpar)&= - \sum_{n=1}^{N} \left( \underbrace{  \vg_{2}^n  -  \vg_\alpha^{\text{entropy}}/N }_{ := \vg_{\alpha}^n } \right) \underbrace{- \delta  ( \valpha+c\vmu)}_{ \vg_{\alpha}^{\textrm{prior}} } \label{eq:ske_g_update_alpha} \\
 \nabla_{\Sigma} \mathcal{L}(\vvarpar) &= - \half \sum_{n=1}^{N} \left( \vg_{3}^n - 2 \vg_\Sigma^{\text{entropy}}/N \right) \underbrace{-\frac{\delta}{2} \vI}_{  \vg_{\Sigma}^{\textrm{prior}} }  +  \half (\vSigma^{-1} -\vSigma^{-1}) \nonumber \\
&= - \half \sum_{n=1}^{N} \left( \underbrace{ \vg_{3}^n - 2 \vg_\Sigma^{\text{entropy}}/N + \vSigma^{-1}/N }_{:= \vg_{S}^n}\right) -\frac{\delta}{2} \vI  +  \half \vSigma^{-1} \label{eq:ske_g_update_sigma}
\end{align}

For stochastic approximation, we can sub-sampling a data point $n$ and use MC samples to approximate $\vg_{1}^n$, $\vg_{2}^n$, and $\vg_{3}^n$.
Plugging these stochastic gradients into 
\eqref{eq:ngd_skew_sigma}-\eqref{eq:ngd_skew_alpha}, we obtain the NGD update:
\begin{align*}
 \vSigma^{-1} & \leftarrow (1-\beta) \vSigma^{-1} + \beta (\delta \vI + N \vg_{S}^n) \\
 \vmu & \leftarrow  \vmu - \beta \vSigma (\frac{N}{1-c^2} ( \vg_{\mu}^n - c \vg_{\alpha}^n) + \delta \vmu) \\
 \valpha & \leftarrow  \valpha - \beta \vSigma (\frac{N}{1-c^2} ( \vg_{\alpha}^n - c \vg_{\mu}^n) + \delta \valpha)
\end{align*}

   \section{Multivariate Exponentially Modified Gaussian Distribution}
   \label{app:exp_gauss}
      We consider the following mixture distribution.
\begin{align*}
 q(\vlat,\mix) =  \gauss(\vlat|\vmu +  \mix \valpha, \vSigma) \expdist(\mix|1) 
\end{align*}
The marginal distribution is a multivariate extension of the exponentially modified Gaussian distribution \citep{grushka1972characterization} and the Gaussian minus exponential distribution \citep{carr2009saddlepoint} due to the following lemma.

\begin{lemma}
The marginal distribution $q(\vlat)$ ($\valpha \neq \mathbf{0}$) is
\begin{align*}
q(\vlat)
=  \frac{ \sqrt{2\pi}\mathrm{det}\left(2\pi\vSigma \right)^{-\half} }{  \sqrt{\valpha^T \vSigma^{-1}\valpha} } \Phi\left( \frac{ \left( \vlat-\vmu \right)^T \vSigma^{-1} \valpha -1 }{ \sqrt{ \valpha^T \vSigma^{-1} \valpha }  } \right)\exp\crl{  \half\sqr{ \frac{ \left( \left( \vlat-\vmu \right)^T \vSigma^{-1} \valpha -1 \right)^2  }{ \valpha^T \vSigma^{-1} \valpha } - \left(\vlat-\vmu\right)^T \vSigma^{-1} \left(\vlat-\vmu \right)  }  },
\end{align*} where  $\Phi(\cdot)$ is the CDF of the standard univariate Gaussian distribution.

In the univariate case, the marginal distribution becomes the exponentially modified Gaussian distribution when $\alpha>0$ and the Gaussian minus exponential distribution when $\alpha<0$.
\end{lemma}

\begin{proof}
The marginal distribution $\vlat$ is 
\begin{align}
&  q(\vlat|\vmu,\valpha,\vSigma) =    \int_{0}^{+\infty}  \expdist(\mix|0,1)  \gauss(\vlat|\vmu+\mix\valpha,\vSigma)  d\mix \nonumber \\
&\text{By grouping terms related to $\mix$ together and completing a Gaussian form for $\mix$, we obtain the following expression.} \nonumber \\
=& \frac{ \mathrm{det}\left(2\pi\vSigma \right)^{-\half} }{  \sqrt{\frac{\valpha^T \vSigma^{-1}\valpha}{2\pi}} } \int_{0}^{+\infty} \gauss( \mix| \frac{ \left( \vlat - \vmu \right)^T \vSigma^{-1} \valpha-1 }{\valpha^T\vSigma^{-1}\valpha} , \frac{1}{ \valpha^T \vSigma^{-1} \valpha } ) \exp\crl{  \half\sqr{ \frac{ \left( \left( \vlat-\vmu \right)^T \vSigma^{-1} \valpha -1 \right)^2  }{ \valpha^T \vSigma^{-1} \valpha } - \left(\vlat-\vmu\right)^T \vSigma^{-1} \left(\vlat-\vmu \right)  }  }
 d\mix \label{eq:exp_gauss_eq1}\\
=&  \frac{ \sqrt{2\pi}\mathrm{det}\left(2\pi\vSigma \right)^{-\half} }{  \sqrt{\valpha^T \vSigma^{-1}\valpha} } \Phi\left( \frac{ \left( \vlat-\vmu \right)^T \vSigma^{-1} \valpha -1 }{ \sqrt{ \valpha^T \vSigma^{-1} \valpha }  } \right)\exp\crl{  \half\sqr{ \frac{ \left( \left( \vlat-\vmu \right)^T \vSigma^{-1} \valpha -1 \right)^2  }{ \valpha^T \vSigma^{-1} \valpha } - \left(\vlat-\vmu\right)^T \vSigma^{-1} \left(\vlat-\vmu \right)  }  }
 \label{eq:exp_gauss_eq2}
\end{align}
where we move from Eq. \eqref{eq:exp_gauss_eq1}  to Eq. \eqref{eq:exp_gauss_eq2} using the fact that
$ \int_{0}^{+\infty} \gauss( \mix|u,\sigma^2) d\mix
= \Phi(u/\sigma)
$.
\end{proof}

Similar to the skew-Gaussian case, in this example, the sufficient statistics $\vphi_\lat(\vlat,\mix)=\crl{  \vlat, \mix\vlat , \vlat\vlat^T} $ and the natural parameter  $\vvarpar_\lat=\crl{ \vSigma^{-1}\vmu, \vSigma^{-1}\valpha, -\half \vSigma^{-1} }$ can be read from $q(\vlat|\mix)=\gauss(\vlat|\vmu+\mix\valpha,\vSigma)$.  The joint distribution $q(\vlat,\mix)$ is a conditional EF since  $q(\mix)$ is an exponential distribution with known parameters, which is an EF distribution. Likewise, we can show that the joint distribution is also a minimal conditional EF.
We can derive the expectation parameter shown below:
\begin{align*}
   \vm &:= \Unmyexpect{\expdist(\mix|1)\gauss(\lat|\mu +\mix\alpha, \Sigma) } \sqr{  \vlat } = \vmu +  \valpha ,   \\
 \vm_{\alpha} &:= \Unmyexpect{\expdist(\mix|1) \gauss(\lat|\mu+\mix\alpha, \Sigma) } \sqr{ \mix \vlat } =  \vmu + 2 \valpha, \\
   \vM &:= \Unmyexpect{\expdist(\mix|1) \gauss(\lat|\mu+\mix\alpha, \Sigma) } \sqr{ \vlat \vlat^T } =  \vmu \vmu^T +  2\valpha\valpha^T +  \left( \vmu\valpha^T + \valpha \vmu^T \right) +  \vSigma 
\end{align*}
The sufficient statistics, natural parameters, and expectation parameters are summarized below:
\begin{equation*}
\left[
\begin{array}{c}
\vlat\\
\mix \vlat\\
\vlat\vlat^T
\end{array}
\right]
\quad
\left[
\begin{array}{c}
\vSigma^{-1}\vmu\\
\vSigma^{-1}\valpha\\
-\half \vSigma^{-1}
\end{array}
\right]
\quad
\left[
\begin{array}{c}
\vmu +  \valpha \\
 \vmu +  2 \valpha \\
\vmu \vmu^T + 2 \valpha\valpha^T +  \left( \vmu\valpha^T + \valpha \vmu^T \right) +  \vSigma
\end{array}
\right]
\end{equation*}

\subsection{Derivation of the NGD Update}
As shown in Appendix \ref{app:skew_ngd_deriv}, we consider the variational approximation using the exponentially modified Gaussian distribution $q(\vlat|\vvarpar)$.
We consider the same model as Appendix \ref{app:skew_ngd_deriv} with a Gaussian-prior $\gauss(\vlat|\mathbf{0},\delta^{-1} \vI)$ on $\vlat$.
The lower bound is defined as below:
\begin{align*}
\mathcal{L}(\vvarpar) = \Unmyexpect{q(z|\varpar)}\sqr{  \sum_{n=1}^{N} \underbrace{ \log p(\data_n|\vlat)}_{:= -f_n(\vlat) } + \log \gauss(\vlat|\mathbf{0},\delta^{-1} \vI)  -\log q(\vlat|\vvarpar) },
\end{align*} where
$ q(\vlat) =\frac{ \sqrt{2\pi}\mathrm{det}\left(2\pi\vSigma \right)^{-\half} }{  \sqrt{\valpha^T \vSigma^{-1}\valpha} } \Phi\left( \frac{ \left( \vlat-\vmu \right)^T \vSigma^{-1} \valpha -1 }{ \sqrt{ \valpha^T \vSigma^{-1} \valpha }  } \right)\exp\crl{  \half\sqr{ \frac{ \left( \left( \vlat-\vmu \right)^T \vSigma^{-1} \valpha -1 \right)^2  }{ \valpha^T \vSigma^{-1} \valpha } - \left(\vlat-\vmu\right)^T \vSigma^{-1} \left(\vlat-\vmu \right)  }  }$
and recall that $\Phi(\cdot)$ denotes the CDF of the standard univariate normal distribution.

Our goal is to compute the gradient of this ELBO with respect to the expectation parameters. 

\subsection{Natural Gradient for $q(\vlat|\mix)$}
We do not need to compute these gradients with respect to the expectation parameters directly. The gradients can be computed in terms of $\crl{\vmu,\valpha, \vSigma}$.

Similarly, by the chain rule, we can express the following gradients with respect to the expectation parameters in terms of the gradients with respect to $\vmu$, $\valpha$, and $\vSigma$.
\begin{align*}
 \nabla_{m} \mathcal{L} &=  2 \nabla_{\mu} \mathcal{L} -  \nabla_{\alpha} \mathcal{L} - 2 \left( \nabla_\Sigma \mathcal{L}\right) \vmu \\
 \nabla_{m_{\alpha}} \mathcal{L} &=  \nabla_{\alpha} \mathcal{L} -  \nabla_{\mu} \mathcal{L} - 2 \left( \nabla_\Sigma \mathcal{L} \right)\valpha \\
 \nabla_{M} \mathcal{L} &= \nabla_{\Sigma}\mathcal{L} 
\end{align*}

By plugging the gradients into the update in \eqref{eq:simplengvi} and then re-expressing the update in terms of $\vmu$, $\valpha$, $\vSigma$,
we obtain the natural gradient update in terms of $\vmu$, $\valpha$, and $\vSigma$.
\begin{align}
 \vSigma^{-1} & \leftarrow \vSigma^{-1} - 2 \beta \nabla_{\Sigma}\mathcal{L} \label{eq:ngd_exp_sigma} \\
 \vmu & \leftarrow \vmu + \beta \vSigma \left( 2 \nabla_{\mu} \mathcal{L} -  \nabla_{\alpha} \mathcal{L} \right) \label{eq:ngd_exp_mu}\\
 \valpha &\leftarrow \valpha + \beta \vSigma \left(  \nabla_{\alpha} \mathcal{L} - \nabla_{\mu} \mathcal{L} \right)  \label{eq:ngd_exp_alpha}
\end{align}

Recall that  the lower bound is
\begin{align*}
\mathcal{L}(\vvarpar) = \Unmyexpect{q(z|\varpar)}\sqr{  - \sum_{n=1}^{N} f_n(\vlat)  + \underbrace{ \log \gauss(\vlat|\mathbf{0},\delta^{-1} \vI)}_{\textrm{ prior }} \underbrace{ -\log q(\vlat|\vvarpar) }_{ \textrm{ entropy }  }  },
\end{align*} 

For the prior term, there is a closed-form expression for gradient computation.
\begin{align*}
\Unmyexpect{q(\lat)} \sqr{ \log \gauss(\vlat|\mathbf{0},\delta^{-1} \vI) }
&= -\frac{d}{2} \log (2\pi) + \frac{d \delta}{2} - \frac{\delta}{2} \left( 2\valpha^T\valpha + 2  \vmu^T \valpha + \mathrm{Tr}\left(\vSigma  \right) +  \vmu^T\vmu \right) 
\end{align*}
It is easy to show that the gradients about the prior term are 
\begin{align*}
 \vg_\mu^{\textrm{prior}}  &=  - \delta \left( \vmu+  \valpha \right) \\
 \vg_\alpha^{\textrm{prior}}  &=  - \delta \left( 2 \valpha+  \vmu \right)\\
 \vg_\Sigma^{\textrm{prior}}  &=   - \frac{\delta}{2} \vI
\end{align*}

For the entropy term,  it can be expressed as follows.
\begin{align}
\Unmyexpect{q(\lat)}\sqr{ -\log q(\vlat) } 
= \half \crl{   \log \mathrm{det}\left(2\pi\vSigma\right)  + \log\left( \frac{ \valpha^T \vSigma^{-1} \valpha}{2\pi} \right) - \frac{1}{\valpha^T \vSigma^{-1} \valpha  } + \left( d+1 \right)   } - \Unmyexpect{\expdist(\mix|1)\gauss(\lat_2|\mix \sqrt{\valpha^T\vSigma^{1}\valpha},\, 1)}\sqr{ \log \phi\left( \lat_2 - \frac{ 1  }{\sqrt{\valpha^T\vSigma^{1}\valpha } } \right)   } \label{eq:exp_ent_exact}
\end{align}

Using the expression in \eqref{eq:exp_ent_exact}, the gradients of the entropy term are given as follows: 
\begin{align}
\vg_\mu^{\text{entropy}}  &= 0 \\
\vg_\alpha^{\text{entropy}} 
&= \frac{\vSigma^{-1}\valpha }{  \valpha^T \vSigma^{-1} \valpha }
+\frac{\vSigma^{-1}\valpha }{\left( \valpha^T \vSigma^{-1} \valpha \right)^2 }
- \Unmyexpect{\expdist(\mix|1)\gauss(\lat_2|0,1)}\sqr{ \frac{ \exp\left( - \frac{t^2}{2} - \log\phi(t) \right) }{\sqrt{2\pi} }   \left( \mix + \frac{1}{\valpha^T\vSigma^{-1}\valpha} \right) \left( \frac{ \vSigma^{-1}\valpha }{ \sqrt{\valpha^T\vSigma^{-1}\valpha} } \right) } \label{eq:exp_ent_grad_alpha} \\
\vg_\Sigma^{\text{entropy}} 
&=  \frac{\vSigma^{-1}}{2}
- \frac{\vSigma^{-1} \valpha \valpha^T \vSigma^{-1} }{2\valpha^T\vSigma^{-1}\valpha}
- \frac{\vSigma^{-1} \valpha \valpha^T \vSigma^{-1} }{2\left(\valpha^T\vSigma^{-1}\valpha\right)^2}
+ \Unmyexpect{\expdist(\mix|1)\gauss(\lat_2|0,1)}\sqr{ \frac{ \exp\left( - \frac{t^2}{2} - \log\phi(t) \right) }{\sqrt{2\pi} }   \left( \mix + \frac{1}{\valpha^T\vSigma^{-1}\valpha} \right) \left( \frac{ \vSigma^{-1} \valpha \valpha^T \vSigma^{-1} }{2 \sqrt{\valpha^T\vSigma^{-1}\valpha} } \right) } \label{eq:exp_ent_grad_Sigma}
\end{align}
where $t=\lat_2+ \frac{\mix \valpha^T\vSigma^{-1}\valpha- 1}{\sqrt{\valpha^T\vSigma^{-1}\valpha} }$
and the expectations involve 2d integrations, which can be computed by Gauss-Hermite quadrature and  Gauss-Laguerre quadrature.

The remaining thing is to compute the gradients about $\Unmyexpect{q(\lat)}\sqr{ f_n(\vlat) }$.
To compute the gradients, the reparametrization trick can be used. However, we can use
the extended Bonnet's and Price's theorems \citep{wu-report}.
Assuming that $f_n(\vlat)$ satisfies the assumptions needed for these two theorems, we obtain the following gradient expression:
\begin{align*}
\vg_{1}^n &:=  \nabla_{\mu} \Unmyexpect{q(\lat)} \sqr{ f_n(\vlat) } = \Unmyexpect{ q(\lat)} \sqr{ \nabla_{\lat} f_n(\vlat)} \approx   \nabla_{\lat} f_n(\vlat) \\
\vg_{2}^n &:= \nabla_{\alpha} \Unmyexpect{q(\lat)} \sqr{ f_n(\vlat) } \\
&= \Unmyexpect{ q(\lat)} \sqr{ u(\vlat) \nabla_{\lat} f_n(\vlat)} +  v \Unmyexpect{ \gauss(\hat{\lat}|\mu, \Sigma)} \sqr{ \nabla_{\hat{\lat}} f_n(\hat{\vlat})}  \approx u(\vlat) \nabla_{\lat} f_n(\vlat) +  v \nabla_{\hat{\lat}} f_n(\hat{\vlat})\\
&= \Unmyexpect{ q(\mix,\lat)} \sqr{ \mix \nabla_{\lat} f_n(\vlat)}  \approx \mix \nabla_{\lat} f_n(\vlat) \\
\vg_{3}^n &:= 2 \nabla_{\Sigma} \Unmyexpect{q(\lat)} \sqr{ f_n(\vlat) } =  \Unmyexpect{q(\lat)}\sqr{ \nabla_{\lat}^2 f_n(\vlat)  } \approx  \nabla_{\lat}^2 f_n(\vlat)
\end{align*} 
where 
$v = \frac{1}{(\valpha^T \vSigma^{-1} \valpha )}$, 
$u(\vlat)  := \frac{\left( \vlat-\vmu \right)^T \vSigma^{-1} \valpha - 1 }{ \valpha^T \vSigma^{-1} \valpha }$, and
$\mix  \sim \expdist(\mix|1), \,
  \hat{\vlat}  \sim \gauss(\hat{\vlat}|\vmu,\vSigma), \,
  \vlat = \hat{\vlat} + \mix\valpha$.

Putting together, 
we can express the gradients in the following form, which will be used for deriving the extended variational Adam update.
\begin{align}
 \nabla_{\mu} \mathcal{L}(\vvarpar) 
 &= - \sum_{n=1}^{N}  \left( \vg_{1}^n \underbrace{- \vg_\mu^{\text{entropy}}/N}_{0} \right) \underbrace{- \delta ( \vmu+\valpha)}_{ \vg_{\mu}^{\textrm{prior}} }   \nonumber \\
 &= - \sum_{n=1}^{N} \underbrace{ \vg_{1}^n }_{ := \vg_{\mu}^n } - \delta ( \vmu+\valpha)   \label{eq:exp_g_update_mu}
 \\
 \nabla_{\alpha} \mathcal{L}(\vvarpar)&= - \sum_{n=1}^{N} \left( \underbrace{ \vg_{2}^n  -  \vg_\alpha^{\text{entropy}}/N }_{ := \vg_{\alpha}^n }\right) \underbrace{- \delta (2 \valpha+\vmu)}_{ \vg_{\alpha}^{\textrm{prior}} } \label{eq:exp_g_update_alpha} \\
 \nabla_{\Sigma} \mathcal{L}(\vvarpar) &= - \half \sum_{n=1}^{N} \left( \vg_{3}^n - 2 \vg_\Sigma^{\text{entropy}}/N \right) \underbrace{-\frac{\delta}{2} \vI}_{  \vg_{\Sigma}^{\textrm{prior}} }  +  \half (\vSigma^{-1} -\vSigma^{-1}) \nonumber \\
&= - \half \sum_{n=1}^{N} \left( \underbrace{ \vg_{3}^n - 2 \vg_\Sigma^{\text{entropy}}/N + \vSigma^{-1}/N }_{:= \vg_{S}^n}\right) -\frac{\delta}{2} \vI  +  \half \vSigma^{-1} \label{eq:exp_g_update_sigma}
\end{align}

Similarly, for stochastic approximation, we can sub-sampling a data point $n$ and use MC samples to approximate $\vg_{1}^n$, $\vg_{2}^n$, and $\vg_{3}^n$.
Plugging these stochastic gradients into 
\eqref{eq:ngd_exp_sigma}-\eqref{eq:ngd_exp_alpha}, we obtain the NGD update:
\begin{align*}
 \vSigma^{-1} & \leftarrow (1-\beta) \vSigma^{-1} + \beta (\delta \vI + N \vg_{S}^n) \\
 \vmu & \leftarrow  \vmu - \beta \vSigma (N ( 2\vg_{\mu}^n -  \vg_{\alpha}^n) + \delta \vmu) \\
 \valpha & \leftarrow  \valpha - \beta \vSigma (N ( \vg_{\alpha}^n -  \vg_{\mu}^n) + \delta \valpha)
\end{align*}

  \section{Multivariate Normal Inverse-Gaussian Distribution}
   \label{app:nig}
      We consider the following mixture distribution \cite{barndorff1997normal}, which is a Gaussian variance-mean mixture distribution. For simplicity,  we assume $\lambda$ is known. 
\begin{align*}
 q(\mix,\vlat) = \gauss(\vlat|\vmu+\mix\valpha,\mix\vSigma) \IGauss(\mix|1,\lambda)
\end{align*}
where $\IGauss(\mix|1,\lambda)=\left(\frac{\lambda}{2\pi \mix^3} \right)^{\half}\exp\crl{ -\frac{\lambda}{2}\left(\mix+\mix^{-1}\right) + \lambda }$ denotes the inverse Gaussian distribution.

\begin{lemma}
The marginal distribution is
\begin{align*}
  q(\vlat) 
=\frac{\lambda^{\half}}{(2\pi)^{\frac{d+1}{2}}} \mathrm{det}\left(  \vSigma \right)^{-1/2} \exp\sqr{ \left(\vlat-\vmu \right)^T \vSigma^{-1} \valpha +\lambda } \frac{ 2 \mathcal{K}_{\frac{d+1}{2}} \left( \sqrt{ \left( \valpha^T \vSigma^{-1} \valpha  + \lambda \right) \left( \left(\vlat-\vmu \right)^T \vSigma^{-1} \left(\vlat-\vmu \right) + \lambda \right) } \right) }{ \left(\sqrt{  \frac{ \valpha^T \vSigma^{-1} \valpha + \lambda }{ \left( \vlat-\vmu \right)^T \vSigma^{-1} \left( \vlat-\vmu \right) + \lambda } } \right)^{\frac{-d-1}{2}} },
\end{align*}
where $\mathcal{K}_v(x)$ denotes the modified Bessel function of the second kind.
\end{lemma}

\begin{proof}
By definition, we can compute the marginal distribution as follows.
\begin{align*}
&  q(\vlat) 
  = \int_{0}^{+\infty} q(\vlat|\mix) q(\mix) d\mix \\
  =& \int_{0}^{+\infty} \mathrm{det}\left( 2\pi \mix \vSigma \right)^{-1/2} \exp\crl{  -\half \sqr{  \left( \vlat-\vmu-\mix\valpha \right)^T \left(\mix\vSigma \right)^{-1}  \left( \vlat-\vmu-\mix\valpha\right) } }  
  \left( \frac{\lambda}{2\pi \mix^3} \right)^{1/2}\exp\sqr{ -\frac{\lambda}{2} \left( \mix + \frac{1}{\mix} \right) + \lambda } d\mix \\
 & \text{By grouping all terms related to $\mix$ together, we have} \\
  =&\frac{\lambda^{\half}}{(2\pi)^{\frac{d+1}{2}}} \mathrm{det}\left(  \vSigma \right)^{-1/2} \exp\sqr{ \left(\vlat-\vmu \right)^T \vSigma^{-1} \valpha + \lambda } \int_{0}^{+\infty} 
  \mix^{-\frac{d+3}{2}}
  \exp\crl{  -\half \sqr{ \mix \left(  \valpha^T \vSigma^{-1} \valpha + \lambda \right)  + \frac{  \left( \vlat-\vmu \right)^T \vSigma^{-1} \left( \vlat-\vmu \right) + \lambda}{\mix} }    }
  d\mix\\
 & \text{By completing a generalized inverse Gaussian form, we have} \\
 =&\frac{\lambda^{\half}}{(2\pi)^{\frac{d+1}{2}}} \mathrm{det}\left(  \vSigma \right)^{-1/2} \exp\sqr{ \left(\vlat-\vmu \right)^T \vSigma^{-1} \valpha +\lambda } \frac{ 2 \mathcal{K}_{\frac{-d-1}{2}} \left( \sqrt{ \left( \valpha^T \vSigma^{-1} \valpha  + \lambda \right) \left( \left(\vlat-\vmu \right)^T \vSigma^{-1} \left(\vlat-\vmu \right) + \lambda \right) } \right) }{ \left(\sqrt{  \frac{ \valpha^T \vSigma^{-1} \valpha + \lambda }{ \left( \vlat-\vmu \right)^T \vSigma^{-1} \left( \vlat-\vmu \right) + \lambda } } \right)^{\frac{-d-1}{2}} } \\
 & \text{We obtain the last step by the fact that $\mathcal{K}_v(x)=\mathcal{K}_{-v}(x)$} \\
 =&\frac{\lambda^{\half}}{(2\pi)^{\frac{d+1}{2}}} \mathrm{det}\left(  \vSigma \right)^{-1/2} \exp\sqr{ \left(\vlat-\vmu \right)^T \vSigma^{-1} \valpha +\lambda } \frac{ 2 \mathcal{K}_{\frac{d+1}{2}} \left( \sqrt{ \left( \valpha^T \vSigma^{-1} \valpha  + \lambda \right) \left( \left(\vlat-\vmu \right)^T \vSigma^{-1} \left(\vlat-\vmu \right) + \lambda \right) } \right) }{ \left(\sqrt{  \frac{ \valpha^T \vSigma^{-1} \valpha + \lambda }{ \left( \vlat-\vmu \right)^T \vSigma^{-1} \left( \vlat-\vmu \right) + \lambda } } \right)^{\frac{-d-1}{2}} } 
\end{align*}
\end{proof}

Similarly,  the sufficient statistics $\vphi_\lat(\vlat,\mix)=\crl{  \vlat/\mix, \vlat , \vlat\vlat^T/\mix} $ and the natural parameter  $\vvarpar_\lat=\crl{ \vSigma^{-1}\vmu, \vSigma^{-1}\valpha, -\half \vSigma^{-1} }$ can be read from $q(\vlat|\mix)=\gauss(\vlat|\vmu+\mix\valpha,\mix\vSigma)$. The joint distribution $q(\vlat,\mix)$ is a conditional EF because $q(\mix)$ is an inverse Gaussian distribution with known parameters, which is a EF distribution.  Likewise, we can show that the joint distribution is also a minimal conditional EF.

We can derive the expectation parameter shown below:
\begin{align*}
   \vm &:= \Unmyexpect{\IGauss(\mix|1,\lambda)\gauss(\lat|\mu +\mix\alpha, \mix\Sigma) } \sqr{  \mix^{-1} \vlat } = \left( 1+ \lambda^{-1} \right) \vmu +  \valpha ,   \\
 \vm_{\alpha} &:= \Unmyexpect{\IGauss(\mix|1,\lambda) \gauss(\lat|\mu+\mix\alpha, \mix\Sigma) } \sqr{ \vlat } =  \vmu +  \valpha, \\
   \vM &:= \Unmyexpect{\IGauss(\mix|1,\lambda) \gauss(\lat|\mu+\mix\alpha,\mix \Sigma) } \sqr{ \mix^{-1} \vlat \vlat^T } = 
\left( 1+\lambda^{-1} \right) \vmu \vmu^T + \valpha\valpha^T +  \vmu\valpha^T + \valpha \vmu^T  +  \vSigma
\end{align*}

The sufficient statistics, natural parameters, and expectation parameters are summarized below:
\begin{equation*}
\left[
\begin{array}{c}
 \vlat/\mix\\
 \vlat\\
\vlat\vlat^T/\mix
\end{array}
\right]
\quad
\left[
\begin{array}{c}
\vSigma^{-1}\vmu\\
\vSigma^{-1}\valpha\\
-\half \vSigma^{-1}
\end{array}
\right]
\quad
\left[
\begin{array}{c}
\left( 1+\lambda^{-1} \right) \vmu +  \valpha \\
 \vmu + \valpha \\
\left( 1+\lambda^{-1} \right) \vmu \vmu^T + \valpha\valpha^T +  \vmu\valpha^T + \valpha \vmu^T  +  \vSigma
\end{array}
\right]
\end{equation*}

\subsection{Derivation of the NGD Update}
We consider the variational approximation using the normal inverse Gaussian distribution $q(\vlat|\vvarpar)$.
We consider the same model as Appendix \ref{app:fmg_elbo} with a Gaussian-prior $\gauss(\vlat|\mathbf{0},\delta^{-1} \vI)$ on $\vlat$.
The lower bound is defined as below:
\begin{align*}
\mathcal{L}(\vvarpar) &= \Unmyexpect{q(z|\varpar)}\sqr{  \sum_{n=1}^{N}  \log p(\data_n|\vlat) + \log \gauss(\vlat|\mathbf{0},\delta^{-1} \vI)  -\log q(\vlat|\vvarpar) }\\
&= \Unmyexpect{q(\lat)}\sqr{ - h(\vlat)}, \textrm{ where } h(\vlat)  :=-\sqr{  \log \frac{\gauss(\vlat|\mathbf{0},\delta^{-1} \vI)}{q(\vlat)} + \sum_{n} \log p(\mathcal{D}_n|\vlat)   } .
\end{align*}
Our goal is to compute the gradient of this ELBO with respect to the expectation parameters. 

\subsection{Natural Gradient for $q(\vlat|\mix)$ }
Likewise, we do not need to compute these gradients with respect to the expectation parameters directly. The gradients can be computed in terms of $\crl{\vmu,\valpha, \vSigma}$.
Similarly, by the chain rule, we can express the following gradients with respect to the expectation parameters in terms of the gradients with respect to $\vmu$, $\valpha$, and $\vSigma$.
\begin{align*}
 \nabla_{m} \mathcal{L} &=  \lambda \nabla_{\mu} \mathcal{L} -  \lambda \nabla_{\alpha} \mathcal{L} - 2 \left( \nabla_\Sigma \mathcal{L}\right) \vmu \\
 \nabla_{m_{\alpha}} \mathcal{L} &= \left( 1+\lambda \right) \nabla_{\alpha} \mathcal{L} - \lambda  \nabla_{\mu} \mathcal{L} - 2 \left( \nabla_\Sigma \mathcal{L} \right)\valpha \\
 \nabla_{M} \mathcal{L} &= \nabla_{\Sigma}\mathcal{L} 
\end{align*}

By plugging the gradients into the update in \eqref{eq:simplengvi} and then re-expressing the update in terms of $\vmu$, $\valpha$, $\vSigma$,
we obtain the natural gradient update in terms of $\vmu$, $\valpha$, and $\vSigma$.
\begin{align}
 \vSigma^{-1} & \leftarrow \vSigma^{-1} - 2 \beta \nabla_{\Sigma}\mathcal{L} \label{eq:ngd_nig_sigma} \\
 \vmu & \leftarrow \vmu + \beta \vSigma \left( \lambda \nabla_{\mu} \mathcal{L} -  \lambda \nabla_{\alpha} \mathcal{L} \right) \label{eq:ngd_nig_mu}\\
 \valpha &\leftarrow \valpha + \beta \vSigma \left( (1+\lambda) \nabla_{\alpha} \mathcal{L} - \lambda \nabla_{\mu} \mathcal{L} \right)  \label{eq:ngd_nig_alpha}
\end{align}

We can compute gradients with respect to $\vmu$, $\valpha$, and $\vSigma$ by the extended Bonnet's and Price's theorem \citep{wu-report}.
In \citet{wu-report}, they discuss the conditions of the target function $h(\vlat)$ when it comes to applying these theorems.

\begin{align*}
  \nabla_{\mu}   \elbofinal (\vvarpar )
=& -  \Unmyexpect{q(\lat)} \sqr{    \nabla_{\lat} h(\vlat)  } \quad \approx - \nabla_\lat h(\vlat) \\
  \nabla_{\alpha}   \elbofinal (\vvarpar )
=& -\Unmyexpect{ q(\mix,\lat)} \sqr{ \mix \nabla_{\lat} h_n(\vlat)}  \approx -\mix \nabla_{\lat} h_n(\vlat) \\
=& -\Unmyexpect{ q(\lat)} \sqr{ u(\vlat) \nabla_{\lat} h_n(\vlat)} \approx -u(\vlat) \nabla_{\lat} h_n(\vlat) \\
  \nabla_{\Sigma}   \elbofinal (\vvarpar )
=& - \half \Unmyexpect{q(\lat)} \sqr{   \mix \nabla_{\lat}^2 h(\vlat)  }  \quad \approx - \frac{\mix}{2} \nabla_\lat^2 h(\vlat) . \\
=& - \half \Unmyexpect{q(\lat)} \sqr{   u(\vlat) \nabla_{\lat}^2 h(\vlat)  }  \quad \approx - \frac{u(\vlat)}{2} \nabla_\lat^2 h(\vlat) . 
\end{align*}

where 
$u(\vlat) := 
\sqrt{  \frac{ \left( \vlat-\vmu \right)^T \vSigma^{-1} \left( \vlat-\vmu \right) + \lambda }{ \valpha^T \vSigma^{-1} \valpha + \lambda } }  \frac{ \mathcal{K}_{\frac{d-1}{2}} \left( \sqrt{ \left( \valpha^T \vSigma^{-1} \valpha  + \lambda \right) \left( \left(\vlat-\vmu \right)^T \vSigma^{-1} \left(\vlat-\vmu \right) + \lambda \right) } \right)} { \mathcal{K}_{\frac{d+1}{2}} \left( \sqrt{ \left( \valpha^T \vSigma^{-1} \valpha  + \lambda \right) \left( \left(\vlat-\vmu \right)^T \vSigma^{-1} \left(\vlat-\vmu \right) + \lambda \right) } \right)} $ and
$\mix  \sim \IGauss(\mix|1,\lambda), \,
  \vlat  \sim \gauss(\vlat|\vmu+\mix\valpha,\mix\vSigma).$

Directly calculating the ratio between the modified Bessel functions of the second kind is expensive and numerically unstable when $v$ is large. 
However, the ratio between consecutive order has a tight and algebraic bound \cite{ruiz2016new} as given below.
When $v \in \mathcal{R}$ and $v\geq \half$, the bound of the ratio is
\begin{align*}
 D_{2v-1}(v,x) \leq \frac{\mathcal{K}_{v-1} \left(x\right)}{\mathcal{K}_{v} \left(x\right)} \leq  D_0(v,x)
\end{align*} where function $D_\alpha(v,x)$ is defined as below.
\begin{align*}
D_{\alpha}(v,x) & := \frac{x}{\psi_{\alpha}(v,x)+\sqrt{ \left(\psi_{\alpha}(v,x)\right)^2 + x^2}} \\
\psi_{\alpha}(v,x) &:=(v-\half) - \frac{\tau_\alpha(v)}{2\sqrt{\left(\tau_\alpha(v)\right)^2+x^2}}, \,\,\, \tau_\alpha(v) := v-\frac{\alpha+1}{2} 
\end{align*}

For natural number $v \in  \mathbb{N}$,  a tighter bound (see Eq (3.10) at \citet{yang2017approximating}) with higher computation cost can  be used, where we make use of the following relationship due to Eq \eqref{eq:bessel_rl_1} and \eqref{eq:bessel_rl_2}.
\begin{align*}
\frac{\mathcal{K}_{v-1} \left(x\right)}{\mathcal{K}_{v} \left(x\right)} & =  
\frac{\mathcal{K}_{v+1} \left(x\right)}{\mathcal{K}_{v} \left(x\right)} - \frac{2v}{x} 
\end{align*}

To compute the ratio, we propose to use the following approximation when $v\geq \half$.
A similar approach to approximate the ratio between two modified Bessel functions of the first kind is used in \citet{oh2019radial} and \citet{kumar2018mises}.
\begin{align}
\frac{\mathcal{K}_{v-1} \left(x\right)}{\mathcal{K}_{v} \left(x\right) }  \approx  \frac{D_{2v-1}(v,x)+D_0(v,x)}{2} \label{eq:ratio_bessel_sec}
\end{align}

Now, we discuss how to compute $\nabla_x \log \mathcal{K}_{v} \left(x\right) $ and $\nabla_x^2 \log \mathcal{K}_{v} \left(x\right) $. The first term appears when we compute $\nabla_\lat h(\vlat)$. Similarly, the second term appears when we compute $\nabla_\lat^2 h(\vlat)$. 

First, we make use of the recurrence forms of the modified Bessel function of the second kind for $v \in \mathcal{R}$ (see page 20 at \citet{Richard-notes}).
\begin{align}
 \nabla_x \mathcal{K}_{v} \left(x\right) &= -\mathcal{K}_{v-1} \left(x\right) - \frac{v}{x} \mathcal{K}_{v} \left(x\right) \label{eq:bessel_rl_1}\\
 \nabla_x \mathcal{K}_{v} \left(x\right) &= \frac{v}{x} \mathcal{K}_{v} \left(x\right) -  \mathcal{K}_{v+1} \left(x\right) \label{eq:bessel_rl_2}
\end{align}
Using these recurrence forms, we have
\begin{align}
\frac{\nabla_{x} \mathcal{K}_{v} \left(x\right)}{\mathcal{K}_{v} \left(x\right)} &=  -\frac{\mathcal{K}_{v-1} \left(x\right)}{\mathcal{K}_{v} \left(x\right)} -\frac{v}{x} \label{eq:bessel_ratio_first_grad}
\end{align}
Furthermore, we have the following result due to the recurrence forms.
\begin{align}
 \nabla_x^2 \mathcal{K}_{v} \left(x\right)
&=\nabla_x \sqr{ \overbrace{-\mathcal{K}_{v-1} \left(x\right) - \frac{v}{x} \mathcal{K}_{v} \left(x\right)}^{ \nabla_x \mathcal{K}_{v} \left(x\right)} } \nonumber \\
 &= -\nabla_x\mathcal{K}_{v-1} \left(x\right) - \frac{v}{x} \nabla_x \mathcal{K}_{v} \left(x\right) + \frac{v}{x^2}  \mathcal{K}_{v} \left(x\right)\nonumber\\
 &= -\left( \underbrace{ \frac{v-1}{x} \mathcal{K}_{v-1} \left(x\right) -  \mathcal{K}_{v} \left(x\right) }_{ \nabla_x\mathcal{K}_{v-1} \left(x\right) } \right) - \frac{v}{x} \nabla_x \mathcal{K}_{v} \left(x\right) + \frac{v}{x^2}  \mathcal{K}_{v} \left(x\right)\nonumber\\
 &= -  \frac{v-1}{x} \mathcal{K}_{v-1} \left(x\right) + \frac{v+x^2}{x^2}  \mathcal{K}_{v} \left(x\right)  - \frac{v}{x} \nabla_x \mathcal{K}_{v} \left(x\right) \nonumber\\
 &= -  \frac{v-1}{x} \mathcal{K}_{v-1} \left(x\right) + \frac{v+x^2}{x^2}  \mathcal{K}_{v} \left(x\right)  - \frac{v}{x} \left( \underbrace{-\mathcal{K}_{v-1} \left(x\right) - \frac{v}{x} \mathcal{K}_{v} \left(x\right)}_{ \nabla_x \mathcal{K}_{v} \left(x\right)} \right) \nonumber\\
 &= \frac{1}{x} \mathcal{K}_{v-1} \left(x\right) + \frac{v+x^2+v^2}{x^2}  \mathcal{K}_{v} \left(x\right) \nonumber
\end{align} which implies that
\begin{align}
\frac{\nabla_x^2 \mathcal{K}_{v} \left(x\right)}{ \mathcal{K}_{v}} = \frac{1}{x} \frac{\mathcal{K}_{v-1} \left(x\right) }{ \mathcal{K}_{v} \left(x\right) }+ \frac{v+x^2+v^2}{x^2}  \label{eq:bessel_ratio_second_grad}
\end{align}

Using Eq \eqref{eq:bessel_ratio_first_grad} and \eqref{eq:bessel_ratio_second_grad}, we have
\begin{align*}
 \nabla_x \log \mathcal{K}_{v} \left(x\right) &=  \frac{\nabla_x \mathcal{K}_{v} \left(x\right)}{\mathcal{K}_{v} \left(x\right)} \\
&= -\frac{\mathcal{K}_{v-1} \left(x\right)}{ \mathcal{K}_{v} \left(x\right)} - \frac{v}{x}  \\
\nabla_x^2 \log \mathcal{K}_{v} \left(x\right)
& =\nabla_x \sqr{ \frac{\nabla_x \mathcal{K}_{v} \left(x\right)}{\mathcal{K}_{v} \left(x\right)} } \\
&= \frac{\nabla_x^2 \mathcal{K}_{v} \left(x\right)}{\mathcal{K}_{v} \left(x\right)} - \left( \frac{\nabla_x \mathcal{K}_{v} \left(x\right) }{\mathcal{K}_{v} \left(x\right)} \right)^2 \\
&= \frac{1}{x} \frac{ \mathcal{K}_{v-1} \left(x\right) }{ \mathcal{K}_{v} \left(x\right)  } + \frac{v+x^2+v^2}{x^2} - \left( -\frac{\mathcal{K}_{v-1} \left(x\right)}{ \mathcal{K}_{v} \left(x\right)} - \frac{v}{x}\right)^2 \\
&= \frac{1-2v}{x} \left( \frac{\mathcal{K}_{v-1} \left(x\right)}{\mathcal{K}_{v} \left(x\right)} \right) - \left(  \frac{\mathcal{K}_{v-1} \left(x\right)}{\mathcal{K}_{v} \left(x\right)}\right)^2 + \frac{v}{x^2} + 1
\end{align*} where the ratio can be approximated by Eq \eqref{eq:ratio_bessel_sec}.

   \section{Multivariate Symmetric Normal Inverse-Gaussian Distribution}
\label{app:sym_nig}
      The symmetric normal inverse-Gaussian distribution is a scale mixture distribution.
A difference between the distribution at Appendix \ref{app:nig} is shown in red. Such difference allows this distribution to have heavy tails.
\begin{align*}
q(\mix,\vlat) = \gauss(\vlat|\vmu,\textcolor{red}{ \mix^{-1} } \vSigma) \IGauss(\mix|1,\lambda)
\end{align*} where we assume $\lambda$ is fixed for simplicity and 
$\IGauss(\mix|1,\lambda)=\left(\frac{\lambda}{2\pi \mix^3} \right)^{\half}\exp\crl{ -\frac{\lambda}{2}\left(\mix+\mix^{-1}\right) + \lambda }$.

Similarly, we can show that the marginal distribution is 
\begin{align*}
q(\vlat)= \frac{\lambda^{\half}}{(2\pi)^{\frac{d+1}{2}}} \mathrm{det}\left(  \vSigma \right)^{-1/2} \exp( \lambda ) \frac{ 2 \mathcal{K}_{\frac{d-1}{2}} \left( \sqrt{ \lambda  \left( \left(\vlat-\vmu \right)^T \vSigma^{-1} \left(\vlat-\vmu \right) + \lambda \right) } \right) }{ \left(\sqrt{  \frac{ \left( \vlat-\vmu \right)^T \vSigma^{-1} \left( \vlat-\vmu \right) + \lambda  }{ \lambda } } \right)^{\frac{d-1}{2}} }
\end{align*}

Furthermore, the mixture distribution is a minimal conditional EF distribution. 
The sufficient statistics, natural parameters, and expectation parameters are summarized below:
\begin{equation*}
\left[
\begin{array}{c}
 \mix \vlat\\
\mix \vlat\vlat^T
\end{array}
\right]
\quad
\left[
\begin{array}{c}
\vSigma^{-1}\vmu\\
-\half \vSigma^{-1}
\end{array}
\right]
\quad
\left[
\begin{array}{c}
 \vmu  \\
 \vmu \vmu^T  +  \vSigma
\end{array}
\right]
\end{equation*}

Likewise, we can derive the natural-gradient update in terms of $\vmu$ and $\vSigma$ as shown below.
\begin{align*}
 \vSigma^{-1} & \leftarrow \vSigma^{-1} - 2 \beta \nabla_{\Sigma}\mathcal{L}\\ 
 \vmu & \leftarrow \vmu + \beta \vSigma  \nabla_{\mu} \mathcal{L}  
\end{align*}

Now, we discuss the gradient computation of the following lower bound.
\begin{align*}
\mathcal{L}(\vvarpar) &= \Unmyexpect{q(z|\varpar)}\sqr{  \sum_{n=1}^{N}  \underbrace{\log p(\data_n|\vlat)}_{:-f_n(\vlat)} + \log \gauss(\vlat|\mathbf{0},\delta^{-1} \vI)  -\log q(\vlat|\vvarpar) }
\end{align*}

For the prior term, we have
\begin{align}
 \Unmyexpect{q(\lat)}\sqr{ \log \gauss(\vlat|\mathbf{0}, \delta^{-1}\vI ) } = -\frac{d}{2} \log(2\pi) + \frac{d \delta}{2} - \frac{\delta}{2} \left( \vmu^T\vmu  + \left( 1+\lambda^{-1} \right) \mathrm{Tr}(\vSigma) \right) \label{eq:snig_prior}
\end{align}
Due to Eq \eqref{eq:snig_prior}, the closed-form gradients are given below.
\begin{align*}
 \vg_\mu^{\textrm{prior}}  &=  -\delta \vmu\\
 \vg_\Sigma^{\textrm{prior}}  &= -\frac{\delta\left( 1+ \lambda^{-1} \right)}{2} \vI
\end{align*}

In this example, we can re-express the entropy term as below.
\begin{align}
 \Unmyexpect{q(\lat)}\sqr{ - \log q(\vlat) } 
=& \frac{(d+1) }{2}\left(\log(2\pi) - \log \sqrt{\lambda} \right)   +  \frac{\log \mathrm{det}\left(  \vSigma \right)}{2} -\lambda    -\Unmyexpect{\IGauss(\mix|1,\lambda)\chisq(\lat_1|d)}\sqr{\log\left( \frac{2  \mathcal{K}_{\frac{d-1}{2}} \left( \sqrt{  \lambda \left( \mix^{-1} \lat_1 + \lambda \right) } \right)}{ \left( \mix^{-1} \lat_1  + \lambda \right)^{\frac{d-1}{4}}  } \right)} \label{eq:snig_ent}
\end{align} where $\chisq(\lat_1|d)$ denotes the chi-squared distribution with $d$ degrees of freedom. The expectation can be computed by the inverse Gaussian quadrature \cite{choi2018high} and the generalized  Gauss-Laguerre quadrature.

By Eq \eqref{eq:snig_ent}, the closed-form gradients of the entropy term are shown below.
\begin{align*}
\vg_\mu^{\text{entropy}}  &= 0 \\
\vg_\Sigma^{\text{entropy}}  &= \frac{\vSigma^{-1}}{2}
\end{align*}

The remaining step is to compute the gradients about $\Unmyexpect{q(\lat)}\sqr{ f_n(\vlat) }$.
To compute the gradients, we use the extended Bonnet's and Price's theorems \citep{wu-report}.
Assuming that $f_n(\vlat)$ satisfies the assumptions needed for these two theorems, we obtain the following gradient expression:
\begin{align*}
\vg_{1}^n &:=  \nabla_{\mu} \Unmyexpect{q(\lat)} \sqr{ f_n(\vlat) } = \Unmyexpect{ q(\lat)} \sqr{ \nabla_{\lat} f_n(\vlat)} \approx   \nabla_{\lat} f_n(\vlat) \\
\vg_{2}^n &:= 2 \nabla_{\Sigma} \Unmyexpect{q(\lat)} \sqr{ f_n(\vlat) } \\
&=  \Unmyexpect{q(\lat)}\sqr{ u(\vlat) \nabla_{\lat}^2 f_n(\vlat)  } \approx u(\vlat) \nabla_{\lat}^2 f_n(\vlat) \\
&=  \Unmyexpect{q(\mix,\lat)}\sqr{ \mix \nabla_{\lat}^2 f_n(\vlat)  } \approx \mix^{-1} \nabla_{\lat}^2 f_n(\vlat) 
\end{align*} 
where 
$u(\vlat)  :=   
\sqrt{  \frac{ \left( \vlat-\vmu \right)^T \vSigma^{-1} \left( \vlat-\vmu \right) + \lambda }{  \lambda } }  \frac{ \mathcal{K}_{\frac{d-3}{2}} \left( \sqrt{   \lambda \left( \left(\vlat-\vmu \right)^T \vSigma^{-1} \left(\vlat-\vmu \right) + \lambda \right) } \right)} { \mathcal{K}_{\frac{d-1}{2}} \left( \sqrt{  \lambda  \left( \left(\vlat-\vmu \right)^T \vSigma^{-1} \left(\vlat-\vmu \right) + \lambda \right) } \right)} $, where the ratio between the Bessel functions can be approximated by Eq \eqref{eq:ratio_bessel_sec} when $d\geq 2$, and
$\mix  \sim \IGauss(\mix|1,\lambda), \,
  \vlat  \sim \gauss(\vlat|\vmu,\mix^{-1} \vSigma).$

Putting together, 
we can express the gradients in the following form:
\begin{align}
 \nabla_{\mu} \mathcal{L}(\vvarpar) 
 &= - \sum_{n=1}^{N} \underbrace{ \vg_{1}^n }_{ := \vg_{\mu}^n } - \delta \vmu   \label{eq:snig_g_update_mu}
 \\
 \nabla_{\Sigma} \mathcal{L}(\vvarpar)
&= - \half \sum_{n=1}^{N}  \underbrace{ \vg_{2}^n }_{:= \vg_{S}^n} -\frac{\delta(1+\lambda^{-1})}{2} \vI  +  \half \vSigma^{-1} \label{eq:snig_g_update_sigma}
\end{align}

Similarly, for stochastic approximation, we can sub-sampling a data point $n$ and use MC samples to approximate $\vg_{1}^n$ and $\vg_{2}^n$.
Plugging these stochastic gradients into 
\begin{align*}
 \vSigma^{-1} & \leftarrow (1-\beta) \vSigma^{-1} + \beta (\delta \left( 1+\lambda^{-1} \right) \vI + N \vg_{S}^n) \\
 \vmu & \leftarrow  \vmu - \beta \vSigma (N  \vg_{\mu}^n  + \delta \vmu) 
\end{align*}

   \section{Matrix-Variate Gaussian Distribution}
   \label{app:MGauss}
      We first show that MVG is a multi-linear exponential-family distribution.
\begin{lemma}
 Matrix Gaussian distribution is a member of the multi-linear exponential family.
\end{lemma}

\begin{proof}
Let $\vVarpar_{1}= \vW$, $\vVarpar_{2}= \vU^{-1}$, and $\vVarpar_{3}= \vV^{-1}$.  The distribution on $\vLat \in \mathcal{R}^{d \times p}$ can be expressed as follows.
\begin{align*}
& \mgauss(\vLat|\vW,\vU,\vV) \\ \nonumber
& = \left( 2\pi \right)^{-dp/2} \exp \sqr{  -\half \mathrm{Tr} \left( \vV^{-1} (\vLat-\vW)^T \vU^{-1} (\vLat-\vW)  \right)  - \left( d/2  \log \mathrm{Det}(\vV)  + p/2  \log \mathrm{Det}(\vU)   \right) } \\ \nonumber
&=  \left( 2\pi \right)^{-dp/2} \exp \Big\{    \mathrm{Tr}\left( \vVarpar_{3} \left( -\half \vLat + \vVarpar_{1}  \right) ^T \vVarpar_{2} \vLat  \right)  \nonumber\\ 
&\quad\quad\quad -  \half\Big[  \mathrm{Tr}\left( \vVarpar_{3} \vVarpar_{1}^T \vVarpar_{2} \vVarpar_{1}  \right) + d   \log \mathrm{Det}(\vVarpar_{3})  + p   \log \mathrm{Det}(\vVarpar_{2})    \Big] \Big\} .
\end{align*} 
The function $ \mathrm{Tr}\left( \vVarpar_{3} \left( -\half \vLat + \vVarpar_{1}  \right) ^T \vVarpar_{2} \vLat  \right)$ is linear with respect each $\vVarpar_j$ given others.
\end{proof}


We now derive the NGVI update using our new expectation parameterization.
We can obtain function $\phi_1$, $\phi_2$, and $\phi_3$ from the multi-linear function
\begin{align*} 
f(\vZ,\vVarpar) := \mathrm{Tr}\left( \vVarpar_{3} \left( -\half \vLat + \vVarpar_{1}  \right) ^T \vVarpar_{2} \vLat  \right).
\end{align*}
For example, we can obtain function $\phi_1$ from $f(\vZ,\vVarpar)$ as shown below:
\begin{align*}
f(\vZ,\vVarpar) & =  \myang{ \vVarpar_{1}, \underbrace{  \vVarpar_{2} \vLat  \vVarpar_{3}}_{\phi_1\rnd{\boldsymbol{\Lat},\boldsymbol{\Lambda}_{-1}}} } \underbrace{ - \half \mathrm{Tr}\left(  \vVarpar_{3} \vLat^T \vVarpar_{2} \vLat \right) }_{r_1\rnd{ \boldsymbol{\Lat}, \boldsymbol{\Lambda}_{-1}} }.
\end{align*}
Similarly, we can obtain functions $\phi_2$ and $\phi_3$.
The corresponding expectation parameters of the Matrix Gaussian distribution can then be derived as below:
\begin{align*}
\vVarmean_{1} &= \Unmyexpect{\mgauss(\Lat|W,U,V)} \sqr{ \vVarpar_{2} \vLat  \vVarpar_{3} } =  \vVarpar_{2}\vVarpar_{1}\vVarpar_{3} \\
\vVarmean_{2} &=  \Unmyexpect{\mgauss(\Lat|W,U,V)} \sqr{  -\half \vLat \vVarpar_{3} \vLat^T +\vLat \vVarpar_{3} \vVarpar_{1}^T    } =   \half \left(  \vVarpar_{1} \vVarpar_{3} \vVarpar_{1}^T - p \vVarpar_{2}^{-1}  \right) \\ 
\vVarmean_{3} &= \Unmyexpect{\mgauss(\Lat|W,U,V)} \sqr{  -\half \vLat^T \vVarpar_{2} \vLat +\vVarpar_{1}^T   \vVarpar_{2}  \vLat  } = \half \left(  \vVarpar_{1}^T \vVarpar_{2} \vVarpar_{1} - d \vVarpar_{3}^{-1}  \right)
\end{align*}
We can then compute the gradient with respect to the expectation parameters using chain-rule:
\begin{align*}
 \nabla_{\Varmean_{1}}   \Unmyexpect{q(\Lat|\varpar)}  \sqr{ h(\vLat) } &= \left( \vVarpar_{2} \right) ^{-1} \nabla_{W}   \Unmyexpect{\mgauss(\Lat|W,U,V)}  \sqr{ h(\vLat) }  \left( \vVarpar_{3} \right)^{-1} \\
 \nabla_{\Varmean_{2}}    \Unmyexpect{q(\Lat|\varpar)}  \sqr{ h(\vLat) }   &=\frac{-2}{p} \nabla_{U}   \Unmyexpect{\mgauss(\Lat|W,U,V)}  \sqr{ h(\vLat) }  \\
 \nabla_{\Varmean_{3}}  \Unmyexpect{q(\Lat|\varpar)} \sqr{ h(\vLat) }    &=\frac{-2}{d} \nabla_{V}     \Unmyexpect{\mgauss(\Lat|W,U,V)}  \sqr{ h(\vLat) }   
\end{align*} 
We will now express the gradients in terms of the gradient of the function $h(\vLat)$. This leads to a simple update because gradient of $h(\vZ)$ can be obtained using automatic gradients (or backpropagation when using a neural network).
Let $\vlat=\mathrm{vec}(\vLat)$ and $\vLat=\mathrm{Mat}(\vlat)$. The distribution can be re-expressed as a multivariate Gaussian distribution 
$\gauss(\vlat| \vmu, \vSigma)$, where $\vmu=\mathrm{vec}(\vW)$, $\vSigma=\vV \otimes \vU$, and $\otimes$ denotes the Kronecker product.
Furthermore, the lower bound can be re-expressed as $\Unmyexpect{\gauss(\lat|\mu,\Sigma)} \sqr{ -\hat{h}(\vlat) }$, where $\hat{h}(\vlat) = h(\vLat)$. We make use of the Bonnet's and Price's theorems \citep{opper2009variational}:
\begin{align*}
\nabla_{\mu}  \Unmyexpect{\gauss(\lat|\mu,\Sigma)} \sqr{ \hat{h}(\vlat) }  & =   \Unmyexpect{\gauss(\lat|\mu,\Sigma)} \sqr{  \nabla_{\lat} \hat{h}(\vlat) } \\
\nabla_{\Sigma} \Unmyexpect{\gauss(\lat|\mu,\Sigma)} \sqr{  \nabla_{\lat} \hat{h}(\vlat) }  & =  \half  \Unmyexpect{\gauss(\lat|\mu,\Sigma)} \sqr{  \nabla_{\lat}^2 \hat{h}(\vlat) } 
\end{align*} 

Since $\vSigma=\vV \otimes \vU$, we have the following  result. 
\begin{align*}
& \mathrm{Tr} \sqr{ \left( \nabla_{U_{ij}} \vSigma \right)   \Unmyexpect{\gauss(\lat|\mu,\Sigma)} \sqr{   \nabla_{\lat} \hat{h}(\vlat) \nabla_{\lat} \hat{h}(\vlat)^T   } } \\
= &   \mathrm{Tr} \sqr{   \Unmyexpect{\gauss(\lat|\mu,\Sigma)} \sqr{  \left( \nabla_{U_{ij}} \vSigma \right)  \nabla_{\lat} \hat{h}(\vlat) \nabla_{\lat} \hat{h}(\vlat)^T   } } \\  
= &   \mathrm{Tr} \sqr{   \Unmyexpect{\gauss(\lat|\mu,\Sigma)} \sqr{  \nabla_{\lat} \hat{h}(\vlat)^T \left( \nabla_{U_{ij}} \vSigma \right) \nabla_{\lat} \hat{h}(\vlat)    } } \\  
= &  \mathrm{Tr} \sqr{   \Unmyexpect{\gauss(\lat|\mu,\Sigma)} \sqr{  \nabla_{\lat} \hat{h}(\vlat)^T \left( \nabla_{U_{ij}} (\vV \otimes \vU) \right) \nabla_{\lat} \hat{h}(\vlat)    } } \\  
=  & \mathrm{Tr} \sqr{   \Unmyexpect{\gauss(\lat|\mu,\Sigma)} \sqr{  \nabla_{\lat} \hat{h}(\vlat)^T \left(  \vV \otimes \nabla_{U_{ij}}\vU \right) \nabla_{\lat} \hat{h}(\vlat)    } }  \\
=  & \mathrm{Tr} \sqr{   \Unmyexpect{\mgauss(\Lat|W,U,V)} \sqr{ \mathrm{vec}( \nabla_{\Lat} h(\vLat))^T \left(  \vV \otimes \nabla_{U_{ij}}\vU \right)\mathrm{vec}( \nabla_{\Lat} h(\vLat))    } }  \\
= &  \mathrm{Tr} \sqr{   \Unmyexpect{\mgauss(\Lat|W,U,V)} \sqr{ \mathrm{vec}( \nabla_{\Lat} h(\vLat) )^T \mathrm{vec}\left(    (\nabla_{U_{ij}}\vU) \nabla_{\Lat} h(\vLat)  \vV^T \right)   } }  \\
= &  \mathrm{Tr} \sqr{   \Unmyexpect{\mgauss(\Lat|W,U,V)} \sqr{  \nabla_{\Lat} h(\vLat)^T\left(    (\nabla_{U_{ij}}\vU) \nabla_{\Lat} h(\vLat)  \vV^T \right)   } }  \\
= &  \mathrm{Tr} \sqr{   \Unmyexpect{\mgauss(\Lat|W,U,V)} \sqr{ \nabla_{\Lat} h(\vLat)  \vV^T \nabla_{\Lat} h(\vLat)^T    (\nabla_{U_{ij}}\vU)   } }  \\
= &  \mathrm{Tr} \sqr{   \Unmyexpect{\mgauss(\Lat|W,U,V)} \sqr{ \nabla_{\Lat} h(\vLat)  \vV \nabla_{\Lat} h(\vLat)^T    (\nabla_{U_{ij}}\vU)   } } 
\end{align*} where we use the identity of the Kronecker product, $(\vB^T \otimes \vA) \underbrace{\mathrm{vec}(\vX)}_{  \vx } = \mathrm{vec}(\vA \vX \vB) $ to move from step 6 to step 7.

Therefore, we have the following identity.
\begin{align*}
 \left( \nabla_{U} \vSigma \right)   \Unmyexpect{\gauss(\lat|\mu,\Sigma)} \sqr{   \nabla_{\lat} \hat{h}(\vlat) \nabla_{\lat} \hat{h}(\vlat)^T   }  
= &     \Unmyexpect{\mgauss(\Lat|W,U,V)} \sqr{ \nabla_{\Lat} h(\vLat)  \vV \nabla_{\Lat} h(\vLat)^T     }  
\end{align*}

Similarly, we have
\begin{align*}
\left( \nabla_{V} \vSigma \right)   \Unmyexpect{\gauss(\lat|\mu,\Sigma)} \sqr{   \nabla_{\lat} \hat{h}(\vlat) \nabla_{\lat} \hat{h}(\vlat)^T   }  
&=  \Unmyexpect{\mgauss(\Lat|W,U,V)} \sqr{     \nabla_{\Lat} h(\vLat)^T    \vU    \nabla_{\Lat} h(\vLat) }. \nonumber
\end{align*}

These identities can be used to express the gradient with respect to the expectation parameters in terms of the gradient with respect to $\vZ$:
\begin{align}
\nabla_{W}  \Unmyexpect{\mgauss(\Lat|W,U,V)}  \sqr{ h(\vLat) } &= \mathrm{Mat} \left(  \Unmyexpect{\gauss(\lat|\mu,\Sigma)} \sqr{   \nabla_{\lat} \hat{h}(\vlat)   } \right)\nonumber \\
&=  \Unmyexpect{\mgauss(\Lat|W,U,V)} \sqr{\nabla_{\Lat} h(\vLat)   } \nonumber \\
\nabla_{U}  \Unmyexpect{\mgauss(\Lat|W,U,V)}  \sqr{ h(\vLat) } &  = \left( \nabla_{U} \vSigma \right) \nabla_{\Sigma} \Unmyexpect{\gauss(\lat|\mu,\Sigma)} \sqr{   \nabla_{\lat} \hat{h}(\vlat)   } \nonumber   \\
&= \half \left( \nabla_{U} \vSigma \right)   \Unmyexpect{\gauss(\lat|\mu,\Sigma)} \sqr{  \nabla_{\lat}^2   \hat{h}(\vlat)   } \nonumber \\
&\approx \half \left( \nabla_{U} \vSigma \right)   \Unmyexpect{\gauss(\lat|\mu,\Sigma)} \sqr{   \nabla_{\lat} \hat{h}(\vlat) \nabla_{\lat} \hat{h}(\vlat)^T   } \label{aeq:gn1} \\ 
&= \half  \Unmyexpect{\mgauss(\Lat|W,U,V)} \sqr{     \nabla_{\Lat} h(\vLat)   \vV    \nabla_{\Lat} h(\vLat)^T } \nonumber \\ 
\nabla_{V}  \Unmyexpect{\mgauss(\Lat|W,U,V)}  \sqr{ h(\vLat) } &  = \left( \nabla_{V} \vSigma \right) \nabla_{\Sigma} \Unmyexpect{\gauss(\lat|\mu,\Sigma)} \sqr{   \nabla_{\lat} \hat{h}(\vlat)   }  \nonumber \\
&= \half \left( \nabla_{V} \vSigma \right)   \Unmyexpect{\gauss(\lat|\mu,\Sigma)} \sqr{  \nabla_{\lat}^2   \hat{h}(\vlat)   } \nonumber \\
&\approx \half \left( \nabla_{V} \vSigma \right)   \Unmyexpect{\gauss(\lat|\mu,\Sigma)} \sqr{   \nabla_{\lat} \hat{h}(\vlat) \nabla_{\lat} \hat{h}(\vlat)^T   } \label{aeq:gn2} \\
&= \half  \Unmyexpect{\mgauss(\Lat|W,U,V)} \sqr{     \nabla_{\Lat} h(\vLat)^T    \vU    \nabla_{\Lat} h(\vLat) }. \nonumber
\end{align}
To avoid computation of the Hessian, we have used the Gauss-Newton approximation  \citep{khan18a} in Eq \eqref{aeq:gn1} and Eq. \eqref{aeq:gn2}. 
 
We choose the step-size as $\beta=\{\beta_1, p \beta_2, d \beta_2\}$. The  update with the Gauss-Newton approximation can be expressed as
\begin{align*}
\vVarpar_{1} &\leftarrow  \vVarpar_{1} - \beta_1 \left( \vVarpar_{2} \right) ^{-1}     \Unmyexpect{\mgauss(\Lat|W,U,V)} \sqr{\nabla_{\Lat} h(\vLat)   } \left( \vVarpar_{3} \right)^{-1} \\
\vVarpar_{2} &\leftarrow \vVarpar_{2} + \beta_2 \Unmyexpect{\mgauss(\Lat|W,U,V)} \sqr{     \nabla_{\Lat} h(\vLat)   \vV    \nabla_{\Lat} h(\vLat)^T }  \\
\vVarpar_{3}  &\leftarrow \vVarpar_{3} + \beta_2  \Unmyexpect{\mgauss(\Lat|W,U,V)} \sqr{     \nabla_{\Lat} h(\vLat)^T    \vU   \nabla_{\Lat} h(\vLat) }
\end{align*}  
We can re-express these in terms of $\{\vW,\vU^{-1},\vV^{-1}\}$ to get the final updates:
\begin{align*}
\vW  &\leftarrow \vW - \beta_1  \vU   \Unmyexpect{\mgauss(\Lat|W,U,V)} \sqr{\nabla_{\Lat} h(\vLat)   }  \vV \\
\left( \vU \right)^{-1}  &\leftarrow  \left( \vU \right)^{-1} +  \beta_2 \Unmyexpect{\mgauss(\Lat|W,U,V)} \sqr{     \nabla_{\Lat} h(\vLat)   \vV   \nabla_{\Lat} h(\vLat)^T }  \\
\left( \vV \right)^{-1}  &\leftarrow  \left( \vV \right)^{-1} + \beta_2  \Unmyexpect{\mgauss(\Lat|W,U,V)} \sqr{     \nabla_{\Lat} h(\vLat)^T    \vU   \nabla_{\Lat} h(\vLat) }
\end{align*}

   \section{Extensions to Variational Adam}
    \label{app:vadam}
      For simplicity, we consider a case when variational parameters of $q(\mix|\vvarpar_\mix)$ are fixed.
Since $\vvarpar_\mix$ is fixed, using the same derivation as \citet{khan18a}, we obtain the following natural-gradient update with the natural momentum ($0\leq m<1$).
\begin{align}
   \vvarpar_\lat^{t+1} = \frac{1}{1-m} \vvarpar_\lat^{t} - \frac{m}{1-m} \vvarpar_\lat^{t-1} + \frac{\beta}{1-m} \nabla_{\varmean_\lat} \elbofinal(\vvarpar_\lat)\Big|_{\vvarpar_\lat=\vvarpar_\lat^{t}} \label{eq:ngd_mom}
\end{align} 

We assume the model prior is a Gaussian prior $p(\vlat)=\gauss(\vlat|\mathbf{0}, \delta^{-1} \vI )$ to derive extensions of the variational Adam update, 
where the variational distribution is a Gaussian mixture distribution such as skew Gaussian, exponentially modified Gaussian, symmetric normal inverse-Gaussian, and Student's t-distribution.

\subsection{Extension for Skew Gaussian}
 Re-expressing the update \eqref{eq:ngd_mom} in terms of $\vmu$, $\valpha$, $\vSigma$ (the same derivation as \citet{khan18a}), we obtain the following update:
\begin{align*}
 \vSigma_{t+1}^{-1} & = \vSigma_t^{-1} - 2 \frac{\beta}{1-m} \nabla_{\Sigma_t}\mathcal{L} + \frac{m}{1-m} \left(  \vSigma_{t}^{-1}  - \vSigma_{t-1}^{-1}\right) \\
 \vmu_{t+1} & = \vmu_t + \frac{\beta}{1-m} \vSigma_{t+1} \left( \frac{1}{1-c^2} \nabla_{\mu_t} \mathcal{L} - \frac{c}{1-c^2} \nabla_{\alpha_t} \mathcal{L} \right) + \frac{m}{1-m} \vSigma_{t+1} \vSigma_{t-1}^{-1} (\vmu_t - \vmu_{t-1} )  \\
 \valpha_{t+1} &= \valpha_{t} + \frac{\beta}{1-m} \vSigma_{t+1} \left( \frac{1}{1-c^2} \nabla_{\alpha_t} \mathcal{L} - \frac{c}{1-c^2} \nabla_{\mu_t} \mathcal{L} \right) + \frac{m}{1-m} \vSigma_{t+1} \vSigma_{t-1}^{-1} (\valpha_t - \valpha_{t-1} )
\end{align*}
where we use a skew Gaussian distribution as the variational distribution,
$\nabla_{\mu_t} \mathcal{L}$,
$\nabla_{\alpha_t} \mathcal{L}$, and
$\nabla_{\Sigma_t} \mathcal{L}$ are defined at
\eqref{eq:ske_g_update_mu} -\eqref{eq:ske_g_update_sigma}.

We make use of the same approximations as \citet{khan18a} such as the gradient-magnitude of the Hessian approximation, the square root approximation, $\vSigma_{t-1}\approx \vSigma_{t}$, and a diagonal covariance structure in $\vSigma$ to obtain an extension of the variational skew-Adam update.
 Recall that $\vg_{\alpha}^{\textrm{entropy}}$ and $\vg_{\Sigma}^{\textrm{entropy}}$ are defined at \eqref{eq:skew_ent_grad_alpha} and \eqref{eq:skew_ent_grad_Sigma} and $c=\sqrt{\frac{2}{\pi}}$.
Using the same algebra manipulation used in \citet{khan18a}, we obtain the variational Adam  update with Gaussian prior $p(\vlat)=\gauss(\vlat|\mathbf{0},\delta^{-1} \vI)$,
where  $\vSigma^{-1}=\mathrm{Diag}(N\vs+\delta)$, $v = \frac{c}{( 1+\valpha^T \vSigma^{-1} \valpha )}$,  and $u(\vlat)  = \frac{\left( \vlat-\vmu \right)^T \vSigma^{-1} \valpha }{1 + \valpha^T \vSigma^{-1} \valpha }$.
\fbox{
			\begin{minipage}{.85\textwidth}
				\textbf{Skew Gaussian extension}
				\begin{algorithmic}[1]
					\WHILE{not converged}
            \STATE $\hat{\vlat} \leftarrow \vmu {\, +\, \vsigma\circ \vepsilon}$, where $\vepsilon \sim \gauss(\mathbf{0},\vI)$, $\vsigma\leftarrow 1/\sqrt{N\vs + \delta}$
            \STATE $\vlat \leftarrow \hat{\vlat} {\, +\, |\mix| \valpha}$, where $\mix \sim \gauss(0,1)$
               \STATE Randomly sample a data example $\data_i$
               \STATE $\vg_{\mu} \leftarrow - \nabla \log p(\data_i|\vlat) $
               \STATE \textbf{option I:} $\hat{\vg}_{\alpha} \leftarrow - |\mix| \nabla \log p(\data_i|\vlat)$
	      \STATE \textbf{option II:} $\hat{\vg}_{\alpha} \leftarrow - \sqr{ v \nabla \log p(\data_i|\hat{\vlat}) + u(\vlat) \nabla \log p(\data_i|\vlat) } $
               \STATE $\vg_{\alpha} \leftarrow \hat{\vg}_{\alpha} - \vg_{\alpha}^{\textrm{entropy}}/N $  
               \STATE $\vg_{s} \leftarrow  \vg_{\mu} \circ \vg_{\mu} -  \mathrm{diag}\left(2 \vg_{\Sigma}^{\textrm{entropy}}\right)/N + (\vs + \delta/N)$  
               
					\STATE $\vm_{\mu} \leftarrow \gamma_1 \, \vm_{\mu} + (1-\gamma_1) \, \rnd{ \frac{\vg_{\mu}- c\vg_{\alpha}}{1-c^2} { \,\, + \,\, \delta \vmu/N } }$ 
					\STATE $\vm_{\alpha} \leftarrow \gamma_1 \, \vm_{\alpha} + (1-\gamma_1) \, \rnd{ \frac{\vg_{\alpha}- c\vg_{\mu}}{1-c^2} { \,\, + \,\, \delta \valpha/N } }$ 
               \STATE $\vs \leftarrow \gamma_2 \, \vs + (1-\gamma_2) \, \vg_s $ 
					\STATE $\hat{\vm}_\mu \leftarrow \vm_\mu/(1-\gamma_1^t), \quad \hat{\vm}_\alpha \leftarrow \vm_\alpha/(1-\gamma_1^t), \quad \hat{\vs} \leftarrow  (\vs+ \delta/N)/(1-\gamma_2^t)$
               \STATE $\vmu \leftarrow \vmu - \beta \,\, \hat{\vm}_\mu / \sqrt{\hat{\vs}} $, \quad
                $\valpha \leftarrow \valpha - \beta \,\, \hat{\vm}_\alpha / \sqrt{\hat{\vs}})$
               
					\STATE $t \leftarrow t + 1$
					\ENDWHILE
				\end{algorithmic}
	\end{minipage}
	}

\subsection{Extension for Exponentially Modified Gaussian}
Similarly, re-expressing the update \eqref{eq:ngd_mom} in terms of $\vmu$, $\valpha$, $\vSigma$, we obtain the following update:
\begin{align*}
 \vSigma_{t+1}^{-1} & = \vSigma_t^{-1} - 2 \frac{\beta}{1-m} \nabla_{\Sigma_t}\mathcal{L} + \frac{m}{1-m} \left(  \vSigma_{t}^{-1}  - \vSigma_{t-1}^{-1}\right) \\
 \vmu_{t+1} & = \vmu_t + \frac{\beta}{1-m} \vSigma_{t+1} \left( 2 \nabla_{\mu_t} \mathcal{L} - \nabla_{\alpha_t} \mathcal{L} \right) + \frac{m}{1-m} \vSigma_{t+1} \vSigma_{t-1}^{-1} (\vmu_t - \vmu_{t-1} )  \\
 \valpha_{t+1} &= \valpha_{t} + \frac{\beta}{1-m} \vSigma_{t+1} \left(  \nabla_{\alpha_t} \mathcal{L} -  \nabla_{\mu_t} \mathcal{L} \right) + \frac{m}{1-m} \vSigma_{t+1} \vSigma_{t-1}^{-1} (\valpha_t - \valpha_{t-1} )
\end{align*}
where we use an exponentially modified Gaussian distribution as the variational distribution, 
$\nabla_{\mu_t} \mathcal{L}$,
$\nabla_{\alpha_t} \mathcal{L}$, and
$\nabla_{\Sigma_t} \mathcal{L}$ are defined at
\eqref{eq:exp_g_update_mu} -\eqref{eq:exp_g_update_sigma}.

Likewise, we can obtain the variational Adam update with  Gaussian prior $p(\vlat)=\gauss(\vlat|\mathbf{0},\delta^{-1} \vI)$ as shown below, where 
$\vg_{\alpha}^{\textrm{entropy}}$ and $\vg_{\Sigma}^{\textrm{entropy}}$ are defined at \eqref{eq:exp_ent_grad_alpha} and \eqref{eq:exp_ent_grad_Sigma}, 
$\vSigma^{-1}=\mathrm{Diag}(N\vs+\delta)$,
$v = \frac{1}{(\valpha^T \vSigma^{-1} \valpha )}$, 
and  $u(\vlat)  = \frac{\left( \vlat-\vmu \right)^T \vSigma^{-1} \valpha - 1 }{ \valpha^T \vSigma^{-1} \valpha }$
\fbox{
			\begin{minipage}{.85\textwidth}
				\textbf{Exponentially Modified Gaussian extension}
				\begin{algorithmic}[1]
					\WHILE{not converged}
            \STATE $\hat{\vlat} \leftarrow \vmu {\, +\, \vsigma\circ \vepsilon}$, where $\vepsilon \sim \gauss(\mathbf{0},\vI)$, $\vsigma\leftarrow 1/\sqrt{N\vs + \delta}$
            \STATE $\vlat \leftarrow \hat{\vlat} {\, +\, \mix \valpha}$, where $\mix \sim \expdist(1)$
               \STATE Randomly sample a data example $\data_i$
               \STATE $\vg_{\mu} \leftarrow - \nabla \log p(\data_i|\vlat) $
               \STATE \textbf{option I:} $\hat{\vg}_{\alpha} \leftarrow - \mix \nabla \log p(\data_i|\vlat)$
	      \STATE \textbf{option II:} $\hat{\vg}_{\alpha} \leftarrow - \sqr{ v \nabla \log p(\data_i|\hat{\vlat}) + u(\vlat) \nabla \log p(\data_i|\vlat) } $
               \STATE $\vg_{\alpha} \leftarrow \hat{\vg}_{\alpha} - \vg_{\alpha}^{\textrm{entropy}}/N $  
               \STATE $\vg_{s} \leftarrow  \vg_{\mu} \circ \vg_{\mu} -  \mathrm{diag}\left(2 \vg_{\Sigma}^{\textrm{entropy}}\right)/N + (\vs + \delta/N)$  
               
					\STATE $\vm_{\mu} \leftarrow \gamma_1 \, \vm_{\mu} + (1-\gamma_1) \, \rnd{ 2 \vg_{\mu}- \vg_{\alpha} { \,\, + \,\, \delta \vmu/N } }$ 
					\STATE $\vm_{\alpha} \leftarrow \gamma_1 \, \vm_{\alpha} + (1-\gamma_1) \, \rnd{ \vg_{\alpha}- \vg_{\mu} { \,\, + \,\, \delta \valpha/N } }$ 
               \STATE $\vs \leftarrow \gamma_2 \, \vs + (1-\gamma_2) \, \vg_s $ 
					\STATE $\hat{\vm}_\mu \leftarrow \vm_\mu/(1-\gamma_1^t), \quad \hat{\vm}_\alpha \leftarrow \vm_\alpha/(1-\gamma_1^t), \quad \hat{\vs} \leftarrow (\vs+  \delta/N)/(1-\gamma_2^t)$
               \STATE $\vmu \leftarrow \vmu - \beta \,\, \hat{\vm}_\mu / \sqrt{\hat{\vs} }$, \quad
                $\valpha \leftarrow \valpha - \beta \,\, \hat{\vm}_\alpha / \sqrt{\hat{\vs}}$
               
					\STATE $t \leftarrow t + 1$
					\ENDWHILE
				\end{algorithmic}
	\end{minipage}
	}

\subsection{Extension for Student's t-distribution}
Likewise, re-expressing the update \eqref{eq:ngd_mom} in terms of $\vmu$, $\vSigma$, we obtain the following update:
\begin{align*}
 \vSigma_{t+1}^{-1} & = \vSigma_t^{-1} - 2 \frac{\beta}{1-m} \nabla_{\Sigma_t}\mathcal{L} + \frac{m}{1-m} \left(  \vSigma_{t}^{-1}  - \vSigma_{t-1}^{-1}\right) \\
 \vmu_{t+1} & = \vmu_t + \frac{\beta}{1-m} \vSigma_{t+1}  \nabla_{\mu_t} \mathcal{L}  + \frac{m}{1-m} \vSigma_{t+1} \vSigma_{t-1}^{-1} (\vmu_t - \vmu_{t-1} )  
\end{align*}
where we use a Student's t-distribution with fixed $\alpha>1$ as the variational distribution, $\nabla_{\mu_t}\mathcal{L}$
and 
$\nabla_{\Sigma_t}\mathcal{L}$
are defined at
\eqref{eq:st_grad_mu} - \eqref{eq:st_grad_sigma}.

Now, we consider the following lower bound ($\vlat \in \mathcal{R}^d$).
\begin{align*}
\mathcal{L}(\vvarpar) = \Unmyexpect{q(z|\varpar)}\sqr{  \sum_{n=1}^{N} \underbrace{\log p(\data_n|\vlat)}_{-f_n(\vlat)} + \log \gauss(\vz|\mathbf{0},\delta^{-1} \vI)  -\log q(\vz|\vvarpar) }.
\end{align*} where
\begin{align*}
 q(\vlat)  &=   \mathrm{det}\left(\pi \vSigma \right)^{-1/2}  \frac{   \Gamma(\alpha+d/2)  \left(   2\alpha + \left(\vlat-\vmu\right)^T \vSigma^{-1}  \left(\vlat-\vmu\right)    \right)^{-\alpha-d/2} }{   \Gamma(\alpha)     \left( 2\alpha\right)^{-\alpha} } .
\end{align*}

We use the results from \citet{kotz2004multivariate}.
\begin{align*}
\Unmyexpect{q(\lat|\varpar)}\sqr{ -\log q(\vlat|\vvarpar) } &= \half \log |\vSigma| +  \log \frac{ (2a \pi)^{d/2}}{ \Gamma(d/2)} + \log\frac{ \Gamma(d/2) \Gamma(a)}{\Gamma(d/2+a)} + (a+d/2) \left(  \psi \left( a + d/2 \right) -  \psi \left( a \right) \right)\\
\Unmyexpect{q(\lat|\varpar)}\sqr{ \log \gauss(\vlat|\mathbf{0},\delta^{-1} \vI ) } &= -\frac{d}{2}\log (2\pi) + \frac{d \log(\delta)}{2} - \frac{\delta}{2} \vmu^T \vmu - \frac{\delta}{2} \frac{a}{a-1} \mathrm{Tr}\left(\vSigma\right)
\end{align*}
where $\psi(\cdot)$ is the digamma function.

The remaining thing is to compute the gradients about $\Unmyexpect{q(\lat)}\sqr{ f_n(\vlat) }$.
To compute the gradients, the reparametrization trick can be used. However, we can do better by the extended
Bonnet's and Price's theorems for Student's t-distribution \citep{wu-report}.
Assuming that $f_n(\vlat)$ satisfies the assumptions needed for these two theorems, we obtain the following gradient expression:
\begin{align*}
\vg_{1}^n &:=  \nabla_{\mu} \Unmyexpect{q(\lat)} \sqr{ f_n(\vlat) } = \Unmyexpect{ q(\lat)} \sqr{ \nabla_{\lat} f_n(\vlat)} \approx   \nabla_{\lat} f_n(\vlat) \\
\vg_{2}^n &:= 2 \nabla_{\Sigma} \Unmyexpect{q(\lat)} \sqr{ f_n(\vlat) } \\
&=  \Unmyexpect{q(\lat)}\sqr{ u(\vlat) \nabla_{\lat}^2 f_n(\vlat)  } \approx  u(\vlat) \nabla_{\lat}^2 f_n(\vlat) \\
&=  \Unmyexpect{q(\mix,\lat)}\sqr{ \mix \nabla_{\lat}^2 f_n(\vlat)  } \approx  \mix \nabla_{\lat}^2 f_n(\vlat) 
\end{align*} 
where $\vlat \in \mathcal{R}^d$ is generated from $q(\vlat)$, $\mix$ is generated from $q(\mix)$ , and
\begin{align*}
   u(\vlat) := \frac{a + \half \left( \vlat-\vmu \right)^T \vSigma^{-1} \left( \vlat-\vmu \right) }{(a+d/2-1)}  
\end{align*}

The gradients of $\mathcal{L}(\vvarpar)$ can be expressed as
\begin{align}
\nabla_{\mu} \mathcal{L}(\vvarpar)&= -\sum_{n=1}^{N}  \vg_1^n - \delta \vmu \label{eq:st_grad_mu} \\
\nabla_{\Sigma} \mathcal{L}(\vvarpar)&= -\half \sum_{n=1}^{N}  \vg_2^n - \frac{\delta}{2}  \frac{a}{a-1} \vI + \half \vSigma^{-1}  \label{eq:st_grad_sigma}
\end{align}

Likewise, we can obtain the variational Adam update with Gaussian prior $p(\vlat)=\gauss(\vlat|\mathbf{0},\delta^{-1} \vI)$ as shown below,
where 
$\vlat \in \mathcal{R}^d$,
$\vSigma^{-1}=\mathrm{Diag}(N\vs+\frac{\alpha\delta}{\alpha-1})$, and
$u(\vlat) = \frac{a + \half \left( \vlat-\vmu \right)^T \vSigma^{-1} \left( \vlat-\vmu \right) }{(a+d/2-1)}$.
\fbox{
			\begin{minipage}{.85\textwidth}
				\textbf{Student's t ($\alpha>1$) extension}
				\begin{algorithmic}[1]
					\WHILE{not converged}
            \STATE $\vlat \leftarrow \vmu {\, +\,  \vsigma\circ \vepsilon}$, where $\mix \sim \mathcal{IG}(\alpha,\alpha)$, $\vepsilon \sim \gauss(\mathbf{0},\vI)$, $\vsigma\leftarrow \sqrt{\mix/(N\vs + \frac{\alpha \delta}{\alpha-1} )}$
               \STATE Randomly sample a data example $\data_i$
               \STATE $\vg_{\mu} \leftarrow - \nabla \log p(\data_i|\vlat) $
               \STATE \textbf{option I:}
                $\vg_{s} \leftarrow  \mix \vg_{\mu} \circ  \vg_{\mu}   $  
	      \STATE \textbf{option II:}
                $\vg_{s} \leftarrow  u(\vlat) \vg_{\mu} \circ \vg_{\mu} $  
               
					\STATE $\vm_{\mu} \leftarrow \gamma_1 \, \vm_{\mu} + (1-\gamma_1) \, \rnd{  \vg_{\mu} { \,\, + \,\, \delta \vmu/N } }$ 
               \STATE $\vs \leftarrow \gamma_2 \, \vs + (1-\gamma_2) \, \vg_s $ 
					\STATE $\hat{\vm}_\mu \leftarrow \vm_\mu/(1-\gamma_1^t),  \quad \hat{\vs} \leftarrow (\vs+  \frac{\alpha\delta}{N (\alpha-1)} )/(1-\gamma_2^t)$
               \STATE $\vmu \leftarrow \vmu - \beta \,\, \hat{\vm}_\mu / \sqrt{\hat{\vs} }$
               
					\STATE $t \leftarrow t + 1$
					\ENDWHILE
				\end{algorithmic}
	\end{minipage}
	}

\subsection{Extension for Symmetric Normal Inverse-Gaussian Distribution}
 Re-expressing the update \eqref{eq:ngd_mom} in terms of $\vmu$ and $\vSigma$ (the same derivation as \citet{khan18a}), we obtain the following update:
\begin{align*}
 \vSigma_{t+1}^{-1} & = \vSigma_t^{-1} - 2 \frac{\beta}{1-m} \nabla_{\Sigma_t}\mathcal{L} + \frac{m}{1-m} \left(  \vSigma_{t}^{-1}  - \vSigma_{t-1}^{-1}\right) \\
 \vmu_{t+1} & = \vmu_t + \frac{\beta}{1-m} \vSigma_{t+1} \nabla_{\mu_t} \mathcal{L}  + \frac{m}{1-m} \vSigma_{t+1} \vSigma_{t-1}^{-1} (\vmu_t - \vmu_{t-1} ) 
\end{align*}
where we use a symmetric normal inverse-Gaussian distribution  with fixed $\lambda>0$ as the variational distribution, $\nabla_{\mu_t}\mathcal{L}$
and $\nabla_{\Sigma_t} \mathcal{L}$ are defined at
\eqref{eq:snig_g_update_mu} -\eqref{eq:snig_g_update_sigma}.

Likewise, we can obtain the variational Adam update with Gaussian prior $p(\vlat)=\gauss(\vlat|\mathbf{0},\delta^{-1} \vI)$ as shown below.
where $\vlat \in \mathcal{R}^d$,
$\vSigma^{-1}=\mathrm{Diag}(N\vs+\delta(1+\lambda^{-1}))$, 
				and \begin{align*}u(\vlat) &=
\sqrt{  \frac{ \left( \vlat-\vmu \right)^T \vSigma^{-1} \left( \vlat-\vmu \right) + \lambda }{  \lambda } } \,\, \frac{ \mathcal{K}_{\frac{d-3}{2}} \left( \sqrt{   \lambda \left( \left(\vlat-\vmu \right)^T \vSigma^{-1} \left(\vlat-\vmu \right) + \lambda \right) } \right)} { \mathcal{K}_{\frac{d-1}{2}} \left( \sqrt{  \lambda  \left( \left(\vlat-\vmu \right)^T \vSigma^{-1} \left(\vlat-\vmu \right) + \lambda \right) } \right)} 
				    \end{align*}
Recall that the ratio about the Bessel functions can be approximated by Eq \eqref{eq:ratio_bessel_sec} when $d\geq 2$.

\fbox{
			\begin{minipage}{.85\textwidth}
				\textbf{Symmetric Normal Inverse-Gaussian ($\lambda>0$) extension}
				\begin{algorithmic}[1]
					\WHILE{not converged}
            \STATE $\vlat \leftarrow \vmu {\, +\,  \vsigma\circ \vepsilon}$, where $\mix \sim \IGauss(1,\lambda)$, $\vepsilon \sim \gauss(\mathbf{0},\vI)$, $\vsigma\leftarrow \sqrt{1/\sqr{\mix(N\vs +  \delta (1+\lambda^{-1})) } }$
               \STATE Randomly sample a data example $\data_i$
               \STATE $\vg_{\mu} \leftarrow - \nabla \log p(\data_i|\vlat) $
               \STATE \textbf{option I:}
                $\vg_{s} \leftarrow  \mix^{-1} \vg_{\mu} \circ  \vg_{\mu}   $  
	      \STATE \textbf{option II:}
                $\vg_{s} \leftarrow  u(\vlat) \vg_{\mu} \circ \vg_{\mu} $  
               
					\STATE $\vm_{\mu} \leftarrow \gamma_1 \, \vm_{\mu} + (1-\gamma_1) \, \rnd{  \vg_{\mu} { \,\, + \,\, \delta \vmu/N } }$ 
               \STATE $\vs \leftarrow \gamma_2 \, \vs + (1-\gamma_2) \, \vg_s $ 
					\STATE $\hat{\vm}_\mu \leftarrow \vm_\mu/(1-\gamma_1^t),  \quad \hat{\vs} \leftarrow (\vs+  \frac{\delta(1+\lambda^{-1})}{N} )/(1-\gamma_2^t)$
               \STATE $\vmu \leftarrow \vmu - \beta \,\, \hat{\vm}_\mu / \sqrt{\hat{\vs} }$
               
					\STATE $t \leftarrow t + 1$
					\ENDWHILE
				\end{algorithmic}
	\end{minipage}
	}

\end{appendix}

\end{document}